\documentclass[lettersize,journal]{IEEEtran}
\usepackage{amsmath,amsfonts}
\usepackage{algorithmic}
\usepackage{algorithm}
\usepackage{array}
\usepackage{subcaption}
\usepackage{textcomp}
\usepackage{stfloats}
\usepackage{url}
\usepackage{verbatim}
\usepackage{graphicx}
\usepackage{cite}
\usepackage{mathtools}
\usepackage{amsthm}
\usepackage{amsmath}
\usepackage{accents}
\usepackage{enumitem}
\usepackage{tcolorbox}
\usepackage[algo2e,algoruled,boxed,lined,norelsize]{algorithm2e} 
\usepackage{tablefootnote}
\usepackage{booktabs}
\usepackage{threeparttable}
\usepackage{multirow}
\usepackage{multicol}

\definecolor{mycolor}{rgb}{0,0, 1}

\graphicspath{{figures/}}

\newtheorem*{remark*}{Remark}

\newtheorem{lemma}{Lemma}

\newtheorem{assumption}{Assumption}

\DeclarePairedDelimiter{\norm}{\lVert}{\rVert}

\DeclareMathAlphabet{\pazocal}{OMS}{zplm}{m}{n}

\newcommand*{\tran}{^{\mkern-1.5mu\mathsf{T}}}
\newcommand{\obj}{\text{cube}}
\newcommand{\goal}{\text{goal}}
\newcommand{\ft}{\text{fingertip}}

\newcommand{\sign}{\texttt{sign}}

\DeclareMathOperator{\E}{\mathbb{E}}

\SetKwInOut{Initial}{Initial}
\SetKwInOut{Parameter}{Parameters}
\SetKw{Continue}{continue}
\allowdisplaybreaks

\begin{document}

\title{Task-Driven Hybrid Model Reduction for Dexterous Manipulation}

\author{
	Wanxin Jin and
	Michael Posa
\thanks{
	Wanxin Jin is with  the School for Engineering of Matter, Transport, and Energy at Arizona State University, Tempe, AZ 85287, and  Michael Posa is with the  General Robotics, Automation, Sensing and Perception (GRASP) Laboratory, University of Pennsylvania, Philadelphia, PA 19104, USA. Email: {wanxin.jin@asu.edu, posa@seas.upenn.edu}
}

}

\markboth{\normalsize
	\textit{ T\MakeLowercase{his is a preprint}. T\MakeLowercase{he published version can be accessed at} IEEE T\MakeLowercase{ransactions on} R\MakeLowercase{obotics}.
	}
}%
{Shell \MakeLowercase{\textit{et al.}}: A Sample Article Using IEEEtran.cls for IEEE Journals}


\maketitle

\begin{abstract}
In contact-rich tasks, like dexterous manipulation, the hybrid nature of making and breaking contact creates challenges for model representation and control. For example, choosing and sequencing contact locations for in-hand manipulation, where there are thousands of potential hybrid modes, is not generally tractable. In this paper, we are inspired by the observation that far fewer modes are actually necessary to accomplish many tasks. Building on our prior work learning hybrid models, represented as linear complementarity systems,
we find a reduced-order hybrid model requiring only a limited number of task-relevant modes. This simplified representation, in combination with model predictive control, enables real-time control yet is sufficient for achieving high performance. We demonstrate the proposed method first on synthetic hybrid systems, reducing the mode count by multiple orders of magnitude while achieving task performance loss of less than 5\%. We  also apply the proposed method to a three-fingered robotic hand manipulating a previously unknown object. With no prior knowledge, we achieve state-of-the-art closed-loop  performance within a few minutes of online learning, by  collecting only a few thousand environment samples.
\end{abstract}


\section{Introduction}
Many robotic tasks, like legged locomotion or dexterous manipulation, involve a robot frequently making and breaking contact with the physical environment or/and objects. The rich-contact behavior makes the robotic system multi-modular and hybrid, characterized by a  set of discrete contact modes and continuous physical dynamics within each mode.

The hybrid nature of contact-rich robotic systems poses great challenges in their representation and control. For data-driven modeling, recent results have demonstrated that standard deep learning networks struggle to train on and represent stiffness and multi-modality \cite{parmar2021fundamental, bianchini2022generalization}, motivating the learning frameworks explicitly designed to capture hybrid dynamics \cite{pfrommer2020contactnets,jin2022learning}. Similar challenges exist in planning and control of contact-rich  systems, where algorithms must jointly reason over a combinatoric number of discrete contact choices and continuous inputs of physical actuation. This process will quickly become intractable as the number of potential hybrid modes and planning depth grow.  The above two aspects become even more critical for real-time closed-loop control of contact-rich robotic systems, where a compromise between computational tractability and task performance has to be made \cite{chen2020optimal}.

Towards a goal of real-time planning and control of contact-rich manipulation with tens of thousands of modes, we hypothesize that identifying and utilizing a full hybrid dynamics model is almost certainly unnecessary. Instead, one might ask:

\emph{Can a far simpler model, with only a few task-relevant hybrid modes, enable the high performance and real-time  control for contact-rich manipulation?}\newline
Here, we propose to answer the question in the affirmative by building upon recent  progress in hybrid representation learning \cite{jin2022learning} and real-time contact-rich planning and control \cite{AydinogluReal,cleac2021fast}.

If one observes a multi-finger robot manipulating a cube for a reorientation task, the task-critical contact interactions might be dominated by a few modes: for example,  all fingertips stick to the cube, or one fingertip pushing or sliding while others stick to the cube. While other modes might occur, they do so briefly or in a functionally similar manner to another mode. This observation inspires us to study the problem of learning task-driven reduced-order hybrid models. On the technical side,  we have seen the recent progress in learning hybrid representations \cite{jin2022learning,pfrommer2020contactnets}. Particularly in our prior work \cite{jin2022learning}, we have developed an efficient method to learn a piecewise affine system, represented as a linear complementarity system (LCS) [detailed in Section \ref{section.formulation}], with tens of thousands of hybrid modes. We also note recent progress towards fast control and planning of contact-rich robotic systems. For example, in \cite{AydinogluReal}, the authors approximate nonlinear contact-rich dynamics using LCS and develop  LCS-based model predictive control, achieving real-time control performance for a reasonably-sized manipulation system.

Built on the above observation and the foundational prior work, this paper aims to answer the above question theoretically and algorithmically. Our goal is to find a reduced-order hybrid model, containing only a small number of task-relevant modes, which is sufficient for high-performance, real-time control of contact-rich manipulation tasks. We call the problem `task-driven hybrid model reduction'. 
The primary contributions of this work are:
\begin{enumerate}[wide, labelindent=1em, label=(\roman*)]
\item We study the problem of task-driven hybrid model reduction by formulating it as minimizing the task performance gap between  model predictive control (MPC) with the reduced-order hybrid model and MPC with the full hybrid dynamics. We show that the reduced-order model learning with on-policy MPC data provably upper bounds the task performance gap, leading to a simple iterative method to improve the reduced-order model and MPC controller. 
\item We make use of our prior work of learning LCS  \cite{jin2022learning}, and the recent development of real-time LCS-based control on contact-rich systems, such as \cite{AydinogluReal}, to develop our practical learning algorithm. The algorithm runs a real-time closed-loop LCS model predictive controller on the complete hybrid dynamical system (environment), enabling  improving the reduced-order model and its closed-loop control performance.
\item In the first example, we demonstrate the capabilities of the proposed method in reducing synthetic hybrid control systems. We show that the proposed method enables reducing the hybrid mode count by multiple orders of magnitude while achieving a task performance loss of less than 5\%. In the second application,  we use the proposed method to solve three-finger robotic hand manipulation for unknown object reorientation in simulation environment. With no prior knowledge, we achieve state-of-the-art closed-loop performance within a few minutes of online learning, by  collecting only a few thousand environment samples.
\end{enumerate}

The following article is organized as follows. The related work is reviewed in Section \ref{section.relatedwork}. Section \ref{section.formulation} gives  preliminaries and formulates the problem of task-driven hybrid mode reduction. Section \ref{section.theory} presents the theoretical analysis and Section \ref{section.algorithm} develops the algorithm. Section \ref{section.pwa_reduction} uses the proposed method to solve model reduction on synthetic hybrid systems. Section \ref{section.trifinger} applies the proposed method to solve  three-finger robotic hand manipulation. Conclusions are drawn in Section \ref{section.conclusion}.

\section{Related Work}\label{section.relatedwork}

\subsubsection{Learning Hybrid Dynamics Models} This work heavily leverages the recent results in learning multi-modal dynamics representations. Previous studies \cite{pfrommer2020contactnets,parmar2021fundamental} have shown that naive neural networks fail to capture the discontinuity and stiffness of physical systems. A prominent line of recent work focuses on learning smoothing approximations by relaxing the hybrid mode boundaries \cite{de2018end,geilinger2020add,heiden2021neuralsim, howell2022dojo,suh2022differentiable}, though at the cost of some approximation error. Instead of using smoothing approximation, this paper considers learning explicit hybrid structures.
We focus on a simple yet expressive representation for hybrid systems: continuous piecewise-affine (PWA) models. This is motivated by the fact that many physics simulation engines \cite{todorov2012mujoco,stewart2000implicit,coumans2013bullet,drake} locally use  linear complementarity models (a compact form of continuous PWA \cite{heemels2001equivalence})  to handle physical contact events at each simulation step. PWA models bring two benefits. First, they are a well-studied subject in the control community \cite{heemels2001equivalence,bemporad2000piecewise,marcucci2019mixed}, which captures the multi-modality of a hybrid system, by approximating dynamics using polyhedral partitions with each assigned a mode-dependent linear model. Second, they can be tractably incorporated into planning and control for real-time performance due to recent progress in  \cite{AydinogluReal}.  Although continuous PWA models cannot capture the discontinuities that arise from impulsive impact events, there is a large range of manipulation tasks in which such events are not prominent, thus we believe PWA models to be sufficient. Also, we noted that
a  PWA model, via stiffness, can well capture the discontinuity, as adopted by many physics simulation engines \cite{todorov2012mujoco,stewart2000implicit,coumans2013bullet,drake}.

Identifying PWA models is NP-hard in general \cite{lauer2015complexity}. Most existing methods  \cite{ferrari2003clustering,bemporad2022piecewise} for  PWA regression are clustering-based: they alternate  data classification  and model  regression   for each class. Those methods normally have a complexity that scales exponentially with the number of data points or hybrid modes.  In this paper, learning PWA models is based on our recent work \cite{jin2022learning}. Specifically, we write a PWA model compactly as a linear complementarity system (LCS), via  implicit parameterization \cite{heemels2001equivalence}, and propose an implicit violation-based loss that generalizes the physics-based method in ContactNets \cite{pfrommer2020contactnets}. The method does not need explicitly cluster  data  and can handle tens of thousands of (potentially stiff)  modes efficiently. 
Recent results have  proven a superior generalization of this class of methods than  explicit loss methods \cite{bianchini2022generalization}.

\smallskip

\subsubsection{Fast Planning and Control on Multi-Contact Systems} 
The success of the proposed method also relies on the recent progress in real-time multi-contact planning and control.
Planning and control on multi-contact systems are notoriously challenging, as the algorithms must decide when and where to make or break contacts, whose complexity scales exponentially to the number of potential contacts and planning horizon. Traditionally, \cite{sleiman2021unified,mastalli2020crocoddyl}  use the  predefined sequence of mode to achieve  real-time  control on legged locomotion \cite{winkler2018gait} and manipulation \cite{hogan2020reactive}. To enable general-purpose fast multi-contact control,  \cite{AydinogluReal} and \cite{cleac2021fast} consider the LCS linearization of nonlinear multi-contact robot dynamics. Specifically, in  \cite{cleac2021fast}, the authors smooth the stiff complementarity constraint and then apply the interior-based method to approximate the solution sequentially. In a  different way,  \cite{AydinogluReal} maintains the hybrid structures and proposes to decouple the combinatoric complexity from the planning depth and then use the alternating direction method of multipliers (ADMM) to solve the decoupled problem, which can be done in parallel for further acceleration. 
 
In this paper, we include a real-time LCS model predictive controller as part of our learning algorithm for on-policy data collection and closed-loop control. We use a direct method of optimal control to formulate and solve the LCS MPC. This was first proposed in \cite{posa2014direct}. 
In our implementation, we utilize the state-of-the-art optimal control solver \cite{andersson2019casadi} for fast MPC.

\smallskip

\subsubsection{Reinforcement Learning for Contact-Rich Manipulation}  Reinforcement learning (RL) has achieved impressive results in   contact-rich manipulation \cite{nagabandi2020deep,allshire2021transferring,andrychowicz2020learning}. 
Some representative work includes \cite{andrychowicz2020learning}, where  in-hand manipulation
policies are learned for object reorientation, and \cite{allshire2021transferring} for solving  TriFinger Manipulation. However, both methods use model-free RL,  requiring millions or even billions of environment samples and many hours or even days of training. To alleviate  sample inefficiency,  model-based RL has been used to robotic manipulation by first learning a dynamics model to aid policy search \cite{nagabandi2020deep}, though requiring a large amount of training data to fit an unstructured deep neural network. Furthermore,  control with deep neural network models can be challenging. Commonly used shooting-based methods \cite{de2005tutorial} have a complexity exponential to planning depth and system dimension \cite{wang2019benchmarking}. 

In comparison with the work above, the emphasis of this paper is on highly data-efficient hybrid model learning, paired with real-time closed-loop control. Specifically, the tasks that might require hours of data for unstructured learning methods will be trained and completed in minutes.

\smallskip

\subsubsection{Reduced-order Models for Multi-Contact Robotic Tasks}
The idea of using a reduced-order model for hybrid robotic tasks has widely used in robot locomotion \cite{kajita1991study, orin2013centroidal} for real-time generating behavior plans. However, these reduced-order models are manually designed and may miss some key dynamics aspects of the full-order dynamics \cite{pandala2022robust}. To address those challenges, recent results demonstrate the ability to optimize for a reduced-order model that retains the capabilities of the full-order robot dynamics  \cite{chen2020optimal}. In their paper, authors focus on reducing the state dimension needed for planning, while we focus here on the comparatively unexplored problem of hybrid mode reduction.  Recently, model-free RL has been used to learn the unmodeled aspects of a reduced-order model to improve locomotion performance \cite{pandala2022robust}. Our method differs from theirs in three aspects. 
First, their formulation does not \emph{explicitly} encourage the reduction of the performance gap of the reduced-order model, while our formulation is to directly minimize the performance gap. Second, our method is not rooted in model-free RL, which can be data inefficient for our setting, where true dynamics is originally unknown (thus prohibiting sim-to-real transfer). Third, rather than using smooth approximations, we directly identify a hybrid representation.

\section{Preliminaries and Problem Formulation}\label{section.formulation}
This section presents some preliminaries and  formulates the problem of   task-driven hybrid model reduction.

\subsection{Hybrid Models for Multi-contact Dynamics}
Consider the following generic  hybrid  system:
\begin{equation}\label{equ.fullmodel}
\begin{aligned}
\boldsymbol{x}_{t+1}&=\boldsymbol{f}_i(\boldsymbol{x}_t, \boldsymbol{u}_t)\quad \text{with}\quad (\boldsymbol{x}_t,\boldsymbol{u}_t)\in \mathcal{P}_i,\\
           \mathcal{P}_i&=\{(\boldsymbol{x},\boldsymbol{u})\,|\, \boldsymbol{\psi}_{i}(\boldsymbol{x}_t,\boldsymbol{u}_t)\leq \boldsymbol{0}\},\quad
          i\in \{1, 2,\dots, I\}.
\end{aligned}
\end{equation}
Here, $\boldsymbol{{x}}_t\in\mathbb{R}^n$ and $\boldsymbol{{u}}_t\in\mathbb{R}^m$ are  the system state and input at time step $t=0,1,2,\dots$. $i\in \{1, 2,\dots, I\}$ is the index of the system hybrid modes, and  $\mathcal{P}_i\subset\mathbb{R}^n\times\mathbb{R}^m$ denotes the domain of the $i$-th mode, defined as a sublevel set of $\boldsymbol{\psi}_{i}(\boldsymbol{x},\boldsymbol{u})$. $\boldsymbol{f}_i$ is the dynamics model (vector field) in the $i$-th mode. 
\smallskip

A subset  of   hybrid systems in (\ref{equ.fullmodel}) corresponds to   complementarity systems, which have been widely used to describe the multi-contact model of robot dynamics \cite{stewart2000rigid,todorov2011convex,stewart2000implicit}:
\begin{equation}\label{equ.contactdyn.1}
    \boldsymbol{M}(\boldsymbol{q}_t)(\boldsymbol{v}_{t{+}1}-\boldsymbol{v}_t)=\boldsymbol{C}(\boldsymbol{q}_t, \boldsymbol{v}_t)+\boldsymbol{B}\boldsymbol{u}_t +\sum_{i=1}^{I} \boldsymbol{J}_i({\boldsymbol{q}_t})\tran{\Lambda}_{i,t},
\end{equation}
with the  $i$-th contact impulse ${\Lambda}_{i,t}$ satisfying the complementarity constraint
\begin{equation}\label{equ.contactdyn.2}
    \boldsymbol{0}\leq {\Lambda}_{i,t}\perp{\Phi}_{i,t}(\boldsymbol{q}_t, \boldsymbol{v}_{t+1}, {\Lambda}_{i,t})\geq \boldsymbol{0},
    \quad i=1,2,\dots, I.
\end{equation}
Here, $(\boldsymbol{q}_t, \boldsymbol{v}_t)$ is the generalized coordinate and velocity of a robot system.  $\boldsymbol{u}_t$ is the  actuation  impulse with input projection matrix $\boldsymbol{B}$. $\boldsymbol{M}(\boldsymbol{q}_t)$ is the  inertia matrix, and  $\boldsymbol{C}(\boldsymbol{q}_t,\boldsymbol{v}_t)$ includes all non-contact  impulses resulting from gravity and gyroscopic forces. 
   ${\Lambda}_{i,t}$ is  $i$-th contact impulse between the robot and  objects/environments, and $\boldsymbol{J}_i(\boldsymbol{q}_t)$ is its Jacobian matrix. The complementarity constraint  (\ref{equ.contactdyn.2}) means either the contact impulse ${\Lambda}_{i,t}$ or the value of its distance-related function, ${\Phi}_{i,t}(\boldsymbol{q}_t, \boldsymbol{v}_{t+1}, {\Lambda}_{i,t})$, is zero, but both cannot be negative, i.e., contact interaction between robot and object/environment cannot pull or penetrate into each other. Coulomb friction can be similarly described (e.g. \cite{stewart2000implicit}).

Define $\boldsymbol{x}:=[\boldsymbol{q}, \boldsymbol{v}]\tran$,  $\boldsymbol{\Lambda}:=[{\Lambda}_{1},{\Lambda}_{1}, ..., {\Lambda}_{I}]\tran$, and $\boldsymbol{\Phi}:=[{\Phi}_1,{\Phi}_2, ..., {\Phi}_I]\tran$. One can abstractly write  the multi-contact dynamics (\ref{equ.contactdyn.1})-(\ref{equ.contactdyn.2}) into the  general form below, denoted as $\boldsymbol{f}()$,
\begin{equation}\label{equ.fulldyn}
\boldsymbol{f}():\quad\quad \begin{cases}
\boldsymbol{F}( \boldsymbol{x}_{t+1}, \boldsymbol{x}_t,\boldsymbol{u}_t,\boldsymbol{\Lambda}_t)=\boldsymbol{0},
\\
\boldsymbol{0}\leq \boldsymbol{\Lambda}_t\perp\boldsymbol{\Phi}_t(\boldsymbol{x}_{t+1},\boldsymbol{x}_t, \boldsymbol{u}_t,  \boldsymbol{\Lambda}_t)\geq \boldsymbol{0}.
\end{cases} \quad\,
\end{equation}
Connecting (\ref{equ.fulldyn}) to (\ref{equ.fullmodel}), here  the active or inactive constraints in  $\boldsymbol{\Phi}\geq0$ determine the domain of hybrid modes, and $\boldsymbol{F}$ and $\boldsymbol{\Phi}$ jointly and implicitly determine the dynamics model of each hybrid mode. A significant amount of recent work focuses on identifying/learning the above complementary-based hybrid dynamics, such as \cite{pfrommer2020contactnets,bianchini2022generalization,jin2022learning,de2018end,howell2022dojo}.

The above complementary-based hybrid dynamics $\boldsymbol{f}()$ contains all potential contact modes in the aggregated contact impulse vector $\boldsymbol{\Lambda}$ and the corresponding distance vector function $\boldsymbol{\Phi}$. Thus,  we call  $\boldsymbol{f}()$ the \emph{full-order hybrid dynamics}.

\subsection{Full-Order Model Predictive Control}
We consider the model predictive control (MPC)  with the full-order hybrid  dynamics $\boldsymbol{f}()$ in (\ref{equ.fulldyn}) for a given set of tasks:
\begin{equation}\label{equ.mpc}
	\begin{aligned}
		 \min_{\boldsymbol{{u}}_{0:T-1}} &\quad J_{{\boldsymbol{\beta}}}=\sum_{t=0}^{T-1}c_{\boldsymbol{{\boldsymbol{\beta}}}}(\boldsymbol{{x}}_t, \boldsymbol{{u}}_t)+h_{\boldsymbol{\beta}}(\boldsymbol{{x}}_{T})
		 \\
		\text{s.t.} &  \quad	\boldsymbol{{x}}_{t+1}=\boldsymbol{f}(\boldsymbol{{x}}_t, \boldsymbol{{u}}_t),  \,\,\text{given}\,\, \boldsymbol{{x}}_0\sim p_{\boldsymbol{\beta}}(\boldsymbol{x}_0),
			\end{aligned}
\end{equation}
where $T$ is the MPC horizon; $\boldsymbol{f}()$ is the hybrid dynamics model in (\ref{equ.fulldyn}); and  $J_{\boldsymbol{\beta}}$ is a cost function for given tasks. Here, $\boldsymbol{\beta}$ is a general hyperparameter,  indexing a set of robot tasks of interest, subject to a known task distribution  ${\boldsymbol{\beta}}\sim p({\boldsymbol{\beta}})$. For example, if the  tasks of interest are that a robotic hand moving  an object to different target poses, $\boldsymbol{\beta}$  can parameterize the set of target poses of the object, subject to a given distribution $p({\boldsymbol{\beta}})$ reflecting the frequency of appearance of  target pose $\boldsymbol{\beta}$; if the tasks of interest contain different types of robot tasks, such as inserting a peg, turning a crank, etc.,  $\boldsymbol{\beta}$ can be the task index,  and $p({\boldsymbol{\beta}})$ could be a uniform distribution.

For notation simplicity, we write the system input and state trajectories compactly as $\mathbf{u}:=\{\boldsymbol{u}_{0}, \boldsymbol{u}_{1},  \boldsymbol{u}_{2},  \dots, \boldsymbol{u}_{T-1} \} $  and  $\mathbf{x}:=\{\boldsymbol{x}_{0}, \boldsymbol{x}_{1}, \dots, \boldsymbol{x}_{T-1}, \boldsymbol{x}_{T} \} $, respectively. The system state trajectory  given  input trajectory $\mathbf{u}$ and  initial  $\boldsymbol{x}_0$ is written as
 \begin{equation}\label{equ.F}
	\mathbf{F}(\mathbf{u}, \boldsymbol{x}_0):=\Big\{\mathbf{x} \,\big\vert\, \, \boldsymbol{x}_{t+1}{=}\boldsymbol{f}(\boldsymbol{x}_t, \boldsymbol{u}_t), \,\text{given} \,\,\, \boldsymbol{x}_0\,\,\text{and}\,\,
	\mathbf{u} \Big\}.
\end{equation}
The solution to  (\ref{equ.mpc}) then can be compactly written as 
\begin{multline}\label{equ.mpc2}
\boldsymbol{f}\text{-MPC}: \quad \mathbf{u}^{\boldsymbol{f}}(\boldsymbol{x}_0,{\boldsymbol{\beta}}):=\arg\min_{\mathbf{u}} J_{\boldsymbol{\beta}}\big(\mathbf{u}, \mathbf{F}(\mathbf{u},\boldsymbol{x}_0)\big), \,\,\,\\ \boldsymbol{x}_0\sim p_{\boldsymbol{\beta}}(\boldsymbol{x}_0), \,\, {\boldsymbol{\beta}}\sim p({\boldsymbol{\beta}}).
\end{multline}

The   full-order dynamics MPC in (\ref{equ.mpc2}) is  applied to the multi-contact robot system $\boldsymbol{f}()$ in a closed-loop (receding) fashion. Specifically, at rollout time step $k=1,2,3,...$,  $\boldsymbol{x}_0$ in (\ref{equ.mpc2}) is set  to the robot's  actual state: $\boldsymbol{x}_0=\boldsymbol{x}_{k}^{\boldsymbol{f}}$. After solving
$
	\mathbf{{u}}^{\boldsymbol{f}}(\boldsymbol{x}_{k}^{\boldsymbol{f}},  {\boldsymbol{\beta}})=
	\{
			\boldsymbol{{u}}^{\boldsymbol{f}}_0, 
	\dots, 
	\boldsymbol{{u}}^{\boldsymbol{f}}_{T-1} 
	\}
$
from (\ref{equ.mpc2}), only the first input $\boldsymbol{{u}}^{\boldsymbol{f}}_0$ is applied to the robot for execution and drive the robot to the next state: $\boldsymbol{x}_{k+1}^{\boldsymbol{f}}$ via $\boldsymbol{x}_{k+1}^{\boldsymbol{f}}=\boldsymbol{f}(\boldsymbol{x}_{k}^{\boldsymbol{f}},\boldsymbol{{u}}^{\boldsymbol{f}}_0)$.
Then, this process repeats  at the robot new state $\boldsymbol{x}_{t+1}^{\boldsymbol{f}}$. 
The above MPC  leads to a closed-loop control policy: mapping from robot's current state $\boldsymbol{x}_k^{\boldsymbol{f}}$ to its  control input $\boldsymbol{{u}}^{\boldsymbol{f}}_{0}$.

As indicated by the full-order dynamics $\boldsymbol{f}()$ in (\ref{equ.fulldyn}),  solving the MPC in (\ref{equ.mpc2}) requires reasoning over the sequence of contact impulses $\{
\boldsymbol{\Lambda}_0, \boldsymbol{\Lambda}_1, \dots, \boldsymbol{\Lambda}_{T-1}
\}$ in addition to $\mathbf{u}$ and $\mathbf{x}$. Its combinatoric complexity is $2^{TI}$ ($I=\dim\boldsymbol{\Lambda}$). Despite recent progress, particularly on modestly sized problems \cite{AydinogluReal,de2018end,howell2022dojo,marcucci2019mixed}, a large number of potential contact interactions (e.g., large $I$) will make solving (\ref{equ.mpc2}) intractable.

\smallskip
In this paper, we hypothesize that identifying and utilizing a full-order dynamics
 $\boldsymbol{f}()$ for  MPC  in (\ref{equ.mpc2}) is almost certainly unnecessary, because far fewer modes are actually necessary to accomplish many tasks. Thus, we will find a reduced-order hybrid model proxy to replace $\boldsymbol{f}()$ in (\ref{equ.mpc2}), to enable real-time control and sufficiently achieve high task performance.

\subsection{Linear Complementarity Systems}\label{section.formulation.pwa}

To find a reduced-order hybrid representation  for the full-order hybrid dynamics $\boldsymbol{f}()$ in (\ref{equ.fulldyn}), we  consider  piecewise affine (PWA) models. 
This is motivated by the fact that many physics simulation engines \cite{todorov2012mujoco,stewart2000implicit,coumans2013bullet,drake} locally use  linear complementarity models (a compact formation of continuous PWA \cite{heemels2001equivalence})  to handle physical contacts at each simulation step. A PWA model can sufficiently describe
multi-modality  but is tractable enough for planning and control tasks due to their simple (affine) structures. 
As in our previous work \cite{jin2022learning},  we compactly represent PWA models as a linear complementarity system (LCS), defined as $\boldsymbol{g}()$,
\begin{equation}\label{equ.lcs.reduced}
\hspace{-30pt}
\boldsymbol{g}():\qquad
\begin{aligned}
     \boldsymbol{x}_{t+1}&=A\boldsymbol{x}_t+B\boldsymbol{u}_t+C\boldsymbol{\lambda}_t+\boldsymbol{d}\\
          \boldsymbol{0}& \leq\boldsymbol{\lambda}_t\perp D\boldsymbol{x}_t+E\boldsymbol{u}_t +F\boldsymbol{\lambda}_t+\boldsymbol{c} \geq \boldsymbol{0}.
\end{aligned}
\end{equation}
Here, the first line of (\ref{equ.lcs.reduced}) is the affine dynamics and the second line is the complementarity equation. $(A, B, C, \boldsymbol{d}, D, E, F, \boldsymbol{c})$ are system matrix parameters with compatible dimensions. 
$\boldsymbol{\lambda}_t\in\mathbb{R}^r$  is  the complementarity variable and solved from the complementarity equation given $(\boldsymbol{x}_t, \boldsymbol{u}_t)$. Depending on the  active or inactive inequalities in $D\boldsymbol{x}_t+E\boldsymbol{u}_t +F\boldsymbol{\lambda}_t+\boldsymbol{c} \geq \boldsymbol{0}$ (corresponding to different partitions of the state-input space), $\boldsymbol{\lambda}_t\in\mathbb{R}^r$ is a  piecewise function of $(\boldsymbol{x}_t, \boldsymbol{u}_t)$. By composing with affine dynamics, 
 $\boldsymbol{x}_{t+1}$  is eventually a piecewise function of $(\boldsymbol{x}_t, \boldsymbol{u}_t)$, and each linear piece  is a hybrid mode.
 Thus, the  maximum number of the hybrid modes the LCS in (\ref{equ.lcs.reduced}) can represent  is $2^{\dim \boldsymbol{\lambda}}$.
For any  given $(\boldsymbol{x}_t, \boldsymbol{u}_t)$, to guarantee the existence and uniqueness of  $\boldsymbol{\lambda}_t$ solved from the complementary equation, we impose the restriction that the symmetric part of $F$ be positive definite, $F\tran +F \succ 0$ \cite{tsatsomeros2002generating,cottle2009linear}. This property can be accomplished by parameterizing $F$ as
\begin{equation}
    F:=GG\tran+H-H\tran,
\end{equation}
with $G$ and $H$ matrices with the same dimension as $F$.

 As we will seek to learn a reduced-order LCS    $\boldsymbol{g}()$, we can explicitly restrict the number of potential modes in $\boldsymbol{g}()$  by setting the dimension of the complementary variable,  $\dim \boldsymbol{\lambda}$. Compared to the full-order dynamics $\boldsymbol{f}()$ in (\ref{equ.fulldyn}),  $\boldsymbol{g}()$ has 
 \begin{equation}
     \dim \boldsymbol{\lambda}< \dim\boldsymbol{\Lambda}.
 \end{equation}
Note that, we do not expect a tight connection between $\boldsymbol{\lambda}$ in $\boldsymbol{g}()$ and the physical contact impulse vector $\boldsymbol{\Lambda}$ in $\boldsymbol{f}()$. Instead, $\boldsymbol{\lambda}$  in $\boldsymbol{g}()$ here will represent general multi-modality, and while we will later observe that $\boldsymbol{\lambda}$  is empirically related to the contact forces, it is not exactly the same.

\subsection{Problem Formulation}

We aim to find a reduced-order LCS model $\boldsymbol{g}()$ in (\ref{equ.lcs.reduced})  for the  given set of tasks $J_{\boldsymbol{\beta}}$ in (\ref{equ.mpc}),  and establish the following  reduced-order  $\boldsymbol{g}$-MPC (using the  notation convention  in (\ref{equ.mpc2})):
\begin{multline}\label{equ.mpcsimple}
\boldsymbol{g}\text{-MPC}: \quad\mathbf{u}^{\boldsymbol{g}}(\boldsymbol{x}_0,{\boldsymbol{\beta}}):=\arg\min_{\mathbf{u}} J_{\boldsymbol{\beta}}\big(\mathbf{u}, \mathbf{G}(\mathbf{u},\boldsymbol{x}_0)\big), \,\,\,\\ \boldsymbol{x}_0\sim p_{\boldsymbol{\beta}}(\boldsymbol{x}_0), \,\, {\boldsymbol{\beta}}\sim p({\boldsymbol{\beta}}),
\end{multline}
such that when running  the reduced-order   $\boldsymbol{g}\text{-MPC}$  on the full-order robot  dynamics  $\boldsymbol{f}()$, one can achieve a task performance as similar to the task performance of running   $\boldsymbol{f}\text{-MPC}$ on $\boldsymbol{f}()$ as possible. Here, we also compactly write the state trajectory of $\boldsymbol{g}()$ given $\mathbf{u}$ and $\boldsymbol{x}_0$ as
 \begin{equation}\label{equ.G}
	\mathbf{G}(\mathbf{u}, \boldsymbol{x}_0):=\Big\{\mathbf{x} \,\big\vert\, \, \boldsymbol{x}_{t+1}{=}\boldsymbol{g}(\boldsymbol{x}_t, \boldsymbol{u}_t), \,\text{given} \,\,\, \boldsymbol{x}_0\,\,\text{and}\,\,
	\mathbf{u} \Big\}.
\end{equation}
Therefore, the goal of  task-driven reduced-order model learning is to find the reduced-order LCS $\boldsymbol{g}()$ which minimizes the following \emph{task  performance gap}:
\begin{multline}\label{equ.loss1}
\mathcal{L}(\boldsymbol{g}):=  \E_{{\boldsymbol{\beta}}\sim p({\boldsymbol{\beta}})} \E_{\boldsymbol{x}\sim p_{\boldsymbol{\beta}}(\boldsymbol{x}_0)}  \bigg[J_{{\boldsymbol{\beta}}}\big(\mathbf{u}^{\boldsymbol{g}}, \mathbf{F}(\mathbf{u}^{\boldsymbol{g}},\boldsymbol{x}_0)\big) \\{-}J_{{\boldsymbol{\beta}}}\big(\mathbf{u}^{\boldsymbol{f}}, \mathbf{F}(\mathbf{u}^{\boldsymbol{f}},\boldsymbol{x}_0)\big)\bigg],
\end{multline}
where the first cost $J_{{\boldsymbol{\beta}}}\big(\mathbf{u}^{\boldsymbol{g}}, \mathbf{F}(\mathbf{u}^{\boldsymbol{g}},\boldsymbol{x}_0)\big)$ is the task performance of running reduced-order $\boldsymbol{g}\text{-MPC}$  on the robot system $\boldsymbol{f}()$, and  the second  cost $J_{{\boldsymbol{\beta}}}\big(\mathbf{u}^{\boldsymbol{g}}, \mathbf{F}(\mathbf{u}^{\boldsymbol{g}},\boldsymbol{x}_0)\big)$ is the task performance of running full-order  $\boldsymbol{f}\text{-MPC}$ on the robot system $\boldsymbol{f}()$. Here,  $\mathbf{u}^{\boldsymbol{g}}$ is the solution to the reduced-order   $\boldsymbol{g}\text{-MPC}$ in (\ref{equ.mpcsimple}) and $\mathbf{u}^{\boldsymbol{f}}$ is the solution to the full-order $\boldsymbol{f}\text{-MPC}$ in (\ref{equ.mpc2}). 

We make the following remarks on the above problem statement. First, the reduced-order LCS  $\boldsymbol{g}()$  shares the same dimensions of states and inputs as  the full-order  dynamics $\boldsymbol{f}()$, but  has far fewer hybrid modes by setting $\dim \boldsymbol{\lambda}\leq \dim \boldsymbol{\Lambda}$.
Compared to  $\boldsymbol{f}$-MPC, the reduced-order  $\boldsymbol{g}$-MPC is more computationally tractable for real-time implementation. Second, the learning criterion  (\ref{equ.loss1}) is to minimize the performance gap between  the reduced-order  $\boldsymbol{g}$-MPC and   full-order  $\boldsymbol{f}$-MPC, both MPC controllers running on the full-order dynamics $\boldsymbol{f}()$, which is the original hybrid   system. Thus, a minimal  performance gap means that one can confidently use the reduced-order LCS $\boldsymbol{g}()$ to achieve the given tasks $J_{\boldsymbol{\beta}}$,  ${\boldsymbol{\beta}}\sim p({\boldsymbol{\beta}})$.

Directly minimizing the task performance gap $\mathcal{L}(\boldsymbol{g})$  requires access to and optimization with  the full-order dynamics model $\boldsymbol{f}()$, because of  the coupling between  $J_{\boldsymbol{\beta}}()$ and  $\boldsymbol{f}()$ in \eqref{equ.loss1}. However, this is unlikely to be tractable as the full-order model is both unknown and too complex to optimize with. In the following section,  we develop a method to approximately solve (\ref{equ.loss1}) without requiring knowledge of the model  $\boldsymbol{f}()$.

\section{Theoretical Results}\label{section.theory}

 In this section, we will  show that  instead of directly solving (\ref{equ.loss1}), one can minimize its upper bound. This will  lead  to developing a method that is much easier to implement and only requires  samples (zero-order information)  of   $\boldsymbol{f}()$.
To start, we pose a  mild assumption about the Lipschitz continuity of  task cost function $J_{\boldsymbol{\beta}}(\mathbf{u},\mathbf{x})$ for any  ${\boldsymbol{\beta}}\sim p({\boldsymbol{\beta}})$.

\begin{assumption}\label{assumption.1}
For any task   sample ${\boldsymbol{\beta}}\sim p({\boldsymbol{\beta}})$, the task cost function $J_{\boldsymbol{\beta}}(\mathbf{u},\mathbf{x})$  is $M$-Lipschitz continuous, i.e.,  for any $\mathbf{z}_1:=(\mathbf{u}_1,\mathbf{x}_1)$, $\mathbf{z}_2:=(\mathbf{u}_2,\mathbf{x}_2)$, 
\begin{equation}
|J_r(\mathbf{z}_1) -  J_r(\mathbf{z}_2)|\leq M \norm{\mathbf{z}_1-\mathbf{z}_2}
\end{equation}
with $\norm{}$ denoting the $l_2$ norm.
\end{assumption}
The above assumption is  mild, as the  cost function is usually  defined manually and can easily satisfy this condition.
With Assumption \ref{assumption.1}, we have the following lemma stating the upper bound of the task performance gap $\mathcal{L}(\boldsymbol{g})$ in (\ref{equ.loss1}):

\begin{lemma}\label{lemma.1}
Suppose Assumption \ref{assumption.1} holds. For any reduced-order model $\boldsymbol{g}()$, the following inequality holds:
\begin{multline}\label{equ.lossupperbound}
    \mathcal{L}(\boldsymbol{g})
    \leq M \E_{{\boldsymbol{\beta}}\sim p({\boldsymbol{\beta}})} \E_{\boldsymbol{x}\sim p_{\boldsymbol{\beta}}(\boldsymbol{x}_0)} 
     \Big(\big\lVert\mathbf{G}(\mathbf{u}^{\boldsymbol{g}},\boldsymbol{x}_0)-\mathbf{F}(\mathbf{u}^{\boldsymbol{g}},\boldsymbol{x}_0)\big\rVert
    \\
   +
    \big\lVert
     \mathbf{G}(\mathbf{u}^{\boldsymbol{f}},\boldsymbol{x}_0)-\mathbf{F}(\mathbf{u}^{\boldsymbol{f}},\boldsymbol{x}_0)
    \big\rVert\Big) \quad
\end{multline}
\end{lemma}
\begin{proof}
    See Appendix. \ref{appendix.0}.
\end{proof}

Lemma \ref{lemma.1} gives an upper bound for the task  performance gap $\mathcal{L}(\boldsymbol{g})$. Notably, this upper bound is  the \emph{prediction error} between the reduced-order model $\boldsymbol{g}()$ and full-order dynamics $\boldsymbol{f}()$ at their MPC solutions. Specifically, the first term  on the right side of (\ref{equ.lossupperbound})
is the model prediction error  on the  dataset 
\begin{equation}\label{equ.gdata}
    \mathcal{D}^{\boldsymbol{g}}=\Big\{\mathbf{u}^{\boldsymbol{g}}(\boldsymbol{x}_0,{\boldsymbol{\beta}})\,|\,   \boldsymbol{x}_0\sim p(\boldsymbol{x}_0), \,\, {\boldsymbol{\beta}}\sim p({\boldsymbol{\beta}}) \Big\}
\end{equation} generated by the reduced-order $\boldsymbol{g}\text{-MPC}$. The second term on the right side of (\ref{equ.lossupperbound})
is the model prediction error on
\begin{equation}\label{equ.fdata}
    \mathcal{D}^{\boldsymbol{f}}=\Big\{\mathbf{u}^{\boldsymbol{f}}(\boldsymbol{x}_0,{\boldsymbol{\beta}})\,|\, \boldsymbol{x}_0\sim p(\boldsymbol{x}_0), \,\, {\boldsymbol{\beta}}\sim p({\boldsymbol{\beta}}) \Big\}
\end{equation}
generated by the full model $\boldsymbol{f}\text{-MPC}$.
(\ref{equ.lossupperbound}) says that as long as  the  reduced-order model $\boldsymbol{g}()$ captures the full-order dynamics $\boldsymbol{f}()$  \textbf{at the MPC  data} $\mathcal{D}^{\boldsymbol{g}}$ and $\mathcal{D}^{\boldsymbol{f}}$, not necessarily at other parts of data regime (task-irrelevant data),  $\boldsymbol{g}$-MPC can  replace    $\boldsymbol{f}$-MPC  for the same task performance. Thus, Lemma \ref{lemma.1}  justifies the learning of a task-driven reduced-order model.

Although the  $\boldsymbol{f}\text{-MPC}$ policy data $\mathcal{D}^{\boldsymbol{f}}$ in (\ref{equ.fdata}) is not directly   verifiable when   $\boldsymbol{f}()$ is unknown, the next lemma will show  the  $\boldsymbol{g}\text{-MPC}$ data  $\mathcal{D}^{\boldsymbol{g}}$  is related to $\mathcal{D}^{\boldsymbol{f}}$, which thus can be verified indirectly.

\begin{lemma}\label{lemma.2}
Suppose   $\nabla_{\mathbf{x}} J_{\boldsymbol{\beta}}(\mathbf{u}, \mathbf{x})$ is $L_1$-Lipschitz continuous, $\nabla_{\mathbf{u}} J_{\boldsymbol{\beta}}(\mathbf{u}, \mathbf{x})$ is $L_2$-Lipschitz continuous, and  $\norm{\nabla_{\mathbf{x}} J_{\boldsymbol{\beta}}(\mathbf{u}, \mathbf{x})}\leq M_1$ for any $(\mathbf{u},\mathbf{x}, {\boldsymbol{\beta}})$; 
$\norm{\nabla_{\mathbf{u}} \mathbf{G}(\mathbf{u}, \boldsymbol{x}_0)}\leq M_g$ for any $(\mathbf{u}, \boldsymbol{x}_0)$. 
Then, the solutions $\mathcal{D}^{\boldsymbol{g}}$ in (\ref{equ.gdata}) generated by  $\boldsymbol{g}\text{-MPC}$  is also an $\epsilon$-accuracy stationary solution for $\boldsymbol{f}\text{-MPC}$, i.e., 
\begin{equation}\label{equ.lemma2.1}
    \norm{\nabla_{\mathbf{u}} J_{\boldsymbol{\beta}}\big(\mathbf{u}^{\boldsymbol{g}}, \mathbf{F}(\mathbf{u}^{\boldsymbol{g}},\boldsymbol{x}_0)\big)}\leq \epsilon
\end{equation}
for any $\boldsymbol{x}_0\sim p(\boldsymbol{x}_0), \,\, {\boldsymbol{\beta}}\sim p({\boldsymbol{\beta}})$
with 
\begin{multline}\label{equ.lemma2.2}
    \epsilon=M_1\norm{\nabla_{\mathbf{u}} \mathbf{F}(\mathbf{u}^{\boldsymbol{g}},\boldsymbol{x}_0)-\nabla_{\mathbf{u}}\mathbf{G}(\mathbf{u}^{\boldsymbol{g}},\boldsymbol{x}_0)}\\+(L_2+M_gL_1)\norm{ \mathbf{F}(\mathbf{u}^{\boldsymbol{g}},\boldsymbol{x}_0)-\mathbf{G}(\mathbf{u}^{\boldsymbol{g}},\boldsymbol{x}_0)}
\end{multline}

\end{lemma}
\begin{proof}
See Appendix \ref{appendix.1}.
\end{proof}

Lemma \ref{lemma.2} suggests  the $\boldsymbol{g}\text{-MPC}$ data $\mathcal{D}^{\boldsymbol{g}}$ in (\ref{equ.gdata}) can become $\boldsymbol{f}\text{-MPC}$ data $\mathcal{D}^{\boldsymbol{f}}$ in (\ref{equ.fdata}), if the reduced-order model  $\boldsymbol{g}()$ fits well to the true    ${\boldsymbol{f}}()$  in both zeroth and first orders, i.e., 
\begin{align}
         \norm{ \mathbf{F}(\mathbf{u}^{\boldsymbol{g}},\boldsymbol{x}_0)-\mathbf{G}(\mathbf{u}^{\boldsymbol{g}},\boldsymbol{x}_0)} &\rightarrow 0 \label{equ.0th}\\
         \norm{\nabla_{\mathbf{u}} \mathbf{F}(\mathbf{u}^{\boldsymbol{g}},\boldsymbol{x}_0)-\nabla_{\mathbf{u}}\mathbf{G}(\mathbf{u}^{\boldsymbol{g}},\boldsymbol{x}_0)} &\rightarrow 0 \label{equ.1st}
\end{align}
Thus, one can indirectly verify  (\ref{equ.fdata}) by additionally looking at the first-order model precision error $\norm{\nabla_{\mathbf{u}} \mathbf{F}(\mathbf{u}^{\boldsymbol{g}},\boldsymbol{x}_0)-\nabla_{\mathbf{u}}\mathbf{G}(\mathbf{u}^{\boldsymbol{g}},\boldsymbol{x}_0)}$, where $\nabla_{\mathbf{u}} \mathbf{F}(\mathbf{u}^{\boldsymbol{g}},\boldsymbol{x}_0)$ can be estimated numerically via  mesh grid of $\mathcal{D}^{\boldsymbol{g}}$.

Jointly looking at Lemma \ref{lemma.1} and Lemma \ref{lemma.2}, one can conclude that if we can find a reduced-order LCS  $\boldsymbol{g}()$ such that it fits $\boldsymbol{f}()$ well  in both  zeroth and first-order prediction  on the $\boldsymbol{g}\text{-MPC}$ data $\mathcal{D}^{\boldsymbol{g}}$, as  in (\ref{equ.0th}) and (\ref{equ.1st}), respectively, such $\boldsymbol{g}()$ is minimizing the upper bound (\ref{equ.lossupperbound}), thus eventually minimizing the   task performance gap $\mathcal{L}(\boldsymbol{g})$ itself. Those theoretical insights will guide us to develop  algorithms in the next section.

\begin{figure*}[h]
	\centering	\includegraphics[width=0.95\textwidth]{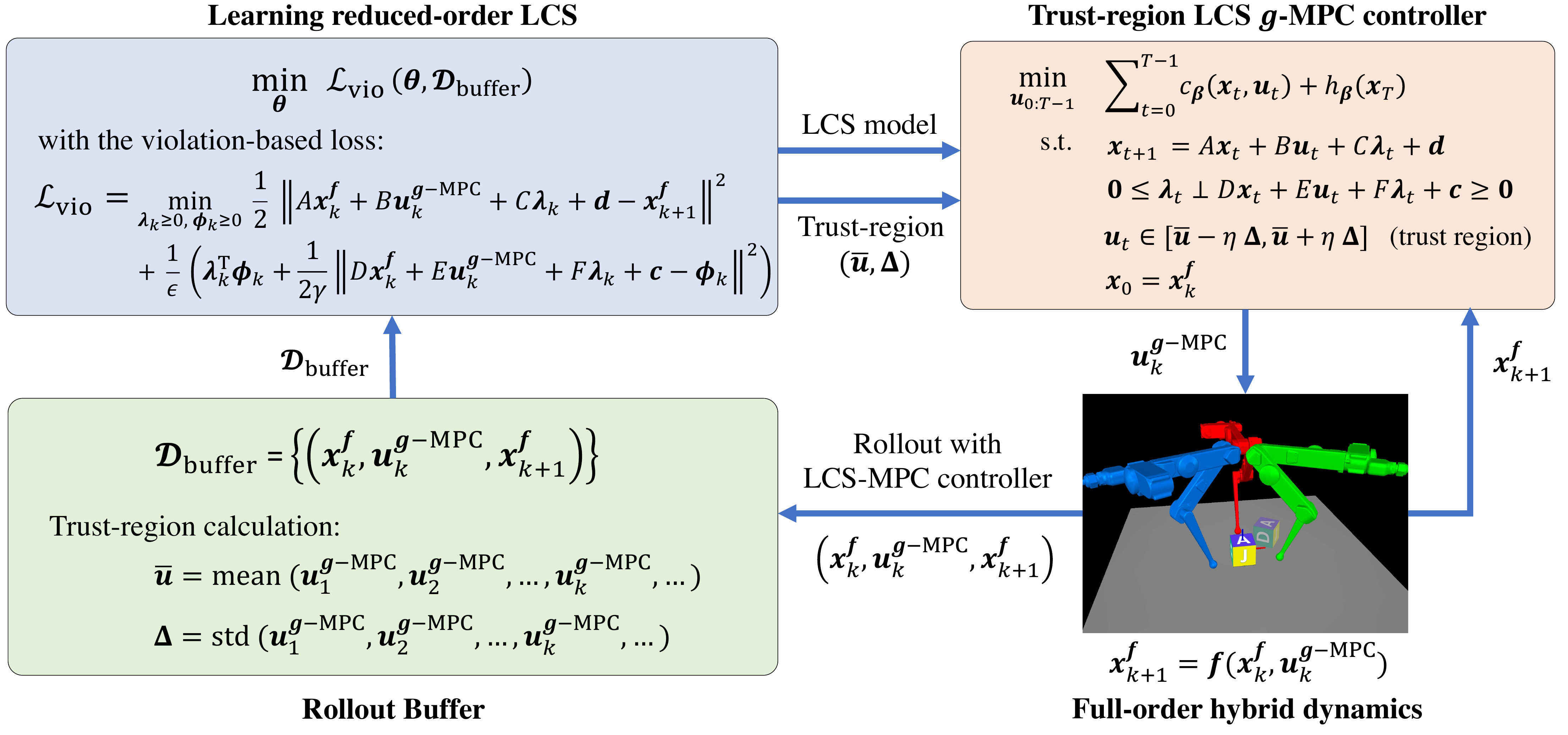}
	\caption{\small Components  of  task-driven hybrid model reduction algorithm. There are three main components: learning reduced-order LCS, Trust-region LCS model predictive controller, and Rollout Buffer, each of which is detailed in the text.}
	\label{fig.algorithm}
  \vspace{-10pt}
\end{figure*}

\medskip
\section{Practical Algorithm} \label{section.algorithm}

The technical analysis in the previous section says that to minimize the upper bound (\ref{equ.lossupperbound}) of the task performance gap, one might  fit a  reduced-order model $\boldsymbol{g}()$ to  full-order dynamics $\boldsymbol{f}()$ well in both  zeroth and first-order prediction using the  $\boldsymbol{g}\text{-MPC}$ policy data  $\mathcal{D}^{\boldsymbol{g}}$.
We take this as inspiration, though we note that, for efficiency, we will minimize zeroth-order error and will not check first-order criteria. Now, we develop the task-driven hybrid-model reduction  algorithm. Throughout the following paper,  $\boldsymbol{g}$-MPC will be implemented in a closed-loop (receding) fashion, i.e., the only first action of the MPC solution is applied to the robot system $\boldsymbol{f}()$.

The  building blocks of the task-driven hybrid model reduction algorithm are  in  Fig. \ref{fig.algorithm}. The learning process is iterative, and each iteration includes the following three components.
\begin{itemize}[leftmargin=10pt]
\setlength\itemsep{3pt}

    \item \textbf{Trust-region LCS  model predictive controller:} The latest LCS  $\boldsymbol{g}()$ is used in the MPC controller. Compared to  (\ref{equ.mpcsimple}), we additionally  introduce a trust region on the control inputs in the reduced-order LCS MPC. This trust region may also be adapted according to the latest Rollout Buffer with details given in Section \ref{section.tr_mpc}.

    \item  \textbf{Rollout Buffer:}  denoted as  $\boldsymbol{\mathcal{D}}_{\text{buffer}}{=}\big\{(\boldsymbol{x}_k^{\boldsymbol{f}}, \boldsymbol{u}_k^{\boldsymbol{g}\text{-MPC}}, \boldsymbol{x}_{k+1}^{\boldsymbol{f}})\big\}$,  stores the current and history rollout data from   running the trust-region  $\boldsymbol{g}\text{-MPC}$ controller on the robot (full-order dynamics). The  buffer  can permit a maximum buffer size.

    \item \textbf{Learning reduce-order LCS:} This  is  to train reduced-order LCS  $\boldsymbol{g}()$ using the  data of the latest Rollout  Buffer  $\boldsymbol{\mathcal{D}}_{\text{buffer}}$. Details of the training process is given in Section \ref{section.trainer}.
    
\end{itemize}

\subsection{Learning Reduced-Order LCS}\label{section.trainer}
We use the method of  our recent work \cite{jin2021learning} to learn the LCS  $\boldsymbol{g}()$  from  Rollout Buffer data
$\boldsymbol{\mathcal{D}}_{\text{buffer}}=\big\{(\boldsymbol{x}_k^{\boldsymbol{f}}, \boldsymbol{u}_k^{\boldsymbol{g}\text{-MPC}}, \boldsymbol{x}_{k+1}^{\boldsymbol{f}})\big\}$. This method enables efficient learning of a PWA model with up to thousands of hybrid modes and effectively handles the stiff dynamics that arises from contact. For self-containment, the method is  described below. 

By learning a  LCS in (\ref{equ.lcs.reduced}), we  mean to learn  all its matrix parameters, denoted as 
\begin{equation}\label{equ.lcs.param}
\boldsymbol{\theta}:=\{
A, B, C, \boldsymbol{d}, D, E, F, \boldsymbol{c}
\}.
\end{equation}
In \cite{jin2021learning}, we presented a new learning method, which  learns  $\boldsymbol{\theta}$ by minimizing the following violation-based loss
\begin{equation}\label{equ.lcs.loss}
    \mathcal{L}_{\text{vio}}(\boldsymbol{\theta}, \boldsymbol{\mathcal{D}}_{\text{buffer}})=\sum_{k}\mathcal{L}_{\text{vio}}\left(\boldsymbol{\theta}, 
    (\boldsymbol{x}_k^{\boldsymbol{f}}, \boldsymbol{u}_k^{\boldsymbol{g}\text{-MPC}},\boldsymbol{x}_{k+1}^{\boldsymbol{f}})\right)
\end{equation}
with 
\begin{equation*}
    \begin{aligned}
    &\mathcal{L}_{\text{vio}}\left(\boldsymbol{\theta}, 
    (\boldsymbol{x}_k^{\boldsymbol{f}}, \boldsymbol{u}_k^{\boldsymbol{g}\text{-MPC}},\boldsymbol{x}_{k+1}^{\boldsymbol{f}})\right):=\\
         &\min_{
         \boldsymbol{\lambda}_k\geq \boldsymbol{0},  \,\,
\boldsymbol{{\boldsymbol{\phi}}}_k\geq \boldsymbol{0}
} 
\,\, \frac{1}{2}\norm{A\boldsymbol{x}_k^{\boldsymbol{f}}+B\boldsymbol{u}_k^{\boldsymbol{g}\text{-MPC}}+C\boldsymbol{\lambda}_k+\boldsymbol{d}-\boldsymbol{x}_{k+1}^{\boldsymbol{f}}}^2+\\ 
&\frac{1}{\epsilon}
\left(
\boldsymbol{\lambda}_k\tran\boldsymbol{{\boldsymbol{\phi}}}_k+ \frac{1}{2\gamma}\norm{
D\boldsymbol{x}_k^{\boldsymbol{f}}+E\boldsymbol{u}_k^{\boldsymbol{g}\text{-MPC}} +F\boldsymbol{\lambda}_k+\boldsymbol{c}-\boldsymbol{\phi}_k
}^2
\right).
    \end{aligned}
\end{equation*}
In the above loss $\mathcal{L}_{\text{vio}}\left(\boldsymbol{\theta}, 
    (\boldsymbol{x}_k^{\boldsymbol{f}}, \boldsymbol{u}_k^{\boldsymbol{g}\text{-MPC}},\boldsymbol{x}_{k{+}1}^{\boldsymbol{f}})\right)$, the first and second terms are the violation of the affine dynamics and complementarity equations by a buffer data point $(\boldsymbol{x}_k^{\boldsymbol{f}}, \boldsymbol{u}_k^{\boldsymbol{g}\text{-MPC}}, \boldsymbol{x}_{k+1}^{\boldsymbol{f}})$, respectively. 
$\epsilon>0$, which empirically takes its value from the range $(10^{-3},1)$, is a hyperparameter that balances the violation of these two terms.  Here, $\boldsymbol{\phi}\in\mathbb{R}^r$ is an introduced slack variable for the complementarity equation, and  $\gamma>0$ can be any value as long as satisfying  $\gamma\leq \sigma_{\min}(F\tran+F)$ (i.e., the smallest singular value of the matrix $(F\tran+F)$). Such a choice of $\gamma$ ensures the strong convexity of the  quadratic objective  $\mathcal{L}_{\text{vio}}\left(\boldsymbol{\theta}, 
    (\boldsymbol{x}_k^{\boldsymbol{f}}, \boldsymbol{u}_k^{\boldsymbol{g}\text{-MPC}},\boldsymbol{x}_{k{+}1}^{\boldsymbol{f}})\right)$ in the variable $(\boldsymbol{\lambda}_k, \boldsymbol{\phi}_k)$.

As theoretically shown  in \cite{jin2021learning}, the above LCS learning loss $\mathcal{L}_{\text{vio}}(\boldsymbol{\theta}, \boldsymbol{\mathcal{D}}_{\text{buffer}})$ has the following properties. 
First, it can be proved that the  inner optimization over $(\boldsymbol{\lambda}_k, \boldsymbol{{\boldsymbol{\phi}}}_k)$ is a convex quadratic program, thus can be efficiently solved in batch using state-of-the-art  solvers, e.g., OSQP \cite{osqp}. Second, the gradient of the  violation-based loss  $\mathcal{L}_{\text{vio}}(\boldsymbol{\theta}, \boldsymbol{\mathcal{D}}_{\text{buffer}})$  with respect to all  matrices in $\boldsymbol{\theta}$ can be analytically obtained using the Envelope Theorem \cite{afriat1971theory} (without differentiating through the solution to the inner optimization). 
\noindent
 Third,  by  adding both the affine dynamics violation  and complementarity violation with a balance weight $\epsilon$, $\mathcal{L}_{\text{vio}}(\boldsymbol{\theta}, \boldsymbol{\mathcal{D}}_{\text{buffer}})$  attains a better conditioned loss landscape,  enabling simultaneous identification of stiff and multi-modal dynamics.

\subsection{Trust-Region LCS Model Predictive Controller} \label{section.tr_mpc}

With the  reduced-order LCS $\boldsymbol{g}()$, one can  establish the following trust-region reduced-order  LCS-based MPC:
\begin{equation}\label{equ.mpcsimple.tr}
\begin{aligned}
        \min_{\boldsymbol{{u}}_{0:T-1}} &\quad \sum_{t=0}^{T-1}c_{\boldsymbol{{\boldsymbol{\beta}}}}(\boldsymbol{{x}}_t, \boldsymbol{{u}}_t)+h_{\boldsymbol{\beta}}(\boldsymbol{{x}}_{T})\qquad\boldsymbol{\beta}\sim p(\boldsymbol{\beta})\\
         \text{subject to} 
                 &\quad \boldsymbol{{u}}_t\in [\boldsymbol{\bar{u}}-\boldsymbol{\Delta},  \boldsymbol{\bar{u}}+\boldsymbol{\Delta}],\\
         &\quad \boldsymbol{x}_{t+1}=A\boldsymbol{x}_t+B\boldsymbol{u}_t+C\boldsymbol{\lambda}_t+\boldsymbol{d},\\
          &\quad\boldsymbol{0} \leq\boldsymbol{\lambda}_t\perp D\boldsymbol{x}_t+E\boldsymbol{u}_t +F\boldsymbol{\lambda}_t+\boldsymbol{c} \geq \boldsymbol{0},
\\
         &\quad\boldsymbol{x}_0=\boldsymbol{x}_k^{\boldsymbol{f}}.
\end{aligned}
\end{equation}

Compared to the  early $\boldsymbol{g}\text{-MPC}$ in (\ref{equ.mpcsimple}), the  difference here is that we have enforced a trust region constraint $\boldsymbol{\bar{u}}-\boldsymbol{\Delta}\leq \boldsymbol{u}_t\leq   \boldsymbol{\bar{u}}+\boldsymbol{\Delta}$ on the control input $\boldsymbol{u}_t$, $t=0,1,\dots, T-1$. This is due to the following reasons. As shown in Fig. \ref{fig.algorithm}, since  $\boldsymbol{g}()$ is trained on the current  buffer data $\boldsymbol{\mathcal{D}}_{\text{buffer}}$, we expect  $\boldsymbol{g}()$ is likely valid only on the region covered by $\boldsymbol{\mathcal{D}}_{\text{buffer}}$, which we refer to as the \emph{trust region}. Thus,  we constrain  $\boldsymbol{g}$-MPC in (\ref{equ.mpcsimple.tr}) to this trust region, prohibiting the controller from attempting to exploit model error and generating undesired controls.

The center $\boldsymbol{\bar{u}}$ and size $\boldsymbol{\Delta}$ of the trust region  may  be updated along with the rollout buffer $\boldsymbol{\mathcal{D}}_{\text{buffer}}$ during each iteration. In our algorithm, at $i$-th iteration, we simply set the trust region  center $\boldsymbol{\bar{u}}_i$ as the mean of all control input data in the current Rollout Buffer $\boldsymbol{\mathcal{D}}_{\text{buffer},i}$, i.e., 
\begin{multline}\label{equ.trustregion.center}
    \boldsymbol{\bar{u}}_{i}= \text{mean}\big(\{
    \boldsymbol{u}^{\boldsymbol{g}\text{-MPC}}_1,  
    ..., 
    \boldsymbol{u}^{\boldsymbol{g}\text{-MPC}}_k, 
    ...
    \}\big),\\  \boldsymbol{u}^{\boldsymbol{g}\text{-MPC}}_k\in \boldsymbol{\mathcal{D}}_{\text{buffer}, i}
\end{multline}
and the trust region size  $\boldsymbol{\Delta}_i$  is set according to the standard deviation of all input data in $\boldsymbol{\mathcal{D}}_{\text{buffer},i}$:
\begin{multline}\label{equ.trustregion.radiusupdate}
    \mathbf{\Delta}_{i} =\eta_i \,\,\text{std}\big(\{
    \boldsymbol{u}^{\boldsymbol{g}\text{-MPC}}_1, 
    ..., 
    \boldsymbol{u}^{\boldsymbol{g}\text{-MPC}}_k, 
    ...
    \}\big),\\ \boldsymbol{u}^{\boldsymbol{g}\text{-MPC}}_k\in \boldsymbol{\mathcal{D}}_{\text{buffer}, i}
\end{multline}
Here, $\eta_i>0$ a hyperparameter of the trust region at the $i$-th iteration, and $\text{mean}()$ and $\text{std}()$  are applied  dimension-wise. It is also possible that $\boldsymbol{\bar{u}}_i$ and  $\boldsymbol{\Delta}_i$ are set using other rules, e.g., following the classic trust-region optimization \cite{yuan2000review}.

To  solve the  LCS MPC in (\ref{equ.mpcsimple.tr}), we adopt the direct method of trajectory optimization  \cite{posa2014direct}. Specifically, the optimization  simultaneously searches over the trajectories  $\boldsymbol{x}_{0:T}$, $\boldsymbol{u}_{0:T-1}$, $\boldsymbol{\lambda}_{0:T-1}$ by treating the LCS  and  trust region  as the separate constraints imposed at each time step. We solve  such nonlinear optimization using CasADi \cite{andersson2019casadi}  interface packed with  IPOPT solver \cite{wachter2006implementation}. In our later applications, as the reduced-order LCS in (\ref{equ.mpcsimple.tr}) has a relatively small number of hybrid modes, e.g.,  $\dim \boldsymbol{\lambda}\leq 5$ and a small MPC horizon $T=5$, we can solve (\ref{equ.mpcsimple.tr}) with a real-time MPC  performance (e.g., MPC running frequency can reach $50$Hz). 

We summarize the algorithm of task-driven hybrid model reduction  in Algorithm \ref{algorithm1}. Here,  subscript ${i}$ denotes the learning iteration. At  initialization, the Rollout Buffer $\boldsymbol{\mathcal{D}}_{\text{buffer}, 0}$ can be filled with data  collected from running random policies on the robot $\boldsymbol{f}()$.

\begin{algorithm2e}[th]
	\small 
	\SetKwInput{Parameter}{Hyperparameter}

	\SetKwInput{Initialization}{Initialization}
	\Initialization{
	Initial reduced-order LCS model $\boldsymbol{g}_{\boldsymbol{\theta}_{0}}$;\\
	\hspace{59pt}Initial Buffer $\boldsymbol{\mathcal{D}}_{\text{buffer},0}$ (by random policy);\\
	\hspace{59pt}Trust region parameter schedule $\{{\eta_i}\}$
	}
	\smallskip
	\For{$i=0,1,2,\cdots$}{

		\medskip
		\tcc{Reduced-order model update}
	    Train   reduced-order LCS  $\boldsymbol{g}_{\boldsymbol{\theta}_{i}}$ with the data from  current Rollout Buffer  $\boldsymbol{\mathcal{D}}_{\text{buffer}, i}$:
	    $\boldsymbol{g}_{\boldsymbol{\theta}_{i+1}}\leftarrow\boldsymbol{g}_{\boldsymbol{\theta}_{i}}$ [Section \ref{section.trainer}]

		\medskip
		\tcc{Set the trust region}
		Set the trust region   from the current  Rollout Buffer  $\boldsymbol{\mathcal{D}}_{\text{buffer}, i}$: 
		$[\boldsymbol{\bar{u}}_i-\boldsymbol{\Delta}_i,  \boldsymbol{\bar{u}}_i+\boldsymbol{\Delta}_i]$
		[see (\ref{equ.trustregion.center}) and (\ref{equ.trustregion.radiusupdate})]
		;\\

        \medskip
		\tcc{MPC rollout and update Buffer}
		With the current  LCS $\boldsymbol{g}_{\boldsymbol{\theta}_{i+1}}$ and  current  trust region $[\boldsymbol{\bar{u}}_i-\boldsymbol{\Delta}_i,  \boldsymbol{\bar{u}}_i+\boldsymbol{\Delta}_i]$, run the trust-region LCS MPC policy in (\ref{equ.mpcsimple.tr}) on the robot,  collect new rollout data $\{(\boldsymbol{x}_{k}^{\boldsymbol{f}}, \boldsymbol{u}_{k}^{\boldsymbol{g}\text{-MPC}}, \boldsymbol{x}_{k+1}^{\boldsymbol{f}})\}$  and add it to Rollout Buffer:
		$\boldsymbol{\mathcal{D}}_{\text{buffer},{i+1}}\leftarrow\boldsymbol{\mathcal{D}}_{\text{buffer},i}\cup\{(\boldsymbol{x}_{k}^{\boldsymbol{f}}, \boldsymbol{u}_{k}^{\boldsymbol{g}\text{-MPC}}, \boldsymbol{x}_{k+1}^{\boldsymbol{f}})\}$;
		
	}
	\caption{Task-driven hybrid model reduction} \label{algorithm1}
\end{algorithm2e}

\section{Model  Reduction for Hybrid Control Systems}\label{section.pwa_reduction}
In this section, we  will use the  proposed method to solve  model reduction of  synthetic hybrid  systems of varying dimension. Examples are written in  Python, available at \url{https://github.com/wanxinjin/Task-Driven-Hybrid-Reduction}.

\subsection{Problem Setting}
Consider  the  MPC of a general PWA system \cite{marcucci2019mixed}
\begin{equation}\label{equ.pwa.mpc}
\begin{aligned}
       \min_{\substack{\boldsymbol{u}_{0:T-1}\\ \boldsymbol{x}_{0:T}}} \,\,\, &J=\sum_{t=0}^{T-1}c(\boldsymbol{{x}}_t, \boldsymbol{{u}}_t)+h(\boldsymbol{{x}}_{T})\\
       \text{s.t.} 
       \,\,\,& \boldsymbol{f}():\begin{cases}
                 \boldsymbol{x}_{t+1}= {A}_j\boldsymbol{x}_t+{B}_j\boldsymbol{u}_t+{\boldsymbol{c}}_j,\,\,\,  (\boldsymbol{x}_t,\boldsymbol{u}_t)\in \mathcal{P}_j,\\
           \mathcal{P}_j=\{(\boldsymbol{x},\boldsymbol{u})\,|\, D_j\boldsymbol{x}+ E_j\boldsymbol{u}+\boldsymbol{h}_j\leq \boldsymbol{0}\},\\
          j\in \{1, 2,\dots, I\},
        \end{cases}\\
        &\boldsymbol{x}_0 \,\,\,\text{given}.
\end{aligned}
\end{equation}
Here,  $\mathcal{P}_j$, $ j\in \{1, 2,\dots, I\}$, is the $j$-th partition of the state-input space, with dynamics $\boldsymbol{x}_{t+1}= {A}_j\boldsymbol{x}_t+{B}_j\boldsymbol{u}_t+{\boldsymbol{c}}_j$. The total number of hybrid modes of the above PWA system is $I$.
Solving (\ref{equ.pwa.mpc}) is generally treated as  a mixed-integer program with $I^T$ possible mode sequences. This exponential scaling quickly becomes computationally intractable as $I$ and $T$ grow.

In the following, we aim to find a reduced-order LCS model $\boldsymbol{g}()$ which maintains a small budget of hybrid modes, such that running $\boldsymbol{g}$-MPC  can  achieve similar performance as  running the  full-order $\boldsymbol{f}$-MPC in (\ref{equ.pwa.mpc}). 
Here, the reduced-order LCS $\boldsymbol{g}()$ in (\ref{equ.lcs.reduced}) has the same dimensions of $\boldsymbol{x}$ and $\boldsymbol{u}$ as $\boldsymbol{f}()$,  but we set 
$\dim\boldsymbol{\lambda}$ such that 
its maximum number of  hybrid modes of $\boldsymbol{g}()$ is far less than $\boldsymbol{f}()$'s, i.e., $2^{\dim\boldsymbol{\lambda}}\ll I$.

\begin{figure*}[h]
	\centering
		\begin{subfigure}{.244\textwidth}
		\centering
		\includegraphics[width=\linewidth]{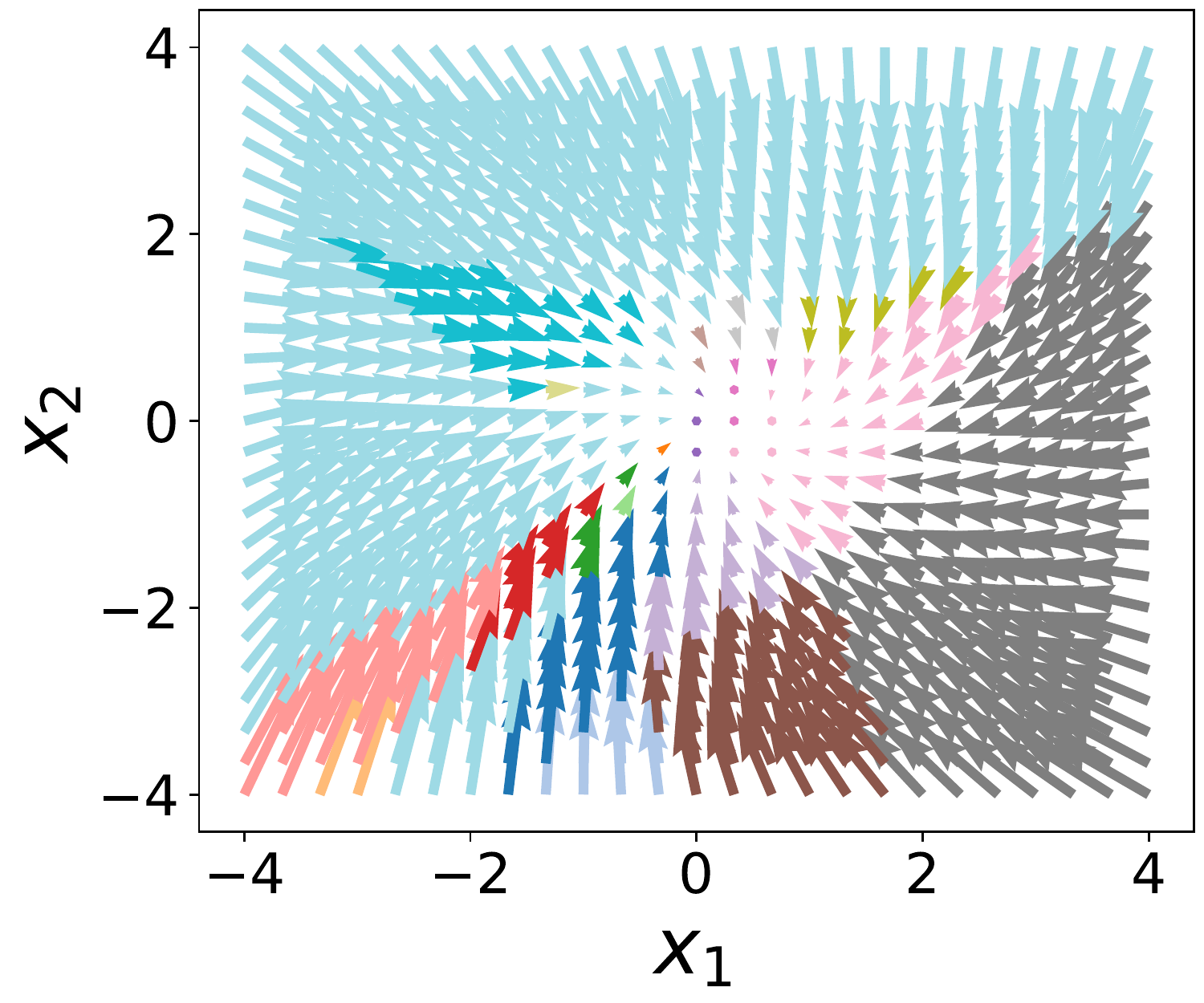}
		\caption{Phase portrait for $\boldsymbol{f}$-MPC}
		\label{fig.lcs.example1.1}
	\end{subfigure}
	\begin{subfigure}{.244\textwidth}
		\centering
		\includegraphics[width=\linewidth]{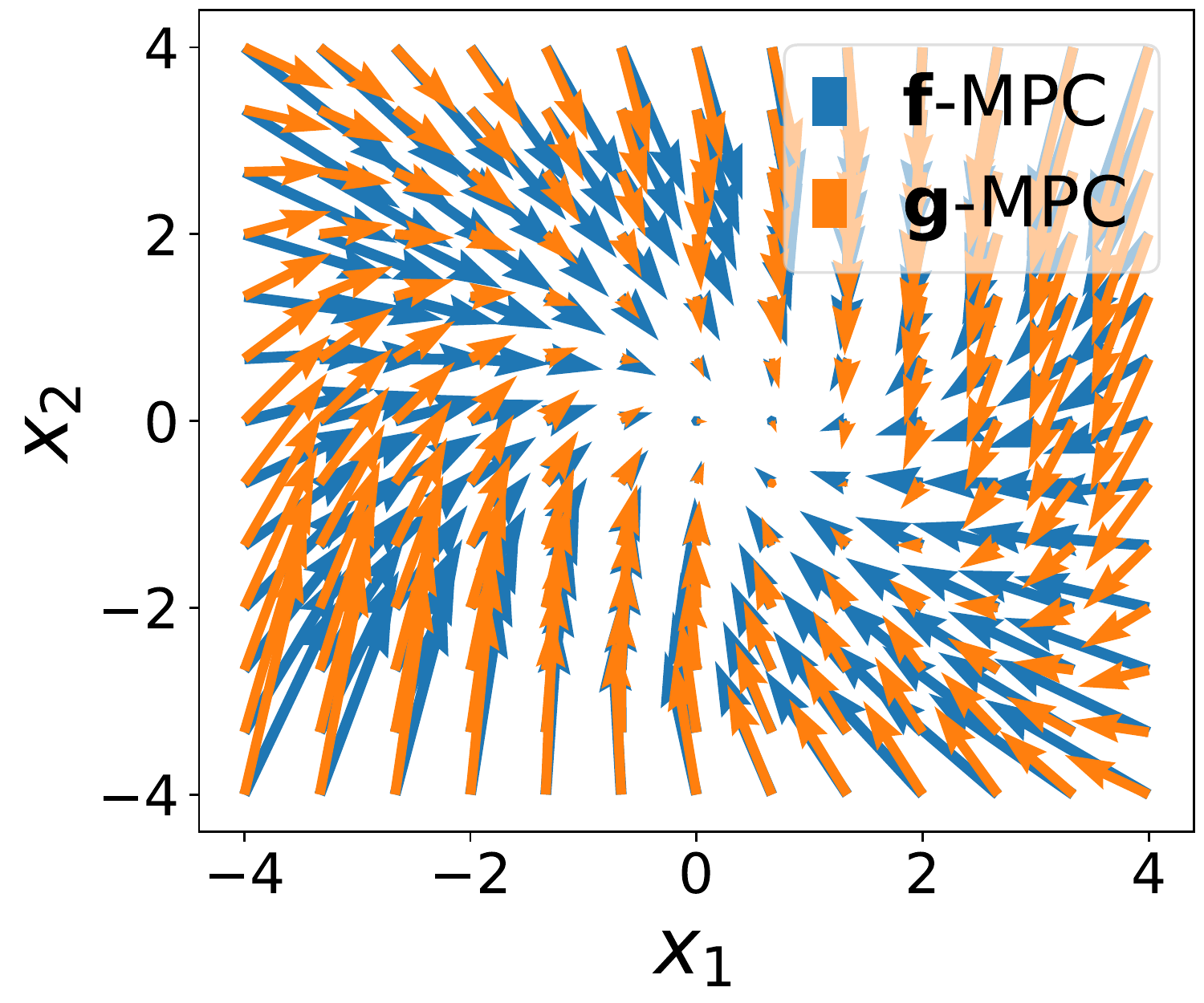}
		\caption{$\boldsymbol{f}$-MPC vs. $\boldsymbol{g}$-MPC at Iter. 0}
		\label{fig.lcs.example1.2}
	\end{subfigure}
	\begin{subfigure}{.244\textwidth}
		\centering
		\includegraphics[width=\linewidth]{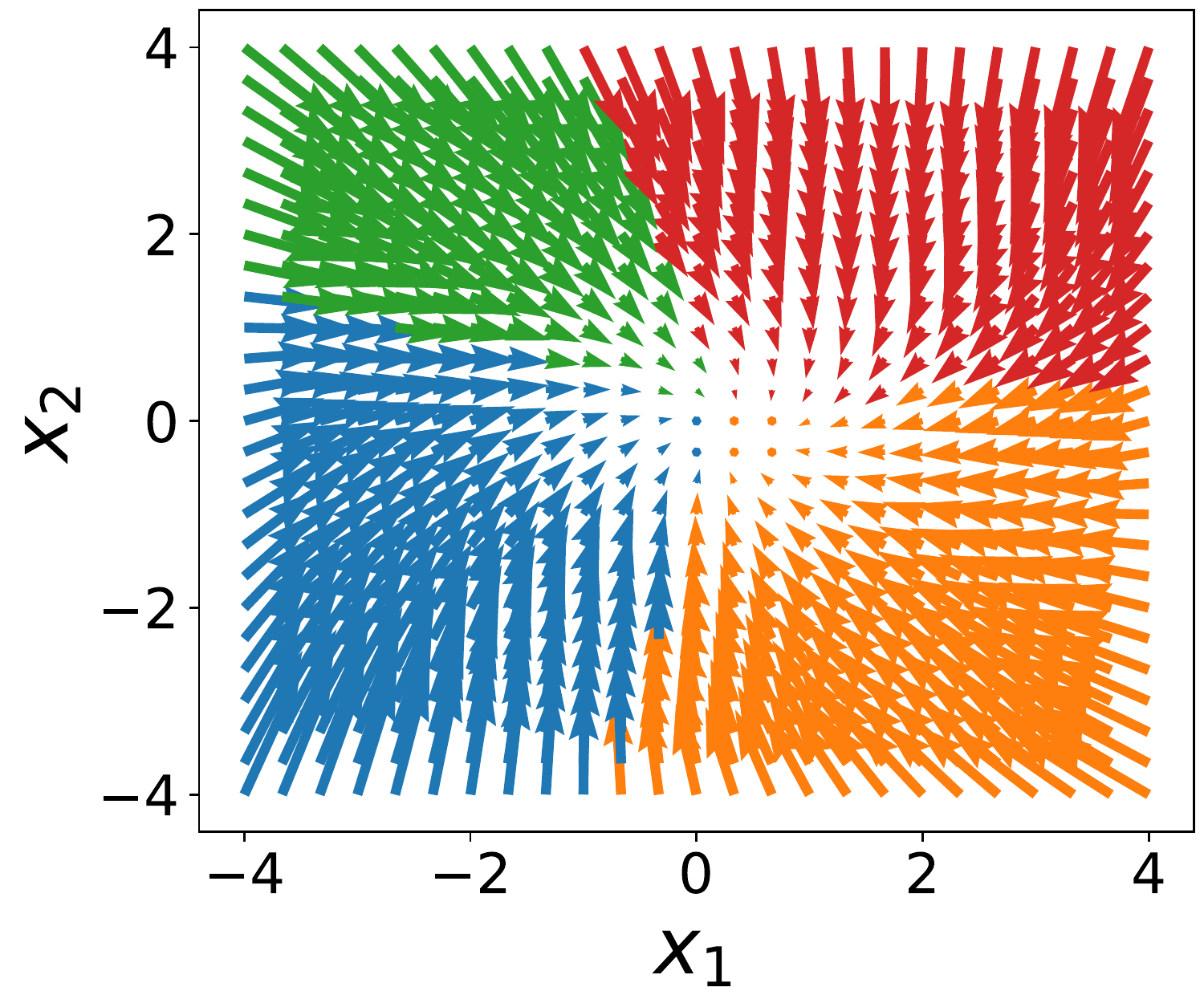}
		\caption{Phase  for $\boldsymbol{g}$-MPC  at Iter. 24}
		\label{fig.lcs.example1.3}
	\end{subfigure}
	\begin{subfigure}{.244\textwidth}
		\centering
		\includegraphics[width=\linewidth]{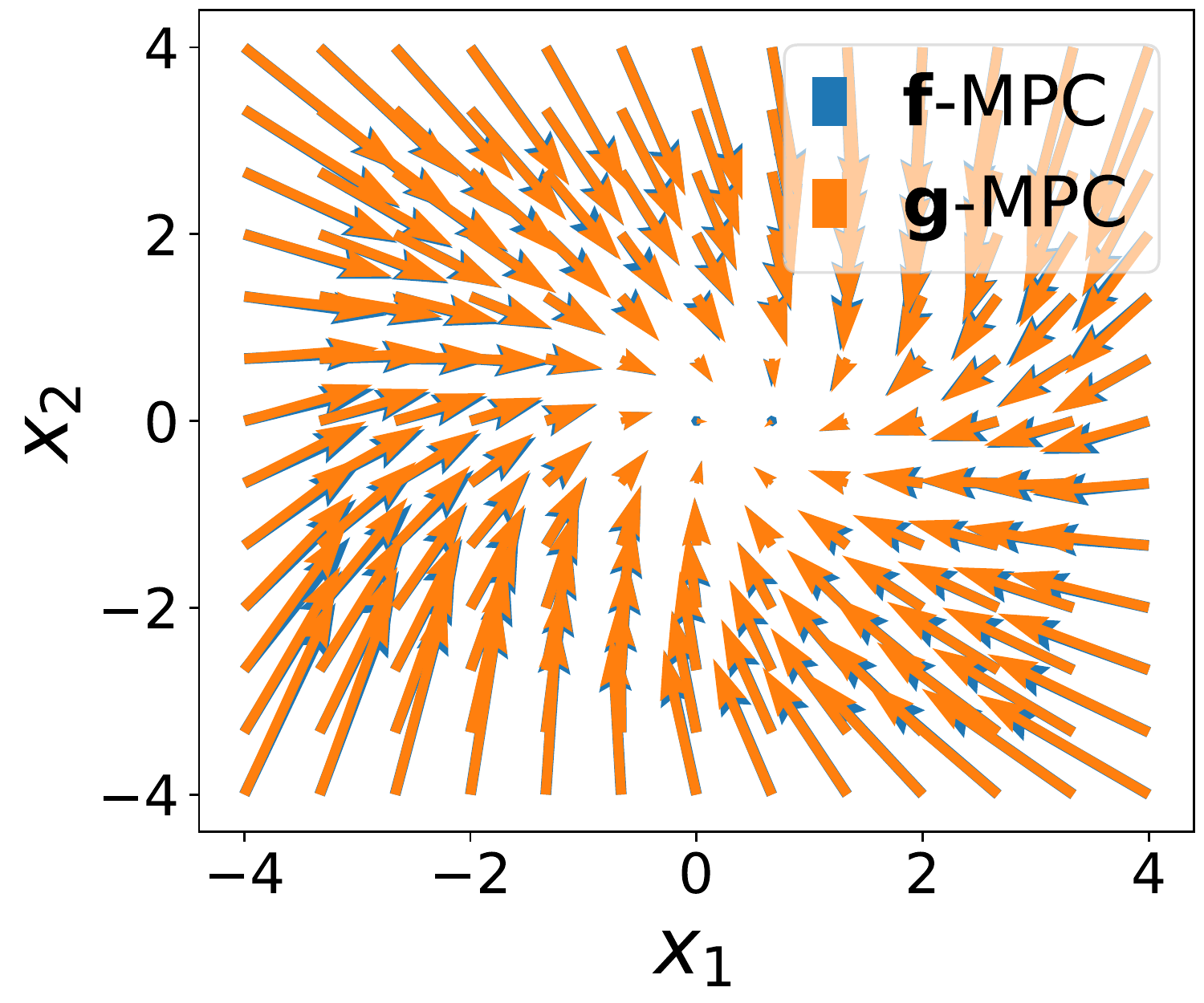}
		\caption{$\boldsymbol{f}$-MPC vs. $\boldsymbol{g}$-MPC at Iter. 24}
		\label{fig.lcs.example1.4}
	\end{subfigure}
	\caption{\small Phase portraits of the MPC-controlled full-order dynamics $\boldsymbol{x}_{t+1}=\boldsymbol{f}\big(\boldsymbol{x}_t, \text{MPC}(\boldsymbol{x}_t)\big)$, where the  controller  $\boldsymbol{u}_t=\text{MPC}(\boldsymbol{x}_t)$ can be either   full-order $\boldsymbol{f}$-MPC or   reduced-order $\boldsymbol{g}$-MPC.  (a) is the phase portrait for the $\boldsymbol{f}$-MPC controller, where different colors indicate different hybrid  modes (42  modes here) in $\boldsymbol{f}()$; (b) is the phase  comparison between using $\boldsymbol{f}$-MPC  and  $\boldsymbol{g}$-MPC controllers at  learning iteration 0; (c) is  the phase portrait for  $\boldsymbol{g}$-MPC controller at learning iteration  24, where different colors indicates different hybrid modes (4  modes here) in $\boldsymbol{g}()$; and (f) is the phase  comparison between the $\boldsymbol{f}$-MPC controller and  $\boldsymbol{g}$-MPC controller at learning iteration 24.
	} 
	\label{fig.lcs.example1}
  \vspace{-10pt}
\end{figure*}

\subsection{Experiment Settings}
We consider the task of stabilizing the hybrid system to a  stationary state (zeros), and thus  set the cost function $J$ in  (\ref{equ.pwa.mpc}) as a  quadratic cost function:
\begin{equation}\label{equ.pwa.cost}
    J=\sum_{t=0}^{T-1} \left(\boldsymbol{x}_t\tran Q\boldsymbol{x}_t+\boldsymbol{u}_t\tran R\boldsymbol{u}_t\right) + \boldsymbol{x}_{T}\tran Q_T\boldsymbol{x}_{T},
\end{equation}
with all weight matrices being identities. We run both   $\boldsymbol{f}\text{-MPC}$  and $\boldsymbol{g}\text{-MPC}$ policies on  full-order  dynamics $\boldsymbol{f}()$ in a closed-loop (receding) fashion, though noting that $\boldsymbol{f}\text{-MPC}$  cannot be solved in real-time for our more complex examples. The initial state $\boldsymbol{x}_0$ of the full-order  dynamics $\boldsymbol{f}()$ is subject to a uniform distribution $\boldsymbol{x}_0\sim U[-4, 4]$. 

In Algorithm \ref{algorithm1}, the hyperparameters are listed in Table \ref{table.pwa.setting}.  An ablation study about how the hyperparameters influence the  performances will be given later in Section \ref{section.pwa.param}.   For the hyperparameter setting in learning LCS,  please refer to our previous paper  \cite{jin2021learning}.

\begin{table}[h]
\begin{center}
\caption{Algorithm hyperparameters for model reduction of synthetic hybrid control systems.}
\label{table.pwa.setting}
\begin{threeparttable}
\begin{tabular}{l l l}
    \toprule
    \textbf{Parameter\tnote{1}} & \textbf{Symbol} & \textbf{Value}  \\ 
    \midrule
      MPC horizon   &        $T$ &5   \\
      Rollout horizon  &     $H$ & $15\sim20$  \\ 
      \# of new  rollouts added to buffer per iter. &  $R_{\text{new}}$ & 5    \\
      Maximum buffer size &    $R_{\text{buffer}}$ & 50 rollouts    \\
      Trust region hyperparameter  & $\eta_i$ & 20,  $\forall i$  \\
      Initial guess $\boldsymbol{\theta}$ in (\ref{equ.lcs.param})  for  $\boldsymbol{g}()$ &  $\boldsymbol{\theta}_0$ & $U[-0.5, 0.5]$\tnote{2} \\ 
      \bottomrule
\end{tabular}
\begin{tablenotes}
\item[1] Other settings not listed here will be stated in text. 
\item[2]  $U[-0.5, 0.5]$ means uniform distribution in range $[-0.5, 0.5]$.
\end{tablenotes}
\end{threeparttable}
\end{center}
\end{table}

\subsection{Results and Analysis}
\subsubsection{Illustration of Learning Progress} 
We  randomly generate full-order dynamics  $\boldsymbol{f}()$ in (\ref{equ.pwa.mpc}). Specifically, all matrices $(A_j, B_j, \boldsymbol{c}_j, D_j, E_j, \boldsymbol{h}_j)$, $i=1,2, \dots, I$, are sampled from uniform distributions, with dimension $\boldsymbol{x}\in \mathbb{R}^2$ and  $\boldsymbol{u}\in\mathbb{R}$, and  mode count  $I\approx120$ for random sampling of $\boldsymbol{x}_{0}\sim U[-4, 4]$ and $\boldsymbol{u}\sim U[-10, 10]$.
In the reduced-order LCS   $\boldsymbol{g}()$ in (\ref{equ.lcs.reduced}), we take $\dim \boldsymbol{\lambda}=2$,
meaning that the maximum number of  modes in $\boldsymbol{g}()$ is $4$, far fewer than $I$ of the full-order dynamics.

We plot the learning progress (iteration) for the task-driven reduced-order model $\boldsymbol{g}()$ in Fig. \ref{fig.lcs.example1}. Here, we show the   phase portraits of the MPC-controlled full-order dynamics:
\begin{equation}
    \boldsymbol{x}_{t+1}=\boldsymbol{f}\big(\boldsymbol{x}_t, \boldsymbol{u}_t\big)=\boldsymbol{f}\big(\boldsymbol{x}_t, \text{MPC}(\boldsymbol{x}_t)\big)
\end{equation}
where   the MPC controller  $\boldsymbol{u}_t=\text{MPC}(\boldsymbol{x}_t)$ can be either the full-order $\boldsymbol{f}$-MPC  (\ref{equ.pwa.mpc}) or the learned reduced-order $\boldsymbol{g}$-MPC. Specifically, Fig. \ref{fig.lcs.example1.1} shows the phase portrait of    $\boldsymbol{f}$-MPC controller, where different colors show different  hybrid modes in  $\boldsymbol{f}()$. Fig. \ref{fig.lcs.example1.2}  shows the phase portrait comparisons between  $\boldsymbol{f}$-MPC controller (blue) and $\boldsymbol{g}$-MPC controller (orange) before learning. Fig. \ref{fig.lcs.example1.3}
shows the phase portrait for $\boldsymbol{g}$-MPC controller after learning, where different modes in $\boldsymbol{g}$-MPC are shown in different colors.  Fig. \ref{fig.lcs.example1.4}
compares the phase plot between $\boldsymbol{f}$-MPC  (blue) and $\boldsymbol{g}$-MPC (orange) controllers  after learning.

Although the full-order dynamics $\boldsymbol{f}()$ has around $I=120$  modes for  random data  $\boldsymbol{x}\sim U[-4, 4]$ and $\boldsymbol{u}\sim U[-10, 10]$, Fig. \ref{fig.lcs.example1.1}  shows 42  hybrid modes in $\boldsymbol{f}()$ with  $\boldsymbol{f}$-MPC controller. One can notice that some  modes correspond to  a small portion of the state space, e.g.,  \textcolor{orange}{orange}  and \textcolor{green}{green}  (near origin),  and thus, most of $\boldsymbol{f}$'s task-relevant  motion (flows) will not enter into or  quickly pass those modes. This makes those  modes less important for the 
 task of minimizing (\ref{equ.pwa.cost}). On the other hand, some other  modes account for a large portion of the state space, such as \textcolor{cyan}{cyan} and \textcolor{gray}{gray}. Most of $\boldsymbol{f}$'s  motion will enter into or stay in those modes, making them  dominant for the minimizing the task cost (\ref{equ.pwa.cost}). 
In Fig. \ref{fig.lcs.example1.3}, after learning, the reduced-order model $\boldsymbol{g}()$ has only 4 hybrid modes (recall $\dim\boldsymbol{\lambda}=2$), which successfully capture the important modes in Fig. \ref{fig.lcs.example1.1}. Comparing the  phase portrait of the full-order $\boldsymbol{f}$-MPC controller and that of  the  reduced-order  $\boldsymbol{g}$-MPC  controller in Fig. \ref{fig.lcs.example1.4}, we see a  similar control performance. Thus, one can conclude  that the proposed method  learns a task-driven reduced-order model for the hybrid system.

\begin{table*}[h]
\begin{center}
\caption{Task-driven model reduction for hybrid control systems}
\label{table.pwa.ex2}
\begin{threeparttable}
\begin{tabular}{ccccccccc}
    \toprule
    \multirow{2}{*}{Case} & 
    \multirow{2}{*}{\begin{tabular}{@{}l@{}}System \\ dimension \end{tabular}} &
    \multirow{2}{*}{{\begin{tabular}{@{}l@{}}Mode reduction \\ $\dim \boldsymbol{\Lambda} \rightarrow\dim \boldsymbol{\lambda}$\end{tabular}}} &
    \multicolumn{2}{c}{\textbf{Random Policy}} & 
    \multicolumn{3}{c}{\textbf{$\boldsymbol{g}$-MPC Policy}} &
    \multirow{2}{*}{$\mathcal{L}(\boldsymbol{g})$ (\%) } 
    \\
    \cmidrule(rl){4-5} \cmidrule(rl){6-8}
     & & &  
     {\# of modes  in $\boldsymbol{f}$}&
     {ME$(\boldsymbol{g})$ (\%)}  &
     {\# of modes  in $\boldsymbol{f}$}&
     {ME$(\boldsymbol{g})$ (\%)}    & 
     {\# of modes  in $\boldsymbol{g}$} &
    \\ 
    \midrule
    1 & 
    {\begin{tabular}{@{}l@{}} $\dim\boldsymbol{x}=6$\\ $\dim\boldsymbol{u}={2}$\end{tabular}}   &
    {\begin{tabular}{@{}l@{}} $\dim\boldsymbol{\Lambda}=8$ \\ $\rightarrow\dim\boldsymbol{\lambda}=3$ \end{tabular}}     &
    {\begin{tabular}{@{}l@{}}   187.3 \\ $\pm$ 14.0      \end{tabular}}&
    {\begin{tabular}{@{}l@{}}    33.0\%     \\     $\pm$ 13.9\%      \end{tabular}}     &
    \begin{tabular}{@{}l@{}}   18.4 \\ $\pm$ 2.8     \end{tabular}&
    {\begin{tabular}{@{}l@{}}     0.5\%     \\     $\pm$ 0.2\%      \end{tabular}}  & 
    \begin{tabular}{@{}l@{}}   6.2 \\ $\pm$ 1.2     \end{tabular}&
     {\begin{tabular}{@{}l@{}}     0.1\%     \\      $\pm$ 0.1\%      \end{tabular}}
     \\[10pt]
     2 & 
     {\begin{tabular}{@{}l@{}} $\dim\boldsymbol{x}={10}$\\ $\dim\boldsymbol{u}={3}$\end{tabular}} & 
     {\begin{tabular}{@{}l@{}} $\dim\boldsymbol{\Lambda}=12$ \\ $\rightarrow\dim\boldsymbol{\lambda}=3$ \end{tabular}}     &
         {\begin{tabular}{@{}l@{}}   $1090.0$ \\ $\pm$ 133.2     \end{tabular}}&
    {\begin{tabular}{@{}l@{}}    29.8\%     \\     $\pm$ 13.0\%      \end{tabular}}      & 
     \begin{tabular}{@{}l@{}}   29.9 \\ $\pm$ 2.5    \end{tabular}&
     {\begin{tabular}{@{}l@{}}     1.0\%    \\    $\pm$0.1\%     \end{tabular}} & 
     \begin{tabular}{@{}l@{}}   6.7 \\ $\pm$ 1.0     \end{tabular}&
     {\begin{tabular}{@{}l@{}}     0.5\%    \\    $\pm$0.2\%     \end{tabular}}
     \\[10pt]
              3 & 
     {\begin{tabular}{@{}l@{}} $\dim\boldsymbol{x}={20}$\\ $\dim\boldsymbol{u}={3}$\end{tabular}} & 
     {\begin{tabular}{@{}l@{}} $\dim\boldsymbol{\Lambda}=15$ \\ $\rightarrow\dim\boldsymbol{\lambda}=1$ \end{tabular}} &
              {\begin{tabular}{@{}l@{}}   2686.2 \\ $\pm$ 197.3      \end{tabular}}&
    {\begin{tabular}{@{}l@{}}    16.8\%     \\     $\pm$ 5.7\%      \end{tabular}}     & 
     \begin{tabular}{@{}l@{}}    50.0 \\ $\pm$  4.7   \end{tabular}&
     {\begin{tabular}{@{}l@{}}     2.1\%    \\    $\pm$0.3\%     \end{tabular}} &
          \begin{tabular}{@{}l@{}}   2.0 \\ $\pm$ 0.0    \end{tabular}&
     {\begin{tabular}{@{}l@{}}     1.1\%    \\    $\pm$0.5\%     \end{tabular}} 
     \\[10pt]
     4 & 
     {\begin{tabular}{@{}l@{}} $\dim\boldsymbol{x}={20}$\\ $\dim\boldsymbol{u}={3}$\end{tabular}} & 
     {\begin{tabular}{@{}l@{}} $\dim\boldsymbol{\Lambda}=15$ \\ $\rightarrow\dim\boldsymbol{\lambda}=2$ \end{tabular}} &
              {\begin{tabular}{@{}l@{}}   2869.2 \\ $\pm$ 165.0      \end{tabular}}&
    {\begin{tabular}{@{}l@{}}    17.5\%     \\     $\pm$ 6.5\%      \end{tabular}}     & 
     \begin{tabular}{@{}l@{}}    52.3 \\ $\pm$  4.5   \end{tabular}&
     {\begin{tabular}{@{}l@{}}     1.9\%    \\    $\pm$0.2\%     \end{tabular}} &
          \begin{tabular}{@{}l@{}}   3.7 \\ $\pm$ 0.4    \end{tabular}&
     {\begin{tabular}{@{}l@{}}     1.1\%    \\    $\pm$0.4\%     \end{tabular}} 
     \\[10pt]
    5 & 
     {\begin{tabular}{@{}l@{}} $\dim\boldsymbol{x}={20}$\\ $\dim\boldsymbol{u}={3}$\end{tabular}} & 
     {\begin{tabular}{@{}l@{}} $\dim\boldsymbol{\Lambda}=15$ \\ $\rightarrow\dim\boldsymbol{\lambda}=3$ \end{tabular}} &
                  {\begin{tabular}{@{}l@{}}   2855.7 \\  $\pm$ 193.2      \end{tabular}}&
    {\begin{tabular}{@{}l@{}}    16.6\%     \\     $\pm$ 4.6\%      \end{tabular}}     &
     \begin{tabular}{@{}l@{}}    50.1\\ $\pm$   3.9 \end{tabular}&
     {\begin{tabular}{@{}l@{}}     1.9\%    \\    $\pm$ 0.3\%     \end{tabular}} & 
          \begin{tabular}{@{}l@{}}   7.1 \\ $\pm$  0.7  \end{tabular}&
     {\begin{tabular}{@{}l@{}}     1.0\%    \\    $\pm$ 0.4\%     \end{tabular}} 
     \\[10pt]
     
    6 & 
     {\begin{tabular}{@{}l@{}} $\dim\boldsymbol{x}={20}$\\ $\dim\boldsymbol{u}={3}$\end{tabular}} & 
     {\begin{tabular}{@{}l@{}} $\dim\boldsymbol{\Lambda}=15$ \\ $\rightarrow\dim\boldsymbol{\lambda}=5$ \end{tabular}} & 
    {\begin{tabular}{@{}l@{}}    2839.9 \\  $\pm$ 172.9      \end{tabular}}&
    {\begin{tabular}{@{}l@{}}     16.3\%     \\     $\pm$ 4.6\%      \end{tabular}} & 
     \begin{tabular}{@{}l@{}}   54.3 \\ $\pm$  4.8  \end{tabular}&
     {\begin{tabular}{@{}l@{}}     1.8\%    \\    $\pm$ 0.2\%     \end{tabular}} & 
          \begin{tabular}{@{}l@{}}   16.2 \\ $\pm$  3.3  \end{tabular}&
          {\begin{tabular}{@{}l@{}}     0.9\%    \\    $\pm$ 0.2\%     \end{tabular}}
     \\[10pt]
     
    7 & 
     {\begin{tabular}{@{}l@{}} $\dim\boldsymbol{x}={30}$\\ $\dim\boldsymbol{u}={3}$\end{tabular}} & 
     {\begin{tabular}{@{}l@{}} $\dim\boldsymbol{\Lambda}=15$ \\ $\rightarrow\dim\boldsymbol{\lambda}=3$ \end{tabular}} & 
        {\begin{tabular}{@{}l@{}}    3232.6 \\  $\pm$ 219.1      \end{tabular}}&
    {\begin{tabular}{@{}l@{}}    11.6\%     \\     $\pm$ 4.7\%      \end{tabular}}     & 
     \begin{tabular}{@{}l@{}}   70.7 \\ $\pm$  7.3   \end{tabular}&
     {\begin{tabular}{@{}l@{}}     2.3\%    \\    $\pm$ 0.6\%     \end{tabular}} &
          \begin{tabular}{@{}l@{}}   7.5 \\ $\pm$  0.6  \end{tabular}&
     {\begin{tabular}{@{}l@{}}     2.0\%    \\    $\pm$ 0.7\%     \end{tabular}}
     \\
      \bottomrule
\end{tabular}
\begin{tablenotes}
\item[*] Results for each case  are based on 10 trials, and each trial uses a different randomly-generated full-order dynamics  $\boldsymbol{f}()$.  The results are reported using mean and standard derivation over all ten trials. See detailed explanations about those quantities in text.
\end{tablenotes}
\end{threeparttable}
\end{center}
\end{table*}


\subsubsection{High Dimensional Examples} \label{section.pwa.ex2}
In this session, we quantitatively analyze   task-driven hybrid model reduction. For easy comparison, we represent the full-order  dynamics $\boldsymbol{f}()$ in (\ref{equ.pwa.mpc})  also in  LCS representation. All  matrices in $\boldsymbol{f}()$ are drawn from uniform distribution. We use $\boldsymbol{\Lambda}$ to denote the complementarity variable of $\boldsymbol{f}()$.
In the reduced-order LCS  $\boldsymbol{g}()$, we vary $\dim\boldsymbol{\lambda}$ to show the effect of varying degrees of mode reduction.

In Table \ref{table.pwa.ex2}, we consider different  full-order $\boldsymbol{f}$(),  listed in the second column, and  different hybrid mode reduction, listed in the third column. From the fourth to ninth columns, we use the following metrics to report the learning performance.


\begin{itemize}[leftmargin=10pt]

    \item \textbf{Random Policy, number of modes in $\boldsymbol{f}$:} This is the total  number of the hybrid modes  that are active in the full-order dynamics $\boldsymbol{f}()$, when one runs a random policy with     input  $\boldsymbol{u}^{\text{rand}}\sim U[-10, 10]$ and  initial condition   $\boldsymbol{x}_{0}\sim U[-4, 4]$. This metric indicates all possible modes experienced by the full-order system in a uniformly sampled state-input space.
    
     \smallskip
        
    \item \textbf{Random Policy, ME$(\boldsymbol{g})(\%)$:}  This  is the relative  prediction error of the learned reduced-order LCS model $\boldsymbol{g}()$ evaluated on the above  random policy data,     defined as
    \begin{equation}\label{equ.pwa.mpe2}
    \frac{\norm{\boldsymbol{g}(\boldsymbol{x},\boldsymbol{u}^{\text{rand}})-\boldsymbol{f}(\boldsymbol{x},\boldsymbol{u}^{\text{rand}})}^2}{\norm{\boldsymbol{f}(\boldsymbol{x},\boldsymbol{u}^{\text{rand}})}^2+10^{-6}} \times 100\%.
    \end{equation}

        \smallskip

    \item \textbf{$\boldsymbol{g}$-MPC Policy, number of modes in $\boldsymbol{f}$:} This is the total  number of  the hybrid modes that are active in the full-order dynamics  $\boldsymbol{f}()$ when one runs the learned   reduced-order $\boldsymbol{g}$-MPC controller on it with  initial  condition  $\boldsymbol{x}_{0}\sim U[-4, 4]$ This metric  indicates  all possible  modes experienced by the full-order system in the task-relevant state-input space.

\smallskip

    \item \textbf{$\boldsymbol{g}$-MPC Policy, ME$(\boldsymbol{g})(\%)$:} This is the relative  prediction error of the learned reduced-order LCS $\boldsymbol{g}()$ evaluated at the above $\boldsymbol{g}$-MPC policy data,     defined as
    \begin{equation}\label{equ.pwa.mpe}
    \frac{\norm{\boldsymbol{g}(\boldsymbol{x},\boldsymbol{u}^{\boldsymbol{g}\text{-MPC}})-\boldsymbol{f}(\boldsymbol{x},\boldsymbol{u}^{\boldsymbol{g}\text{-MPC}})}^2}{\norm{\boldsymbol{f}(\boldsymbol{x},\boldsymbol{u}^{\boldsymbol{g}\text{-MPC}})}^2+10^{-6}} \times 100\%.
    \end{equation}
    As the proposed algorithm chooses to minimize the model error, which is related to the performance gap as in  Lemma \ref{lemma.1}. Thus it is meaningful to include this metric to show the  model error of the learned $\boldsymbol{g}()$  on the task-relevant data.
    \smallskip
    
    \item \textbf{$\boldsymbol{g}$-MPC Policy, number of modes in $\boldsymbol{g}$:} This is the  total number of   hybrid modes that are active inside the $\boldsymbol{g}$-MPC controller, which is run on the full-order dynamics  $\boldsymbol{f}()$ with  initial system condition   $\boldsymbol{x}_{0}\sim U[-4, 4]$.
    
        \smallskip
    
    \item \textbf{Relative task performance gap $\mathcal{L}(\boldsymbol{g})(\%)$:} This is the relative task performance gap $\mathcal{L}(\boldsymbol{g})(\%)$ is  between $\boldsymbol{g}$-MPC controller and $\boldsymbol{f}$-MPC controller: 
    \begin{equation}\label{equ.pwa.loss}
    \mathcal{L}(\boldsymbol{g})=\frac{J(\boldsymbol{g}\text{-MPC})-J(\boldsymbol{f}\text{-MPC})}{J(\boldsymbol{f}\text{-MPC})} \times 100\%,
    \end{equation}
    where $J(\boldsymbol{g}\text{-MPC})$ is the  cost of a rollout by running  $\boldsymbol{g}\text{-MPC}$ controller on the full-order dynamics $\boldsymbol{f}()$, namely,
    \begin{multline}\label{equ.pwa.mpccost}
        J(\boldsymbol{g}\text{-MPC})=\E_{{\boldsymbol{\beta}}\sim p({\boldsymbol{\beta}})} \E_{p_{\boldsymbol{\beta}}(\boldsymbol{x}_0)}\sum_{t=0}^{H}{c}_{{{\boldsymbol{\beta}}}}(\boldsymbol{x}_t^{\boldsymbol{f}}, \boldsymbol{u}^{\boldsymbol{g}\text{-MPC}}_t),
    \end{multline}
    with $\{\boldsymbol{x}_{0:H}^{\boldsymbol{f}},\boldsymbol{u}_{0:H}^{\boldsymbol{g}\text{-MPC}}\}$ being a rollout of running $\boldsymbol{g}\text{-MPC}$ controller on  $\boldsymbol{f}()$. The similar definition applies to $J(\boldsymbol{f}\text{-MPC})$. 
    Recall the original learning loss (\ref{equ.loss1}).  $\mathcal{L}(\boldsymbol{g})(\%)$ can directly indicate the performance of the learned reduced-order model $\boldsymbol{g}()$ for the given task distribution $p({\boldsymbol{\beta}})$.
\end{itemize}

\begin{figure}[h]
     \centering
         \includegraphics[width=0.43\textwidth]{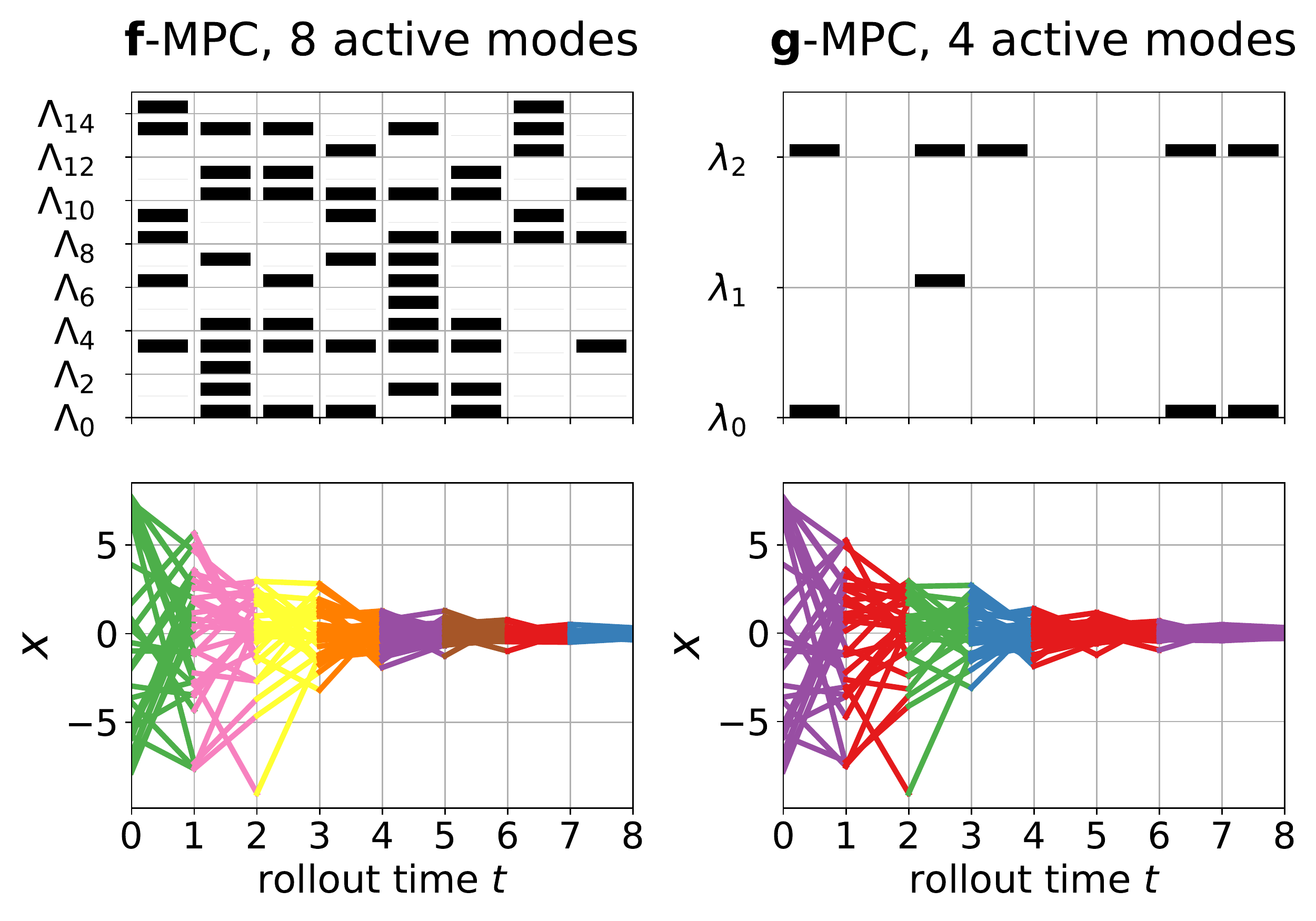}
\caption{\small  An example  rollout of $\boldsymbol{f}()$  with   full-order $\boldsymbol{f}$-MPC controller or reduced-order  $\boldsymbol{g}$-MPC controller, corresponding to Case 7 in Table \ref{table.pwa.ex2}. Specifically, the left column is a single rollout of running  $\boldsymbol{f}$-MPC controller on $\boldsymbol{f}()$, and the right running $\boldsymbol{g}$-MPC controller on $\boldsymbol{f}()$, both under the same initial condition. The upper row shows the activation of $\boldsymbol{\Lambda}$ or $\boldsymbol{\lambda}$ over time (black brick means nonzero and blank means zero). The bottom row shows the state trajectory $\boldsymbol{x}_{t+1}=\boldsymbol{f}\big(\boldsymbol{x}_t, \text{MPC}(\boldsymbol{x}_t)\big)$ over time, with each color representing a different hybrid mode. Note that since each panel only shows a single  instance of rollout (from a fixed $\boldsymbol{x}_0$), there are not many active hybrid modes of $\boldsymbol{f}$ involved in a single trajectory.
	} 
	\label{fig.ex2.runningmpc}
\end{figure}

\begin{figure*}[h]
	\centering
		\begin{subfigure}{.235\textwidth}
		\centering
		\includegraphics[width=\linewidth]{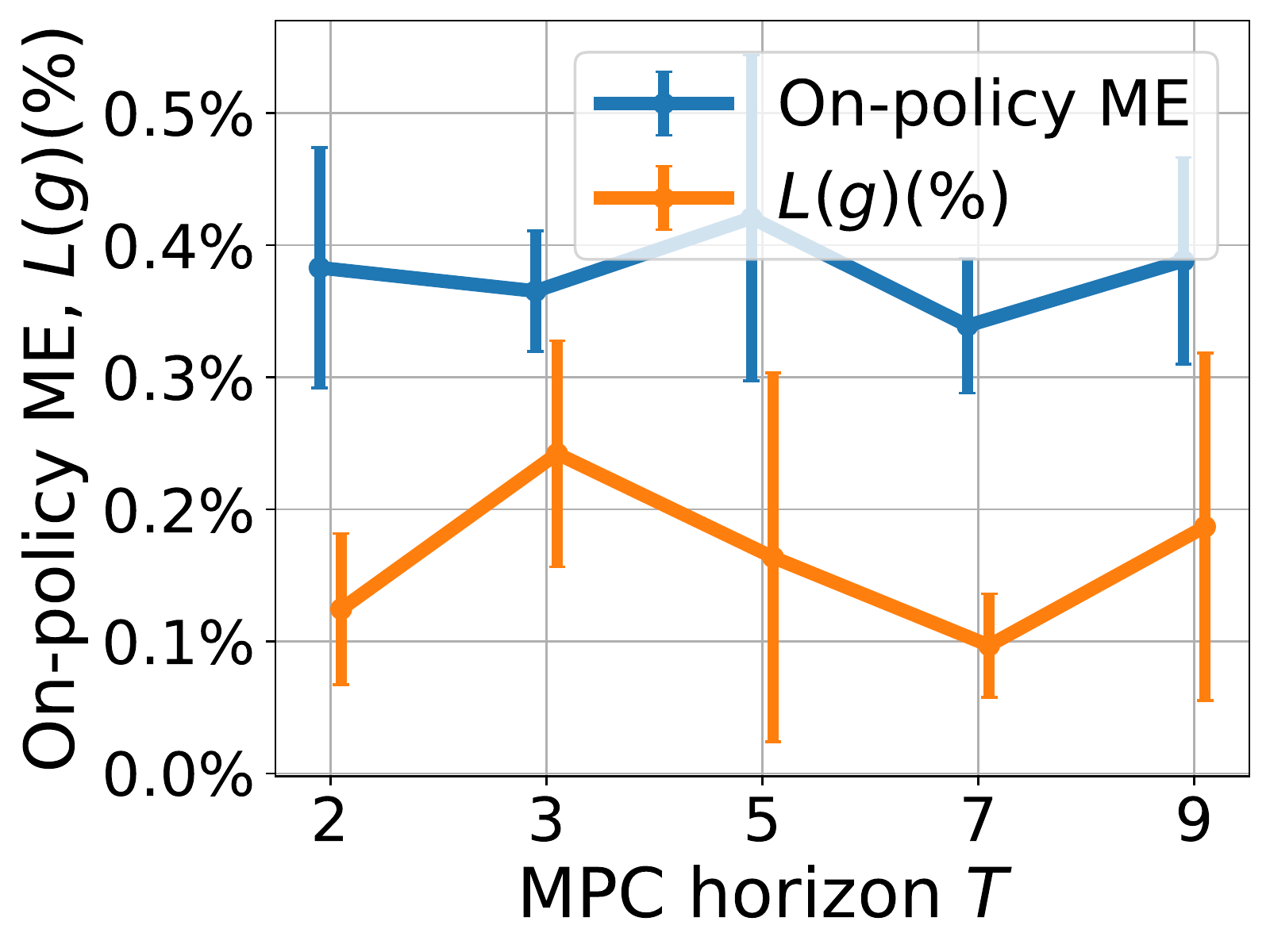}
		\caption{}
		\label{fig.ex3.ablation.2}
	\end{subfigure}
		\begin{subfigure}{.235\textwidth}
		\centering
		\includegraphics[width=\linewidth]{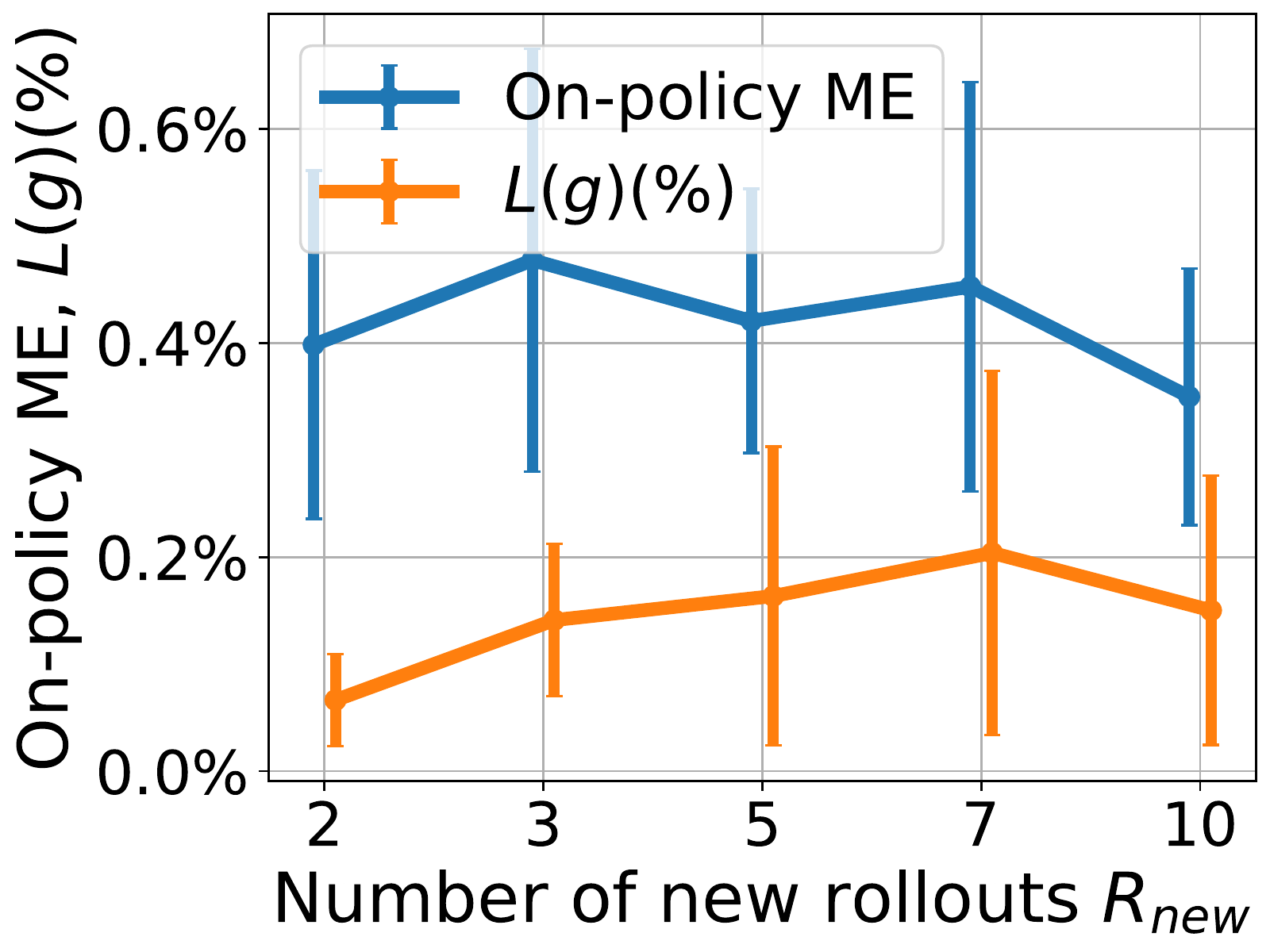}
		\caption{}
		\label{fig.ex3.ablation.1}
	\end{subfigure}
	\begin{subfigure}{.235\textwidth}
		\centering
		\includegraphics[width=\linewidth]{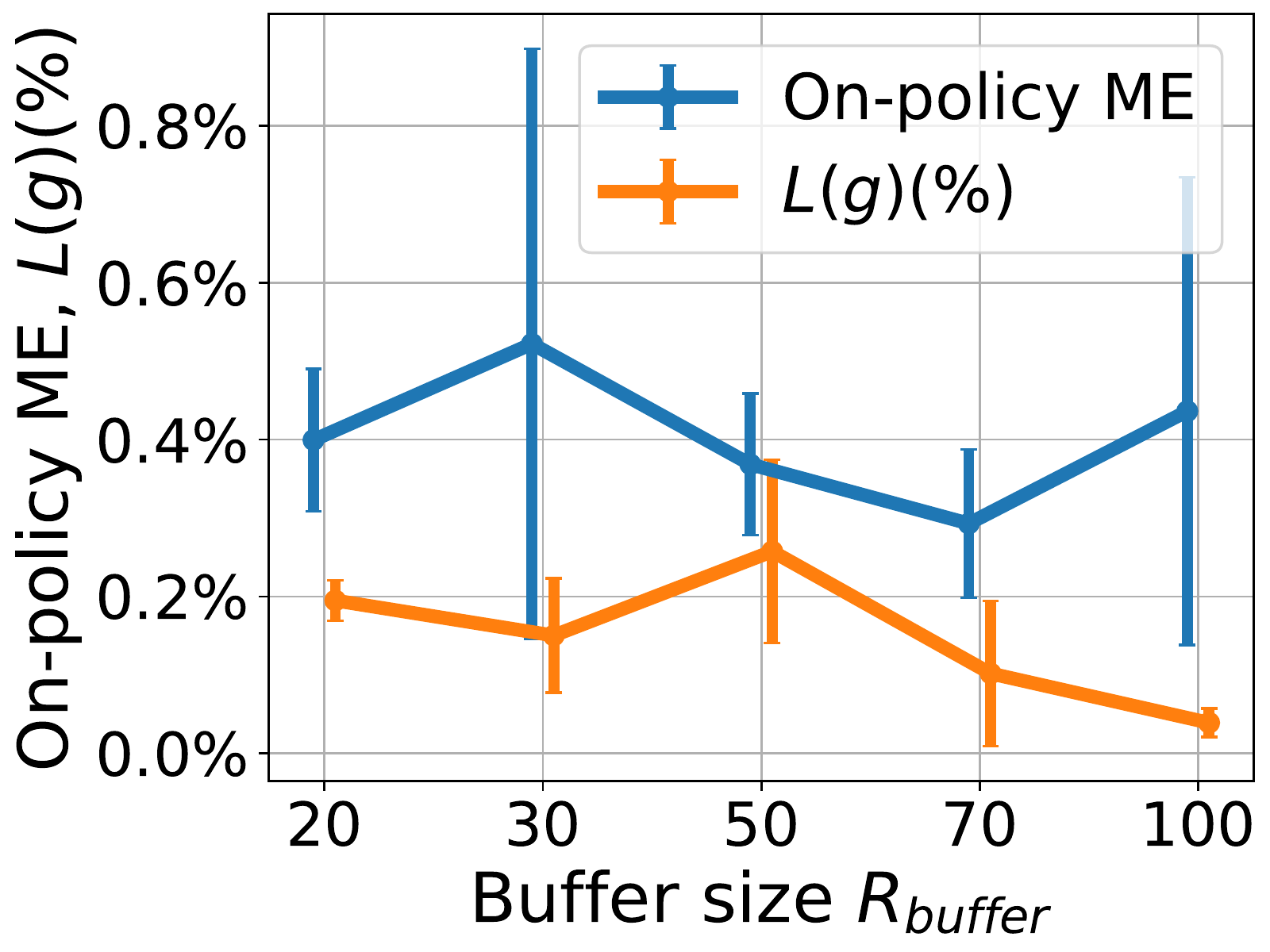}
		\caption{}
		\label{fig.ex3.ablation.3}
	\end{subfigure}
	\begin{subfigure}{.235\textwidth}
		\centering
		\includegraphics[width=\linewidth]{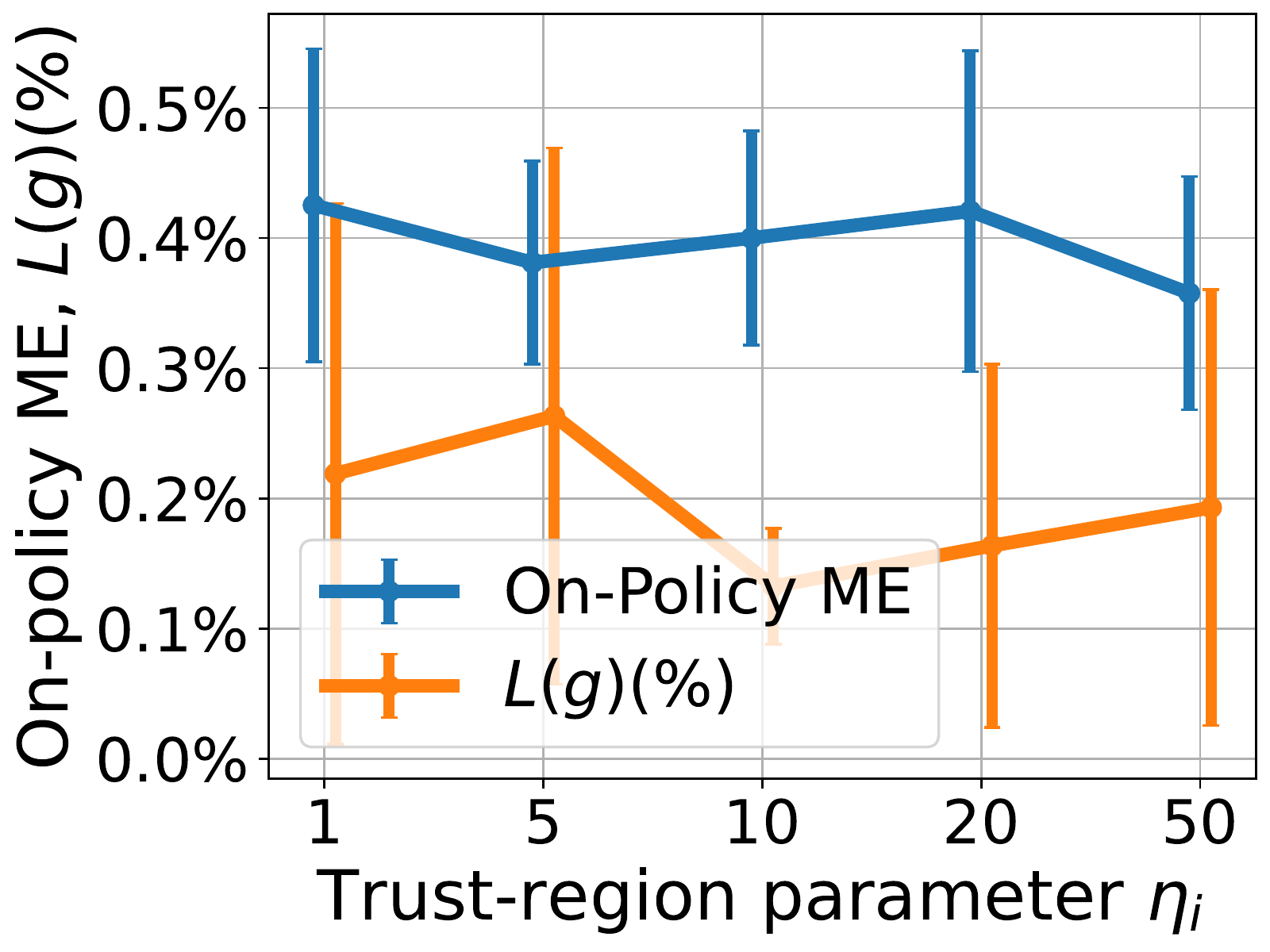}
		\caption{}
		\label{fig.ex3.ablation.4}
	\end{subfigure}
	\caption{\small Ablation study about the effect of  hyperparameter values on the algorithm performance. We here use the performance metrics \textbf{On-policy ME} in (\ref{equ.pwa.mpe})  and relative task performance gap $\mathcal{L}(\boldsymbol{g})(\%)$ in (\ref{equ.pwa.loss}) to report the algorithm performance. The experimenting system's dimensions and other settings follow Case 1 in Table \ref{table.pwa.ex2}. Each result is reported based on  ten trials, and each  trial uses a different randomly-generated full-order LCS $\boldsymbol{f}()$, as stated in the previous session. The error bars represent the standard deviation across all trials.
	} 
	\label{fig.ex3.ablation}
  \vspace{-10pt}
\end{figure*}

The results in Table \ref{table.pwa.ex2} clearly show a reliable performance of the proposed  method. For example, in Case 5, the full-order  $\boldsymbol{f}()$ has $\boldsymbol{x}\in\mathbb{R}^{20}$ and $\boldsymbol{u}\in\mathbb{R}^3$ has $\boldsymbol{\Lambda}\in \mathbb{R}^{15}$. With a random policy run on  $\boldsymbol{f}()$, the number of active hybrid modes in $\boldsymbol{f}()$ is around $2.8$k. The proposed  algorithm  learns a task-driven reduced-order LCS  $\boldsymbol{g}()$ only with around $7$ modes ($\dim \boldsymbol{\lambda}=3$). The resulting  reduced-order  $\boldsymbol{g}$-MPC controller running on the full-order $\boldsymbol{f}()$ only has around 1\% performance loss relative to running the full-order $\boldsymbol{f}$-MPC controller. The relative prediction error of the learned reduced-order LCS $\boldsymbol{g}()$  is less than $2\%$ on the on-policy ($\boldsymbol{g}$-MPC) data, while is 16.6 \% on the random policy data. Also, when run with $\boldsymbol{g}$-MPC controller, full-order $\boldsymbol{f}()$ has around 50 active modes. 
Table \ref{table.pwa.ex2} also shows that the proposed algorithm can handle high-dimensional system, such as  $\boldsymbol{x}\in\mathbb{R}^{30}$.  Additionally, corresponding to Case 7 in Table \ref{table.pwa.ex2}, a single  instance of $\boldsymbol{f}()$'s rollout   with  the full-order $\boldsymbol{f}$-MPC controller and with  the  reduced-order  $\boldsymbol{g}$-MPC controller are compared in  Fig. \ref{fig.ex2.runningmpc}. Based on the results in Table \ref{table.pwa.ex2}, we have the following conclusions.

(i) Jointly looking at the number of modes in $\boldsymbol{f}$ with random policy (fourth column),  the number of modes inside  $\boldsymbol{g}$-MPC policy (eighth column), and the relative performance gap (last column), one can  clearly see that the proposed algorithm can find a reduced-order model with \emph{multiple orders of magnitude fewer hybrid modes}  than the  full-order $\boldsymbol{f}()$, and it can  result in  a similar MPC control  performance as using  full-order MPC policy, with a performance loss less than $2\%-3\%$.

(ii) Comparison between the  mode counts in $\boldsymbol{f}$ with random policy (fourth column) and that with  $\boldsymbol{g}$-MPC policy (sixth column) can  confirm the motivating hypothesis of this paper:  a much fewer hybrid modes are actually necessary to achieve the task (here, the task  is to minimize  the given  cost function in (\ref{equ.pwa.cost})), and the vast majority of the hybrid modes in  $\boldsymbol{f}$ will remain untouched throughout the control process.

(iii) Notably, comparing  the relative model error of  the learned reduced-order LCS $\boldsymbol{g}()$ on   random policy data (fifth column) and only on the $\boldsymbol{g}$-MPC policy data (seventh column), one can conclude that the learned reduced-order LCS $\boldsymbol{g}()$ attains higher  validity on the task-relevant data. This sufficiently shows the success of our task-driven hybrid model reduction.

All the above results and  analysis clearly confirm that the effectiveness and efficiency of the proposed task-driven hybrid model reduction method.
Also, attention needs to be paid to Cases 3-6. Here,  under the same other conditions, we used an increasingly complex reduced-order LCS $\boldsymbol{g}$  from $\dim\boldsymbol{\lambda}=1$ to  $\dim\boldsymbol{\lambda}=5$. The results  show that increasing the hybrid mode budget in $\boldsymbol{g}()$ can lead to a performance improvement, although small, in the model  accuracy (seventh column) and the relative task performance gap (last column).


\subsection{Effect of Hyperparameter Settings}\label{section.pwa.param}

We conduct the ablation study to investigate the effect of  hyperparameter settings in Table \ref{table.pwa.setting} on the algorithm performance. We still use the \textbf{On-Policy ME$(\boldsymbol{g})(\%)$} in (\ref{equ.pwa.mpe}) and the relative task performance gap $\mathcal{L}(\boldsymbol{g})(\%)$ in (\ref{equ.pwa.loss}) to report the results. The  dimensions of the full-order and reduced-order models and other settings follows Case 1 in Table \ref{table.pwa.ex2}. The results are in Fig. \ref{fig.ex3.ablation}.
The results show that the learning performance is quite robust against a large range  of hyperparameter values in Table \ref{table.pwa.setting}.   Fig. \ref{fig.ex3.ablation.3} suggests that using a larger buffer size would slightly lower the final task performance gap, although not much improving the accuracy of the reduced-order model. Fig. \ref{fig.ex3.ablation.4} indicates that the choice of the trust region parameter $\eta_i$ does not significantly influence the learning performance. Overall, Fig. \ref{fig.ex3.ablation} suggests that setting algorithm hyperparameters is not difficult in practice.

\section{Three-Finger Dexterous Manipulation}\label{section.trifinger}
In this section, we will apply the proposed method to solve the three-fingered robotic hand manipulation \cite{wuthrich2020trifinger}. The Python codes to reproduce all the following results are available at \url{https://github.com/wanxinjin/Task-Driven-Hybrid-Reduction}.

\subsection{Three-Finger Robotic Hand Manipulation}

As illustrated in Fig. \ref{fig.tf}, the three-finger robotic hand manipulation system includes three 3-DoF robotic fingers and a  cube with a table. The  goal is to find a control  policy for the three fingers to moves the cube to given target poses. The entire simulation environment is based on MuJoCo physics engine \cite{todorov2012mujoco}. This paper considers two specific tasks.

\begin{figure}[h]
     \centering
         \centering
         \includegraphics[width=0.46\textwidth]{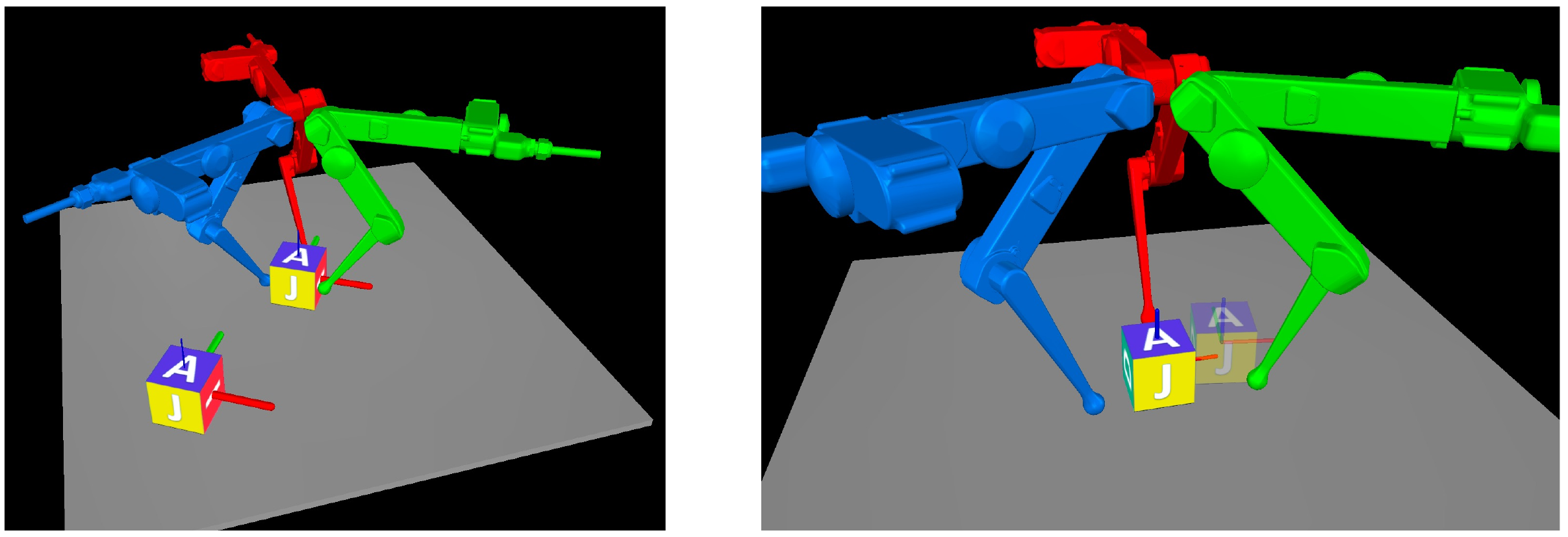}
        \caption{\small Three-finger dexterous manipulation tasks. Left: the three robotic fingers need to turn a cube to a random target orientation, given by a reference in the left corner. Right: the three  fingers need to move the cube to a  target pose with random position and orientation, given by the shaded reference.  The simulation environment uses MuJoCo physics engine \cite{todorov2012mujoco}.
        }
        \label{fig.tf}
\end{figure}

\textbf{Cube Turning Task:}  As shown in the left panel of Fig. \ref{fig.tf}, the cube has one degree of freedom (DoF) relative to the table: it can only rotate around a fixed vertical axis on the table. There is friction between the cube and table and also damping in the  joint of  cube rotation. Three fingers need to  turn the cube to any random target orientation $\alpha^\goal$, sampled from a uniform distribution:
    \begin{equation}\label{equ.tf.task1}
    {{\boldsymbol{\beta}}}=\alpha^\goal  \sim U[-1.5, \,\, 1.5] \quad \text{(radius)}
    \end{equation}
For visualization, the target  is shown in the bottom left corner.

\smallskip
\textbf{Cube Moving Task:}  As shown in the right panel of Fig. \ref{fig.tf}, the cube has 6 DoFs relative to the table: it is a free object on the table. The three fingers need to move the cube to align it to any random  target pose  $\boldsymbol{{\boldsymbol{\beta}}}=(\boldsymbol{p}^\goal, \alpha^\goal)$ on the table, where  $\boldsymbol{p}^{\goal}\in\mathbb{R}^2$  (center of mass) is the cube's target   xy-position  and  $\alpha^\goal$ is the cube's target orientation angle on the table, both sampled from uniform distribution 
\begin{equation}\label{equ.tf.task2}
\begin{aligned}
&\boldsymbol{p}^\goal\sim U[-\boldsymbol{p}_{\max}, \,\, \boldsymbol{p}_{\text{max}}], \quad \boldsymbol{p}_{\text{max}}=[0.06, \,\, 0.06]\tran (\text{m}),
\\
&\qquad\qquad\qquad\alpha^\goal \sim U[-0.5, \,\, 0.5] \quad (\text{rad}).
\end{aligned}
\end{equation}


The challenge of the above three-finger manipulation tasks lies in that the system  contains a large number of potential contact interactions that need to reason about. For example, (i) the contact interaction between each finger and the cube has three modes: separate, stick, and slip; (ii) each fingertip needs to reason which of cube faces to contact with; (iii) the contact interaction between the table and cube also contains at least three modes.  
Thus, the full-order dynamics, although unknown, contains an estimated \emph{thousands} of hybrid modes. Further, it will become  even more challenging  if one aims to perform real-time closed-loop control on the three-finger manipulation system to achieve the given tasks.

In this following, we will focus on solving the above three-finger manipulation tasks {without any prior knowledge} about the  three-finger manipulation system, e.g., geometry, physical properties, etc. We will apply the proposed method to learn a task-driven reduced-order LCS for real-time closed-loop MPC on the  three-finger manipulation to accomplish  the above  tasks. Note that different from our  application to synthetic hybrid systems in previous section, we here do not have a true hybrid model $\boldsymbol{f}()$ (and $\boldsymbol{f}$-MPC) for ground truth comparison.

\subsection{Experiment Settings}
Before proceeding, we here clarify the state and input spaces in the reduced-order LCS model $\boldsymbol{g}()$ in (\ref{equ.lcs.reduced}) and define the cost functions for the above two manipulation tasks. 

\subsubsection{Reduced-Order LCS Model}
We select the state space of the three-finger manipulation system as
\begin{equation}\label{equ.state}
    \boldsymbol{x}=
    \begin{bmatrix}
     \boldsymbol{p}_\obj,&\alpha_\obj, & \boldsymbol{p}_\ft
    \end{bmatrix}\tran
    \in\mathbb{R}^{9},
\end{equation}
where $\boldsymbol{p}_\ft\in\mathbb{R}^6$ is the xy positions of the three fingertips,  $\boldsymbol{p}_\obj\in\mathbb{R}^2$ is the xy position (of center of mass) of the object, and $\alpha_\obj$ is the planar orientation angle of the cube.
The input space of the three-finger manipulation system is 
\begin{equation}\label{equ.control}
    \boldsymbol{u}=\Delta \boldsymbol{p}_\ft\in\mathbb{R}^6,
\end{equation}
which includes  the incremental position of each fingertip. We  use operational space control (OSC) \cite{wensing2013generation} in the lower  level to map from $\boldsymbol{u}$ to the joint torque of each finger, also the OSC controller regularizes the z (vertical) position of each fingertip to be  constant. The OSC control frequency is 10 Hz. 
In the reduced-order LCS model $\boldsymbol{g}()$ in (\ref{equ.lcs.reduced}),  we set  
\begin{equation}\label{equ.tf.lamdim}
    \dim\boldsymbol{\lambda}= 5
\end{equation}
for both manipulation tasks.
This means that the reduced-order model can maximally represent $2^5=32$ hybrid modes, which are far fewer than the estimated thousands of  modes in the full-order dynamics of the three-finger manipulation system. The selection of $\dim \boldsymbol{\lambda}$ will be discussed later in Section \ref{trifinger.discussion}.

\begin{figure*}[h]
	\centering
		\begin{subfigure}{.24\textwidth}
		\centering
		\includegraphics[width=\linewidth]{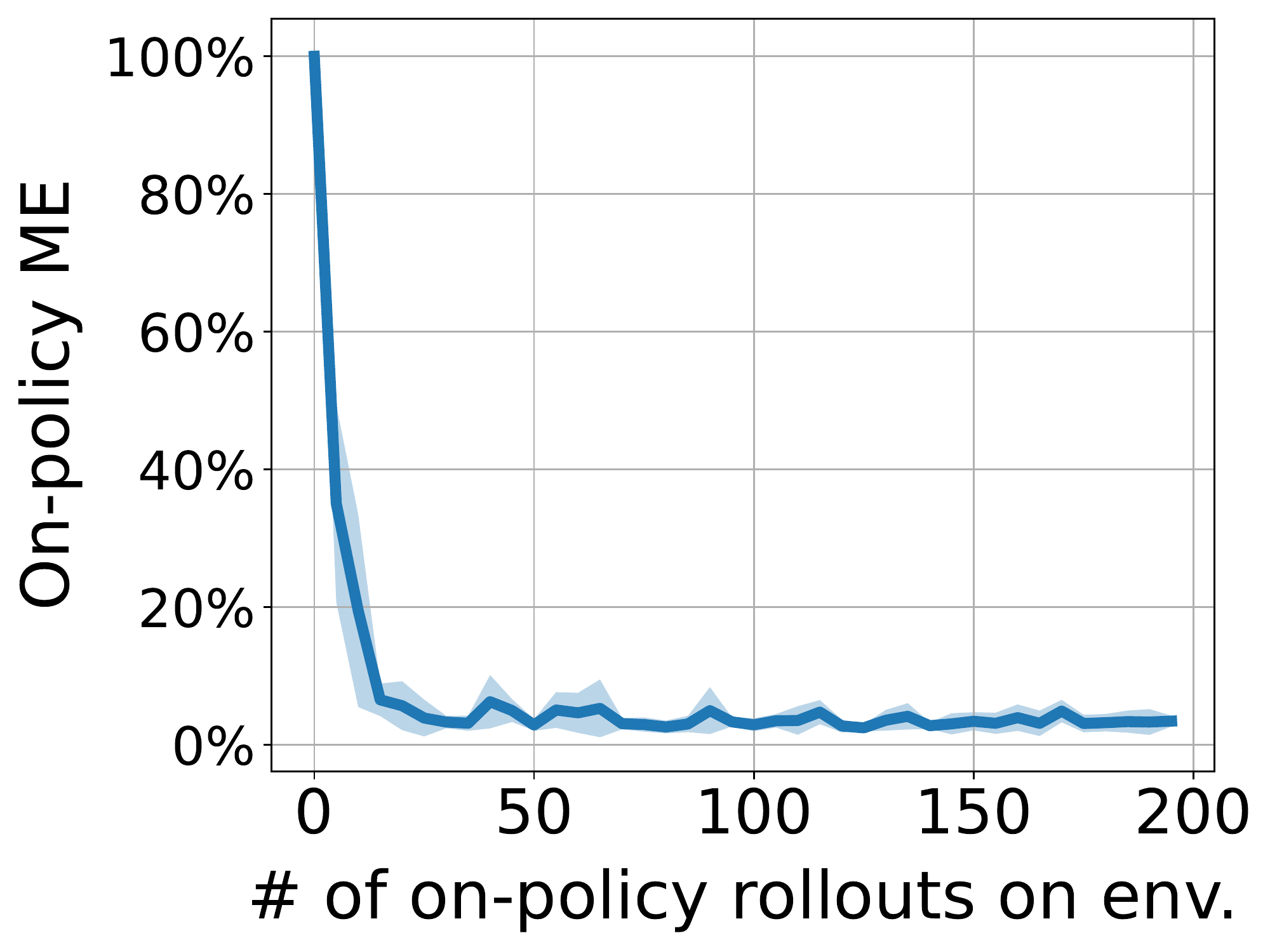}
		\caption{}
		\label{fig.tf.task1.plots.1}
	\end{subfigure}
		\begin{subfigure}{.24\textwidth}
		\centering
		\includegraphics[width=\linewidth]{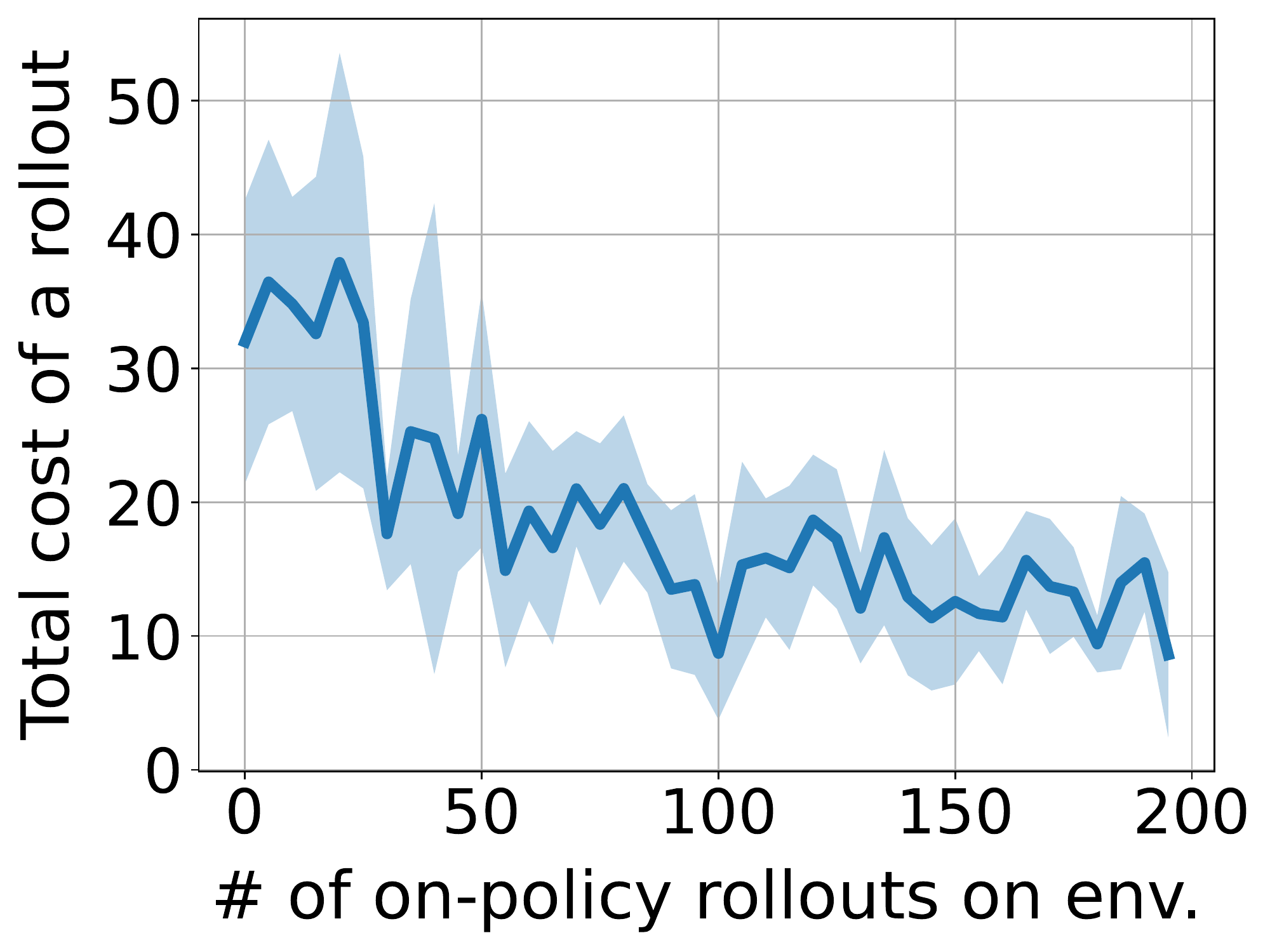}
		\caption{}
		\label{fig.tf.task1.plots.2}
	\end{subfigure}
	\begin{subfigure}{.24\textwidth}
		\centering
		\includegraphics[width=\linewidth]{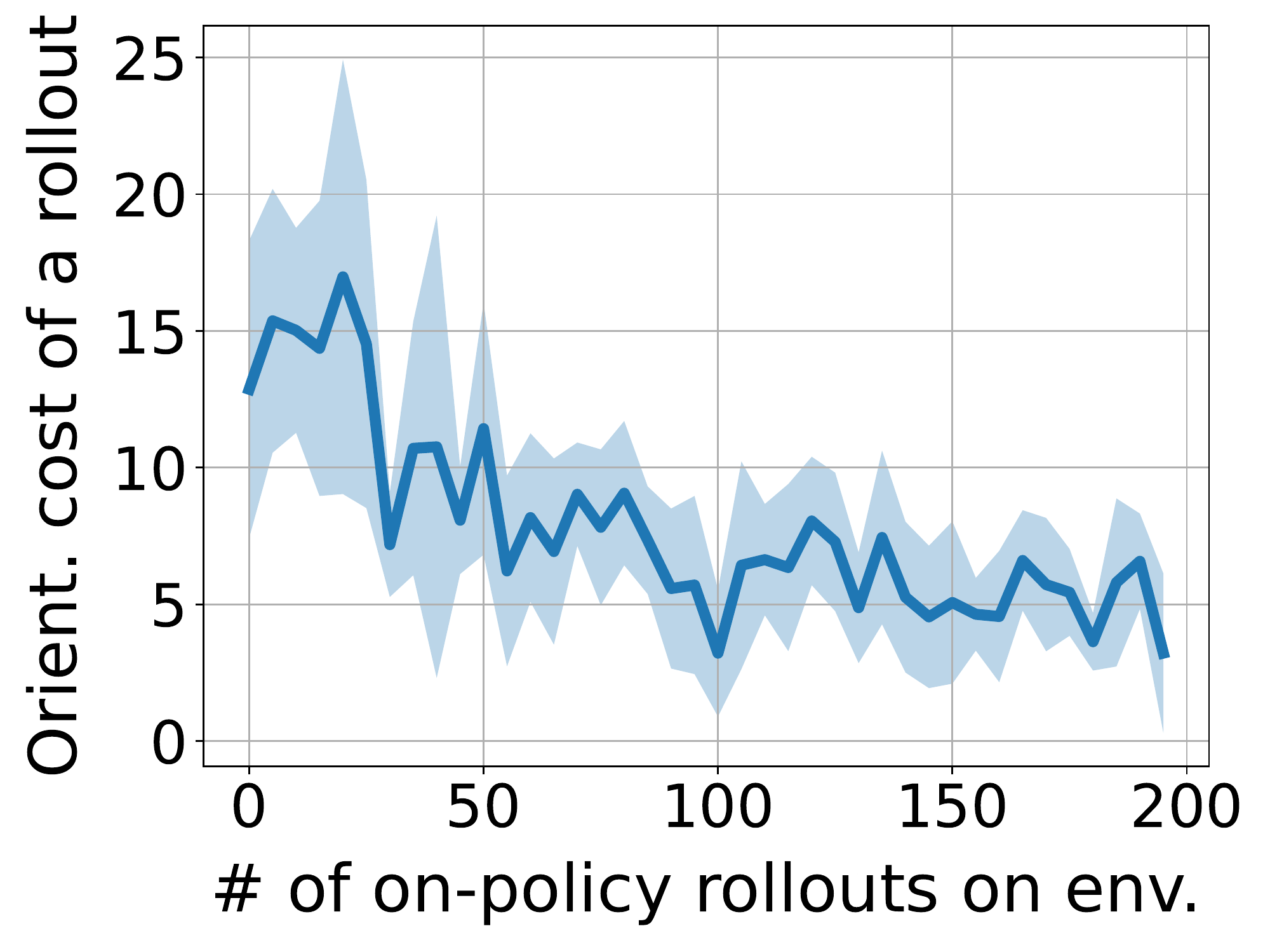}
		\caption{}
		\label{fig.tf.task1.plots.3}
	\end{subfigure}
	\begin{subfigure}{.24\textwidth}
		\centering
		\includegraphics[width=\linewidth]{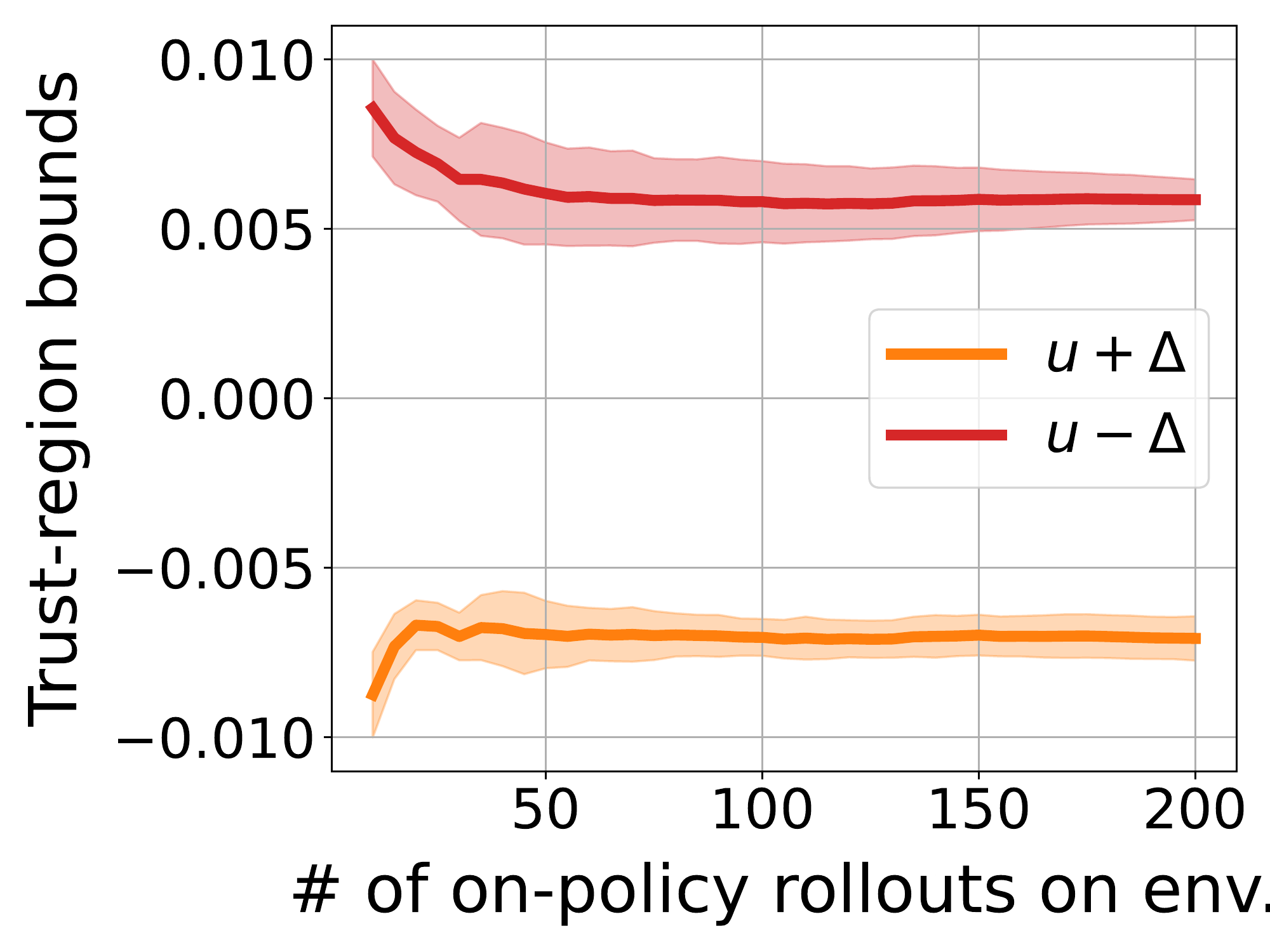}
		\caption{}
		\label{fig.tf.task1.plots.4}
	\end{subfigure}
	\caption{\small
	Learning curves of the three-finger manipulation system for the Cube Turning task. Each curve is the average of five random seeds, and its shaded area shows the  standard deviation. All results  here are shown on an on-policy rollout basis, and each on-policy rollout  is a result of running the reduced-order  $\boldsymbol{g}$-MPC controller  on the three-finger system (environment), i.e., the (unknown) full-order $\boldsymbol{f}()$. Detailed explanations are given in text.
	} 
	\label{fig.tf.task1.plots}
\vspace{-10pt}
\end{figure*}

\subsubsection{Cost Functions}
For given manipulation tasks,  we can simply define a continuous cost function as the sum of three terms: (i) the distance of the object to the goal, (ii) control effort, and (iii) the distance of the actuated fingers to the object (which is part of the state variable) to encourage the contact. Thus, we define the following quadratic cost function $J_{\boldsymbol{\beta}}$
\begin{equation}\label{equ.tf.cost}
    J_{\boldsymbol{{\boldsymbol{\beta}}}}=\sum\nolimits_{t=0}^{T-1} c_{\boldsymbol{{\boldsymbol{\beta}}}}(\boldsymbol{x}_t, \boldsymbol{u}_t)+h_{\boldsymbol{{\boldsymbol{\beta}}}}(\boldsymbol{x}_T),\quad \boldsymbol{\beta}\sim p(\boldsymbol{{\boldsymbol{\beta}}})
\end{equation}
where $p(\boldsymbol{{\boldsymbol{\beta}}})$ is  (\ref{equ.tf.task1}) for the Cube Turning task and (\ref{equ.tf.task2}) for the  Cube Moving task, and 
\begin{equation}\label{equ.tf.cost1}
\begin{aligned}
        c_{\boldsymbol{{\boldsymbol{\beta}}}}=&w_1^c\norm{\boldsymbol{p}_\ft -\boldsymbol{p}_\obj}^2+w_2^c\norm{\boldsymbol{p}_\obj-\boldsymbol{p}^{\goal}}^2\\
        &+w_3^c(\alpha_\obj-\alpha^\goal)^2+0.01\norm{\boldsymbol{u}}^2,\\[4pt]
        h_{\boldsymbol{{\boldsymbol{\beta}}}}= &w_1^h\norm{\boldsymbol{p}_\ft -\boldsymbol{p}_\obj}^2+w_2^h\norm{\boldsymbol{p}_\obj-\boldsymbol{p}^{\goal}}^2\\
        &+w_3^h(\alpha_\obj-\alpha^\goal)^2.
\end{aligned}
\end{equation}
Here, 
the cost term $\norm{\boldsymbol{p}_\ft -\boldsymbol{p}_\obj}^2$ penalizes the distance between the fingertips and  center of the cube, i.e. encouraging contact between fingertips and cube; $\norm{\boldsymbol{p}_\obj-\boldsymbol{p}^{\goal}}^2$ and $(\alpha_\obj-\alpha^\goal)^2$ are the squared distance to the target position or orientation, respectively; and $0.01\norm{\boldsymbol{u}}^2$ penalizes the control cost. $\boldsymbol{w}^c=[w_1^c, w_2^c, w_3^c]\tran$ and  $\boldsymbol{w}^h=[w_1^h, w_2^h, w_3^h]\tran$ are the cost weights, whose values will be given later. An extended discussion of different choices of the cost weights will be given in later Section \ref{trifinger.discussion}.

\smallskip
The other hyperparameters of  Algorithm \ref{algorithm1} are listed in the following table, which largely follows the ones in Table \ref{table.pwa.setting} for the previous  synthetic  system examples (recall that the discussion of the hyperparameter settings is in Section \ref{section.pwa.param}).

\begin{table}[h]
\begin{center}
\caption{Hyperparameters  for three-finger manipulation}
\label{table.tf.setting}
\begin{threeparttable}
\begin{tabular}{l l l}
    \toprule
    \textbf{Parameter\tnote{1}} & Symbol & Value  \\ 
    \midrule
      MPC horizon   &        $T$ &5   \\
      Rollout horizon  &     $H$ & $20$   \\ 
      \# of new  rollouts added to buffer per iter. &  $R_{\text{new}}$ & 5    \\
      Maximum buffer size &    $R_{\text{buffer}}$ & 200 rollouts    \\
      Trust region parameter  & $\eta_i$ & 1.0  \quad$\forall i$  \\
      Initial guess $\boldsymbol{\theta}$ in (\ref{equ.lcs.param})  for  $\boldsymbol{g}()$ &  $\boldsymbol{\theta}_0$ & $U[-0.5, 0.5]$ \\ 
      \bottomrule
\end{tabular}
\begin{tablenotes}
\item[1] Other settings not listed here will be stated in text. 
\end{tablenotes}
\end{threeparttable}
\end{center}
\end{table}

\subsection{Cube Turning Task}\label{section.tf.task1}
This session presents the results and analysis for solving the Cube Turning manipulation task. In this task, we set cost function weights in (\ref{equ.tf.cost1}) as
\begin{equation}\label{equ.cost1.weights}
\begin{aligned}
&\boldsymbol{w}^c=\begin{bmatrix}
10.0 & 0.0& 2.0 
\end{bmatrix}\tran,\\
&\boldsymbol{w}^h=\begin{bmatrix}
2.0 & 0.0& 10.0
\end{bmatrix}\tran.
\end{aligned}
\end{equation}
The above weight values are not deliberately picked. In fact, the  performance is not sensitive to the choice of  $\boldsymbol{w}^c$ and $\boldsymbol{w}^h$. As will be discussed  in Section \ref{trifinger.discussion}, the similar learning and  task performance permits a wide selection of  weight values.

\subsubsection{Results} \label{section.tf.task1.result}
The key learning curves are shown in Fig. \ref{fig.tf.task1.plots}, where each curve is the average of five random-seed trials and the shaded area indicates the  standard deviation.
In all panels, x-axis shows the total number of on-policy ($\boldsymbol{g}$-MPC controller) rollouts of the environment, which is proportional to the learning iteration (i.e., each iteration collects 5 new on-policy rollouts). Specifically, Fig. \ref{fig.tf.task1.plots.1} shows the relative prediction error of the  reduced-order LCS  $\boldsymbol{g}$() evaluated on the on-policy rollout data, defined in (\ref{equ.pwa.mpe}), where $\boldsymbol{f}(\boldsymbol{x},\boldsymbol{u}^{\boldsymbol{g}\text{-MPC}})$  is a direct  observation of  the next state of the environment. Fig. \ref{fig.tf.task1.plots.2} shows the total cost of a rollout from running $\boldsymbol{g}$-MPC controller on the environment, defined in (\ref{equ.pwa.mpccost}). Fig. \ref{fig.tf.task1.plots.3}  shows the orientation cost of a rollout from running $\boldsymbol{g}$-MPC controller in the environment, defined as
\begin{equation}
    \label{equ.tf.ori_cost}
E_{{\boldsymbol{\beta}}\sim p({\boldsymbol{\beta}})} \E_{\boldsymbol{x}\sim p_{\boldsymbol{\beta}}(\boldsymbol{x}_0)}\sum\nolimits_{t=0}^{H}(\alpha_{\obj, t}-\alpha^\goal)^2.
\end{equation}
 Fig. \ref{fig.tf.task1.plots.2} and Fig. \ref{fig.tf.task1.plots.3} show the very similar pattern because the orientation cost term  $(\alpha_\obj-\alpha^\goal)^2$  dominates in (\ref{equ.tf.cost1}) relative to other cost terms in scale.
Fig. \ref{fig.tf.task1.plots.4} shows the trust region upper bound $\boldsymbol{\bar{u}}+\boldsymbol{\Delta}$ and lower bound $\boldsymbol{\bar{u}}-\boldsymbol{\Delta}$ for the first component in the control input vector. 

Some  quantitative results that are not have shown in Fig. \ref{fig.tf.task1.plots} are given in Table \ref{table.tf.task1.table1}. In the second row of Table \ref{table.tf.task1.table1}, the cube's terminal orientation error $|\alpha_{\obj, H}-\alpha^\goal|$ is calculated at the end (at time step $H$) of a rollout. 
In the last row of Table \ref{table.tf.task1.table1}, we   report the  robustness of the $\boldsymbol{g}$-MPC controller against  external disturbance torques added to the cube during $\boldsymbol{g}$-MPC policy rollout. Here, we apply a  3D external   disturbance torque, sampled from $\boldsymbol{\tau}_{\text{disturb}}\sim U[-{\tau}^{\text{mag}}_{\text{disturb}}, {\tau}^{\text{mag}}_{\text{disturb}}]$, during each time interval ($0.1$s) of  rollout steps. We increase the disturbance magnitude ${\tau}^{\text{mag}}_{\text{disturb}}$ until the resulting $\boldsymbol{g}$-MPC rollout has an average cube terminal orientation error  $|\alpha_{\obj, H}-\alpha^\goal|\geq0.3$ (rad).   We report the result by calculating the  maximum angular acceleration of disturbance:  $\frac{{\tau}^{\text{mag}}_{\text{disturb}}}{{I}_{\obj}}$, with ${I}_{\obj}$ the cube inertia.

\begin{table}[h]
\begin{center}
\caption{Performance of  learned $\boldsymbol{g}$-MPC for Cube Turning}
\label{table.tf.task1.table1}
\begin{threeparttable}
\begin{tabular}{l l}
    \toprule
    \textbf{Results} &  Value  \\ 
    \midrule
      Total number of hybrid modes in  $\boldsymbol{g}$-MPC   &  (around)  14 \\[2pt]
      Cube terminal orientation error  $|\alpha_{\obj, H}-\alpha^\goal|$ & 0.06 $\pm$ 0.02 (rad)\\[2pt]
      Terminal  error (relative)  $\frac{(\alpha_{\obj, H}-\alpha^\goal)^2}{(\alpha^\goal)^2}$ & 3.4\% $\pm$ 2.3\%\\[6pt]
      Total training time of the algorithm\tnote{1} & $4.1 \pm 0.1$ mins\\[2pt]
      Total \# of environment samples in training     & 4k\\[4pt]
      Running frequency of reduced-order $\boldsymbol{g}$-MPC\tnote{1} & $>$50 Hz\\[4pt]
      \% of stick-slip-separate modes in rollouts & (approx.) $>70\%$\tnote{2}\\[5pt]
          {\begin{tabular}{@{}l@{}} Maximum
      $\frac{{\tau}^{\text{mag}}_{\text{disturb}}}{{I}_{\obj}}$ until $|\alpha_{\obj, H}-\alpha^\goal|\geq0.3$   \end{tabular}}  &  5000 rad/s$^2$\\[2pt]
      \bottomrule
\end{tabular}
\begin{tablenotes}
\item[1] The experiments are tested on MacBook Pro with M1 Pro chip.
\item[2] This is approximately calculated by observing the environment rollouts with the learned reduced-order $\boldsymbol{g}$-MPC.
\end{tablenotes}
\end{threeparttable}
\end{center}
\end{table}

\begin{figure*}[h]
	\centering
		\begin{subfigure}{.23\textwidth}
		\smallskip
		\centering
		\includegraphics[width=\linewidth]{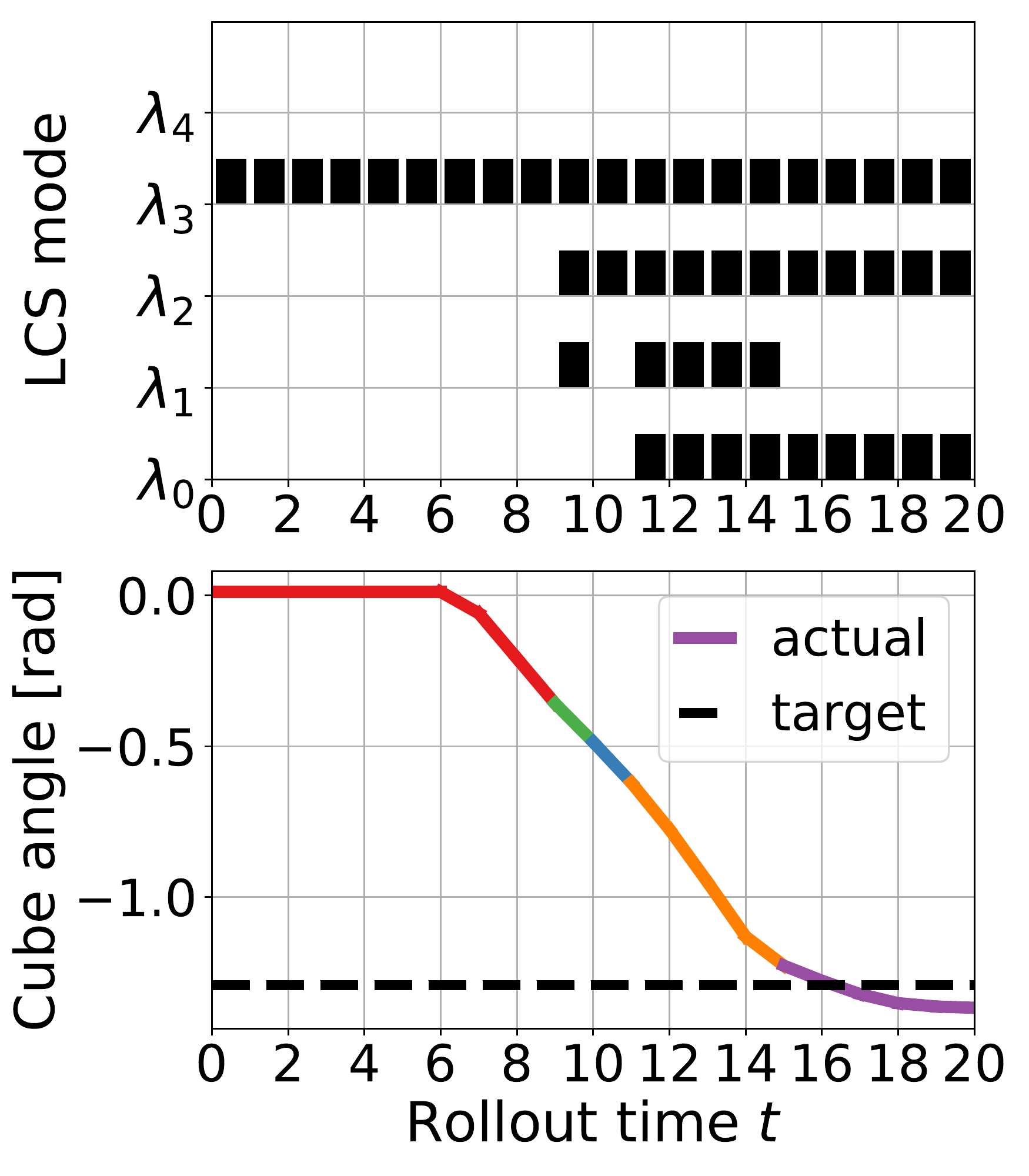}
		\caption{}
		\label{fig.tf.task1.demo11}
	\end{subfigure}
	\hfill
		\begin{subfigure}{.73\textwidth}
		\centering
		\includegraphics[width=\linewidth]{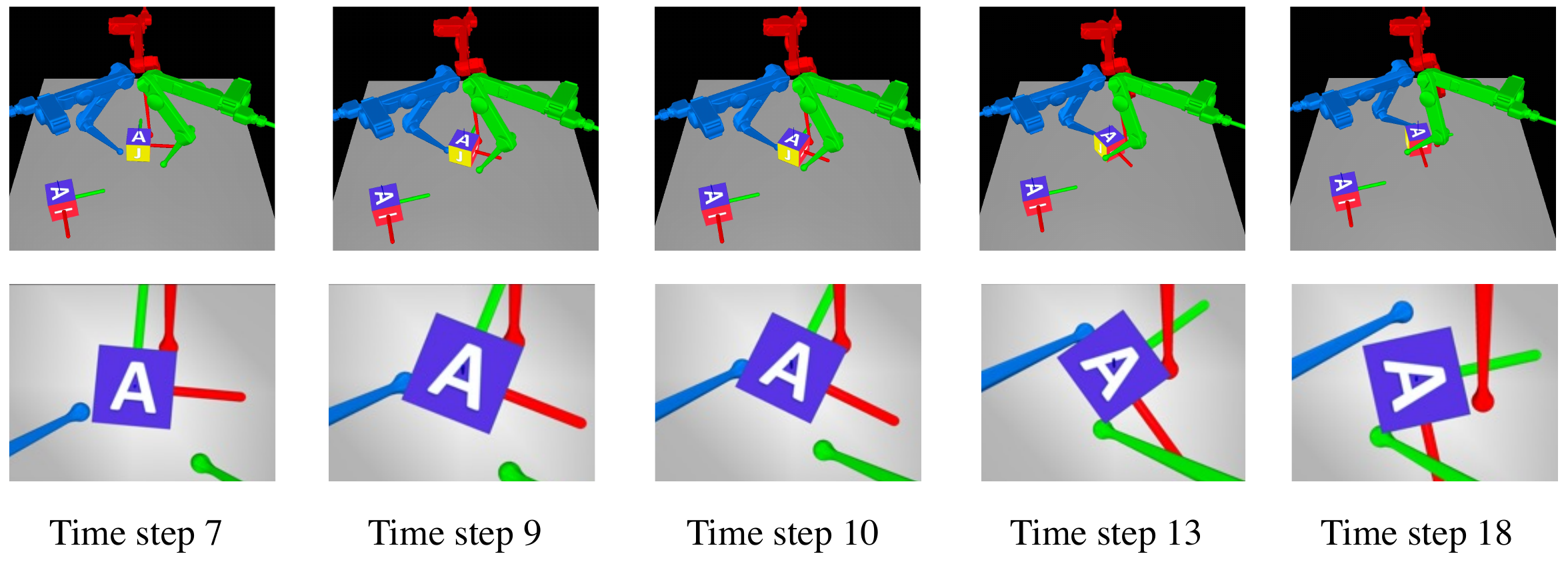}
		\caption{}
		\label{fig.tf.task1.demo12}
	\end{subfigure}
	\centering
		\begin{subfigure}{.225\textwidth}
		\centering
		\includegraphics[width=\linewidth]{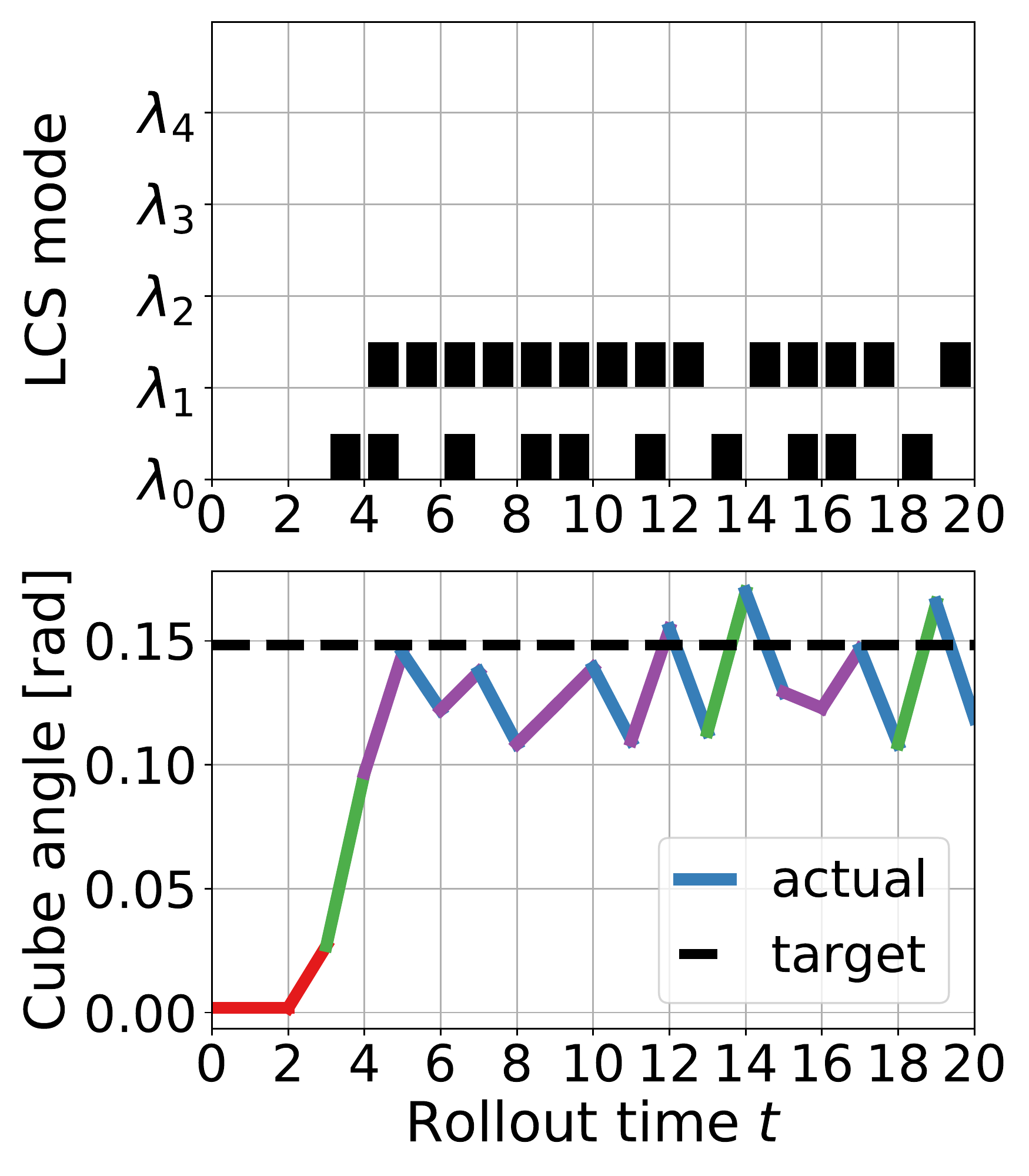}
		\caption{}
		\label{fig.tf.task1.demo21}
	\end{subfigure}
	\hfill
		\begin{subfigure}{.73\textwidth}
		\centering
		\includegraphics[width=\linewidth]{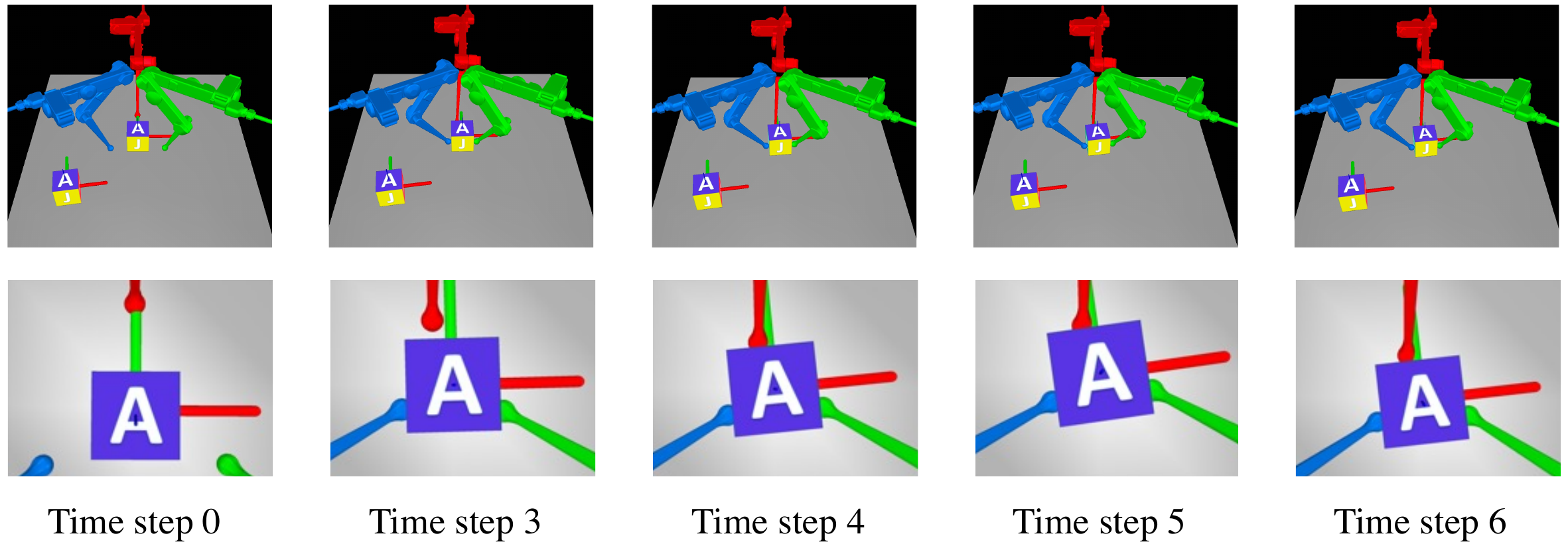}
		\caption{}
		\label{fig.tf.task1.demo22}
	\end{subfigure}
	\caption{\small
	Two rollouts of running the learned reduced-order  $\boldsymbol{g}$-MPC controller on the three-finger system. (a) and (b):  the  target orientation is $\alpha^\goal=-1.29$ (rad). 	(c) and (d):  $\alpha^\goal=-0.148$. 
    The upper panel of	(a) or (c) shows the  mode activation $\sign(\boldsymbol{\lambda})$ in $\boldsymbol{g}()$ over rollout time. Here,  black bricks show $\boldsymbol{\lambda}>0$ and blank  $\boldsymbol{\lambda}=0$. The bottom panel of (a) or (c) shows the trajectory of cube angle $\alpha_{\obj, t}$, where different mode activation are indicated by different colors. (b) or (d) shows the key-time-step snapshots of the simulator, corresponding to (a) or (c), respectively. Here, the upper panels show the environment snapshots, and lower panels show the zoom-in  details. Analysis are given in text and Tables \ref{table.tf.demo1.table1} and \ref{table.tf.demo1.table2}. Note that in (b) and (d) we have attached a  green-red body coordinate frame to the cube only for visualization purposes (i.e., the coordinate frame does not affect the physical contact interaction).
	} 
	\label{fig.tf.task1.demo}
\end{figure*}

Fig. \ref{fig.tf.task1.plots} and Table \ref{table.tf.task1.table1} show the  efficiency of  the proposed method to successfully solve the three-finger dexterous manipulation for the Cube Turning  task. Particularly, we have the following conclusions.

(i) The proposed algorithm learns a reduced-order model to solve the three-finger manipulation of Cube Turning task without any prior knowledge within just 5 minutes of wall-clock time, including real-time closed-loop control on the manipulation system. It only  collects around 4k (200 rollouts $\times$ 20 steps/rollout) data points from the environment.

(ii)  The learned task-driven reduced-order LCS  $\boldsymbol{g}()$ leads to a closed-loop MPC controller on the three-finger manipulation system,  achieving a high  accuracy:  the cube terminal orientation error $|\alpha_{\obj, H}-\alpha^\goal|<0.08$ (rad) and  the relative orientation error $\frac{(\alpha_{\obj, H}-\alpha^\goal)^2}{(\alpha^\goal)^2}<5\%$.

(iii) The learned  LCS  $\boldsymbol{g}()$ results in  a $\boldsymbol{g}$-MPC, which enables real-time closed-loop control on the three-finger manipulation system to achieve the task. The running frequency of $\boldsymbol{g}$-MPC is more than  $50$Hz.

(iv) The learned reduced-order LCS $\boldsymbol{g}()$  maximally contains 32 modes, and around 14 of them are active. Those 14 hybrid modes enables rich contact interactions, including   \texttt{separate}, \texttt{stick}, and \texttt{slip} between the fingertips and the cube,  and \texttt{stick},  \texttt{CCW rotational slip}, and  \texttt{CW rotational slip} between the cube and the table,  happening at different time steps. More than 70\% of rollouts contains the sequence of \texttt{stick-slip-separate} modes.  More detailed explanations will be given in the next session.

(v) The  reduced-order  $\boldsymbol{g}$-MPC controller shows high robustness against large external torque disturbances. The robustness  is  a natural benefit of  the closed-loop  implementation of $\boldsymbol{g}$-MPC controller. Such a high robust performance could also partially due to  the high stiffness gain in our lower-level OSC.

\smallskip

\begin{table*}[h]
\begin{center}
\caption{Empirical correspondence between LCS mode  activation in Fig. \ref{fig.tf.task1.demo11} and physical contact interaction in Fig.\ref{fig.tf.task1.demo12}. }
\label{table.tf.demo1.table1}
\begin{threeparttable}
\begin{tabular}{lccc}
    \toprule
     \textbf{Time $t$} & \textbf{Mode activation in $\boldsymbol{g}()$} & \textbf{Interaction between fingertips and cube}\tnote{1}  
     & \textbf{Interaction between cube and table}\tnote{1} 
     \\ 
    \midrule
      $t=0,1,..,8$ &  
            $\sign(\boldsymbol{\lambda})=[0,0,0,1,0]\tran$     & 
      {\begin{tabular}{@{}c@{}} (\texttt{R, G, B separate})   or   \\  
     (\texttt{G, B separate} and \texttt{R touching})  
      \end{tabular}} 
      &
    {\begin{tabular}{@{}c@{}} \texttt{Cube stick to table} or  \\  
        \texttt{CW rotational slip}  \\
      \end{tabular}} 
      \\[8pt]
      $t=9$ &  
            $\sign(\boldsymbol{\lambda})=[0,1,1,1,0]\tran$     & 
      \texttt{G separate}    
      and \texttt{R, B touching}   &
\texttt{CW rotational slip}
      \\[8pt]
    $t=10$ &  
            $\sign(\boldsymbol{\lambda})=[0, 0, 1, 1, 0]\tran$     & 
            \texttt{R right slip} and \texttt{G separate} and  \texttt{B stick}
& \texttt{CW rotational slip} 
      \\[8pt]
     $t=11, ... 14$ &
            $\sign(\boldsymbol{\lambda})=[1, 1, 1, 1, 0]\tran$    
            &  \texttt{R, G, B touching}  & \texttt{CW rotational slip} 
      \\[8pt]
      $t=15,..., 20$ &  
            $\sign(\boldsymbol{\lambda})=[1, 0, 1, 1, 0]\tran$     & 
             \texttt{G touching}  and  
       \texttt{R, B separate} &  \texttt{CW rotational slip}\\[8pt]
      \bottomrule
\end{tabular}
\begin{tablenotes}
\item[1] `R':`red finger', `B':`red finger', `G':`green finger', `CW’:`clockwise',  `CCW’:`counter-clockwise'. 
\end{tablenotes}
\end{threeparttable}
\end{center}
\end{table*}

\begin{table*}[h]
\begin{center}
\caption{Empirical correspondence between LCS mode  activation in Fig. \ref{fig.tf.task1.demo21} and physical contact interaction in Fig.\ref{fig.tf.task1.demo22}. }
\label{table.tf.demo1.table2}
\begin{threeparttable}
\begin{tabular}{lccc}
    \toprule
     \textbf{Time $t$} & \textbf{Mode activation in $\boldsymbol{g}()$} & \textbf{Interaction between fingertips and cube}\tnote{1}  
     & \textbf{Interaction between cube and table}\tnote{1} 
     \\
    \midrule
      $t=0,1,2$ &  
            $\sign(\boldsymbol{\lambda})=[0, 0, 0, 0, 0]\tran$     & 
      \texttt{R, G, B separate}  & 
          {\begin{tabular}{@{}c@{}} \texttt{Cube stick to table} or  \\  
      \texttt{CCW rotation slip} 
      \end{tabular}}
      \\[8pt]
      $t=3, 13, 18$ &  
            $\sign(\boldsymbol{\lambda})=[1, 0, 0, 0, 0]\tran$     & 
     {\begin{tabular}{@{}c@{}} 
     (\texttt{R separate}  and \texttt{G, B touching})  \\  
      or \texttt{R, G, B touching}  \end{tabular}} & \texttt{CCW rotational slip (large)}
      \\[8pt]
    {\begin{tabular}{@{}r@{}}
            $t=4, 6, 8, 9, 11, 15, 16$
    \end{tabular}} &  
    $\sign(\boldsymbol{\lambda})=[1, 1, 0, 0, 0]\tran$     & 
    \texttt{R, G, B touching} &
     {\begin{tabular}{@{}c@{}} 
      \texttt{CCW rotational slip (small)} \\
      \end{tabular}} 
      \\[8pt]
    {\begin{tabular}{@{}l@{}}
            $t=5, 7, 10, 12, 14, 17, 19$ 
    \end{tabular}} &  
    $\sign(\boldsymbol{\lambda})=[0, 1, 0, 0, 0]\tran$     & \texttt{R, G, B touching} &  \texttt{CW rotational slip} \\
      \bottomrule
\end{tabular}
\end{threeparttable}
\end{center}
\end{table*}

  \begin{figure*}[h]
	\centering
		\begin{subfigure}{.243\textwidth}
		\centering
		\includegraphics[width=\linewidth]{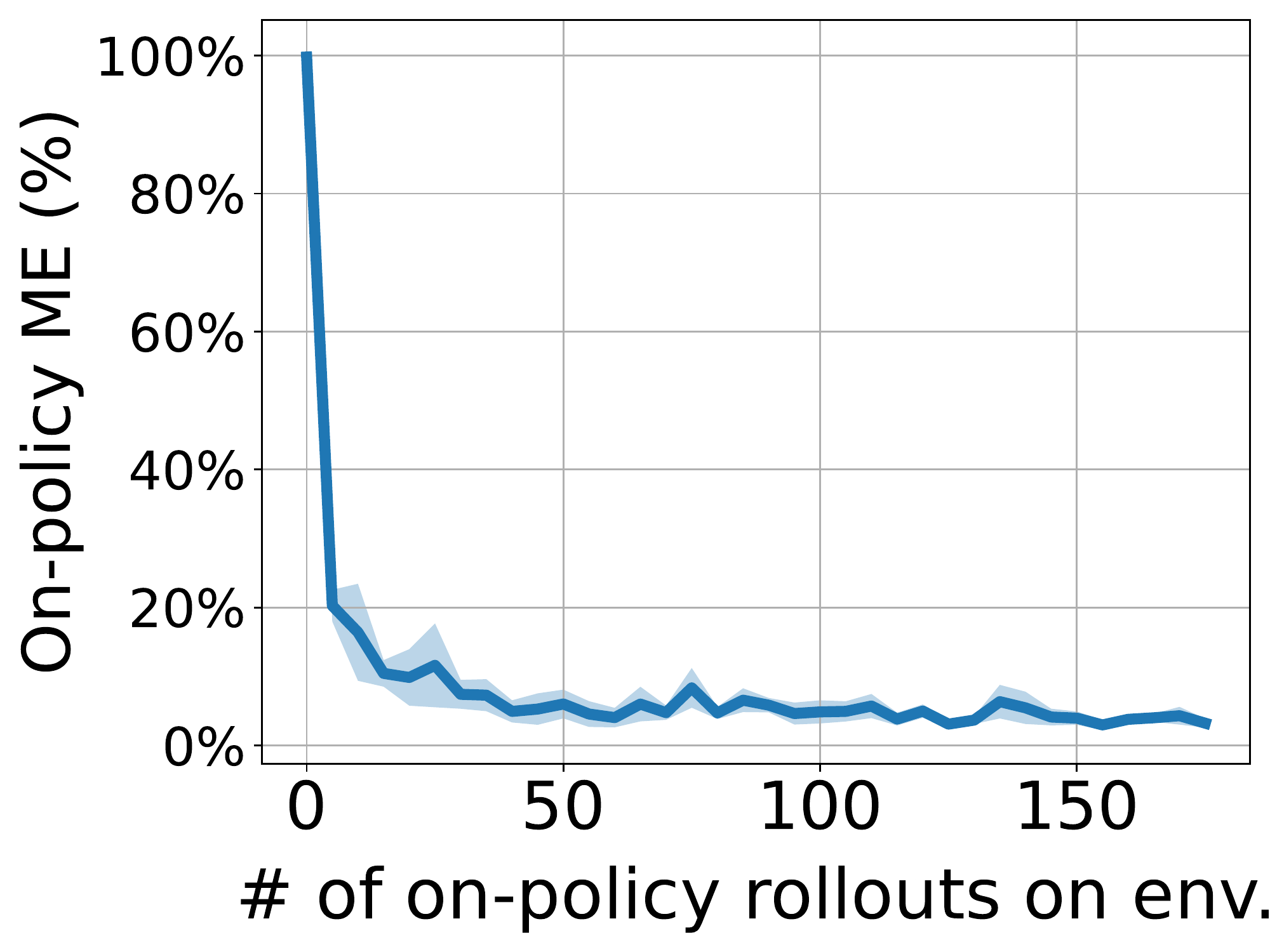}
		\caption{}
		\label{fig.tf.task2.plots.1}
	\end{subfigure}
		\begin{subfigure}{.243\textwidth}
		\centering
		\includegraphics[width=\linewidth]{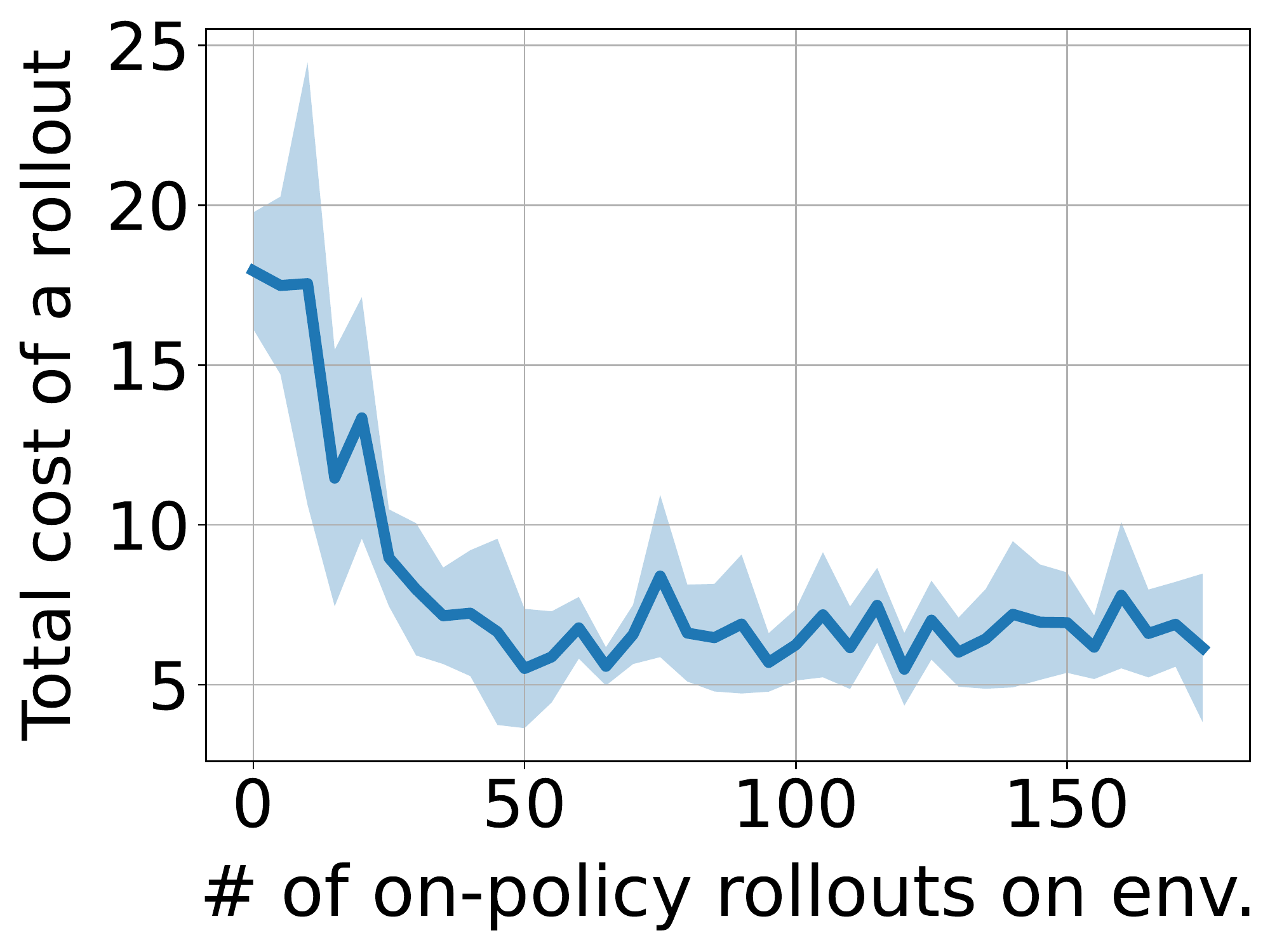}
		\caption{}
		\label{fig.tf.task2.plots.2}
	\end{subfigure}
	\begin{subfigure}{.243\textwidth}
		\centering
		\includegraphics[width=\linewidth]{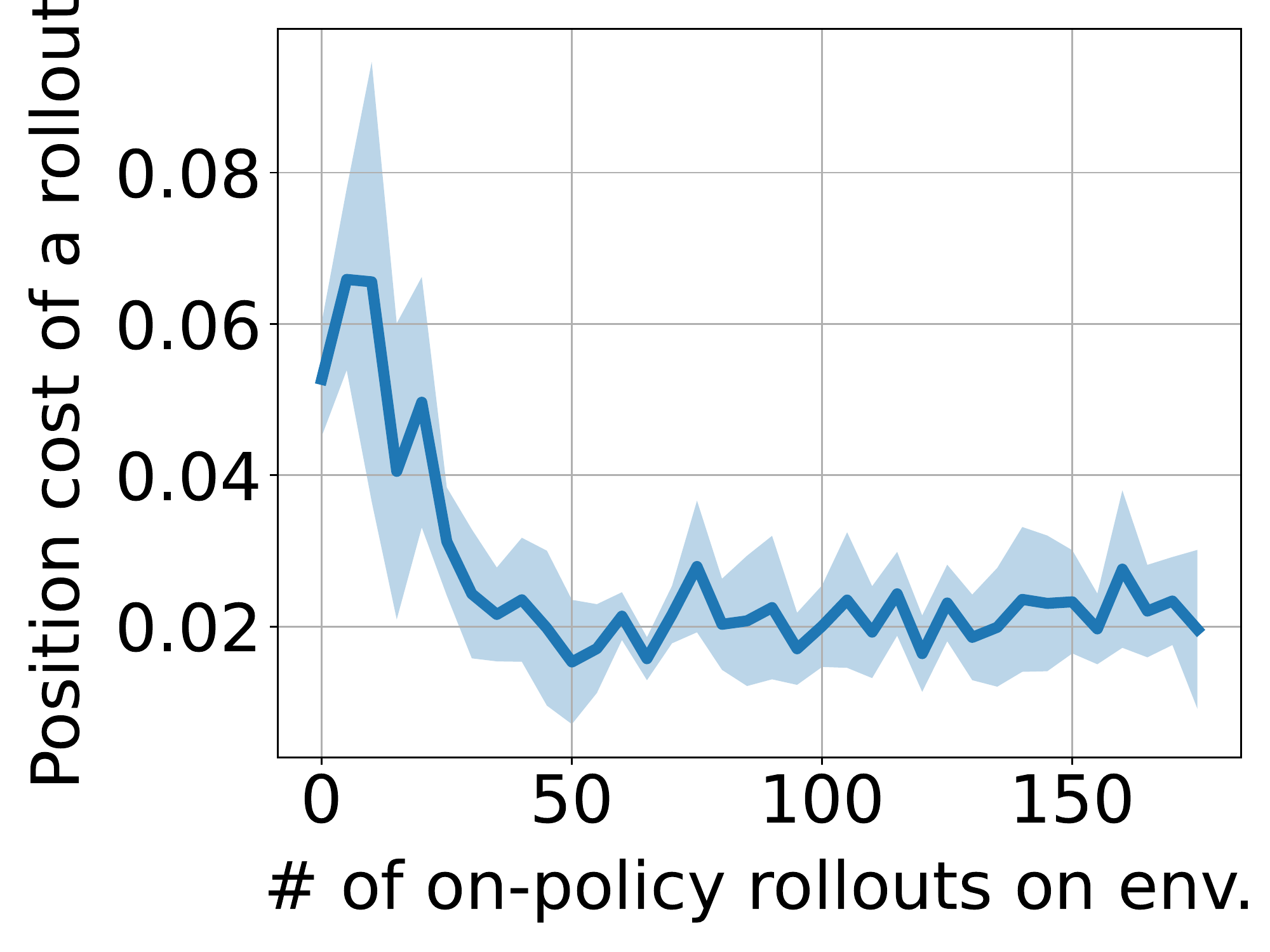}
		\caption{}
		\label{fig.tf.task2.plots.3}
	\end{subfigure}
	\begin{subfigure}{.243\textwidth}
		\centering
		\includegraphics[width=\linewidth]{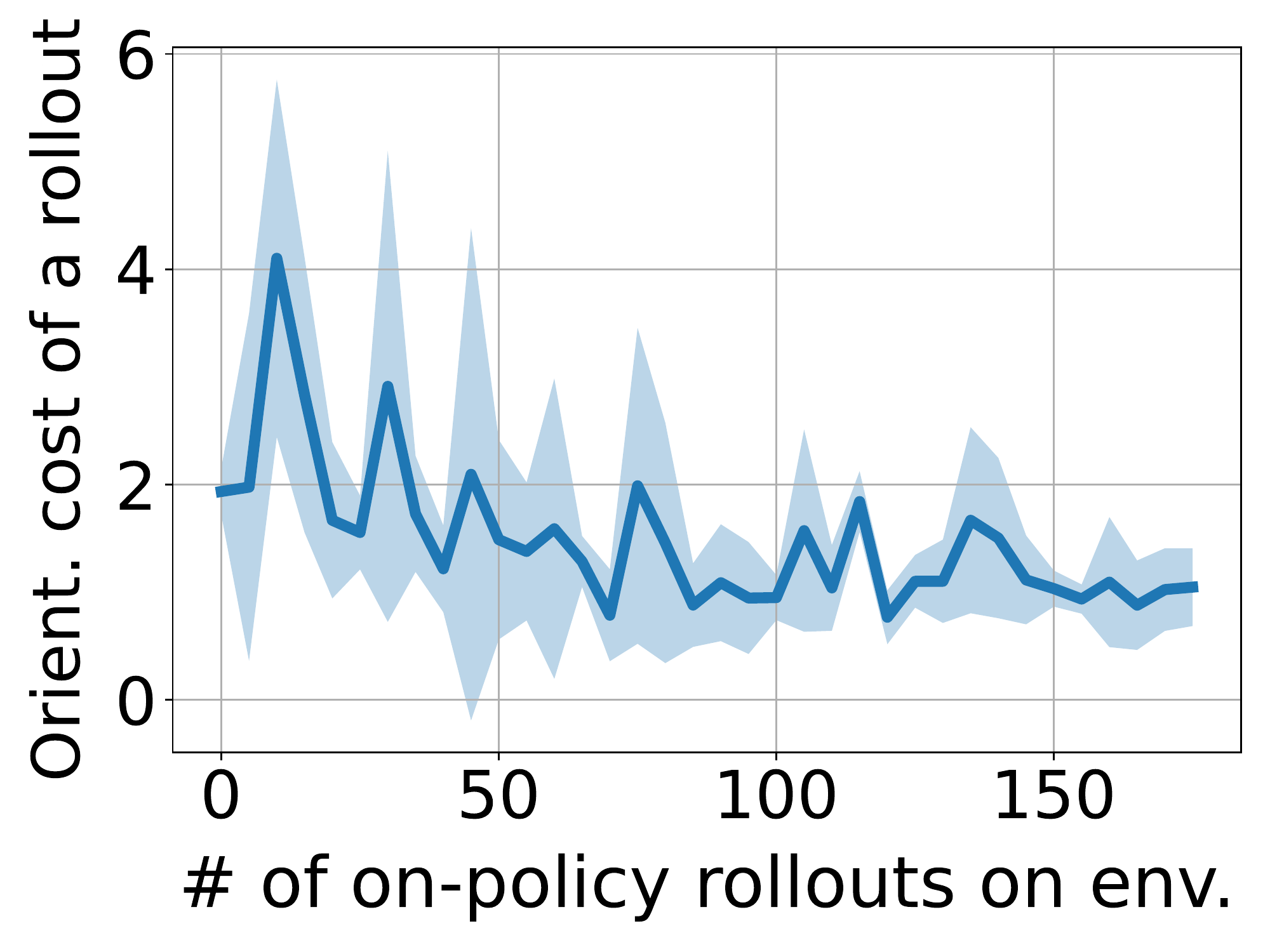}
		\caption{}
		\label{fig.tf.task2.plots.4}
	\end{subfigure}
	\caption{ \small
	Learning curves of the three-finger manipulation system for the Cube Moving task. Each curve is the average of five random seeds, and its shaded area shows the  standard deviation. All results  here are shown on an on-policy rollout basis, and each on-policy rollout  is a result of running the reduced-order  $\boldsymbol{g}$-MPC controller  on the three-finger system (environment), i.e., the (unknown) full-order $\boldsymbol{f}()$. Detailed explanations are given in text.
	} 
	\label{fig.tf.task2.plots}
 \vspace{-10pt}
\end{figure*}

\subsubsection{Analysis of Reduced-Order Hybrid Modes}\label{section.tf.task1.analysis}
In this session,  we will detail how the learned task-driven reduced-order LCS  $\boldsymbol{g}()$ enables the three-finger system to reason about the contact decision in the Cube Tuning task. 
In each row of Fig. \ref{fig.tf.task1.demo}, given a random target orientation,  we show the   rollout trajectories (left) and key-time-step snapshots (right) of the three-finger manipulation  environment by running the learned  $\boldsymbol{g}$-MPC. The rollout horizon $H=20$.

Specifically, in Fig. \ref{fig.tf.task1.demo11} and \ref{fig.tf.task1.demo12}, the   target orientation is $\alpha^\goal=-1.29$ (rad), shown by a reference  at  the lower left corner. The upper panel  in Fig. \ref{fig.tf.task1.demo11} shows the mode activation of each dimension of $\boldsymbol{\lambda}$ in $\boldsymbol{g}$-MPC. Here, the black bricks mean $\boldsymbol{\lambda}>0$, and blank means $\boldsymbol{\lambda}=0$. For exposition simplicity, we use $\sign(\boldsymbol{\lambda})$ to denote the mode activation. For example,  at time step $t=9$ is $\sign(\boldsymbol{\lambda}_9)=[0, 1, 1, 1, 0]\tran$.  The bottom panel in Fig. \ref{fig.tf.task1.demo11} shows the cube's orientation angle  trajectory, where  the segments of different colors show different mode activation. The upper row in Fig. \ref{fig.tf.task1.demo12} shows the snapshots of the simulator at some key time steps of the same rollout, and the lower row gives the zoom-in details. Physically, the red fingertip begins  touching the cube at time step 7, pushes the cube to rotate clockwise during steps 7-14 (during this period it also slips on the surface of the cube), and then separates from the cube at time step 15. The blue fingertip begins touching the cube at time step 9,  then pushes the cube to rotation closewise, and finally separates from the cube at step 15. The green fingertip  begins touching the cube at time step 11, and continue pushing the cube until the end of the rollout.
By connecting Fig. \ref{fig.tf.task1.demo11} and \ref{fig.tf.task1.demo12}, we can observe the  empirical correspondence between  $\boldsymbol{g}()$'s mode  activation in Fig. \ref{fig.tf.task1.demo11} and physical  interaction in Fig.\ref{fig.tf.task1.demo12}, listed in Table \ref{table.tf.demo1.table1}.

In Fig. \ref{fig.tf.task1.demo21} and \ref{fig.tf.task1.demo22}, the target orientation  is $\alpha^\goal=0.148$ (rad). Similar to the above description,  Table \ref{table.tf.demo1.table2} gives the empirical correspondence between the mode activation in Fig. \ref{fig.tf.task1.demo21} and physical contact interaction in Fig. \ref{fig.tf.task1.demo22}.
Table \ref{table.tf.demo1.table2} shows a more interesting connection between the  mode activation in $\boldsymbol{g}()$ and the physical contact. From  the bottom panel of Fig. \ref{fig.tf.task1.demo21}, we see the three fingers  turns the cube to the target  at time step 5, and from that on, it slightly shakes the cube around the target. This makes the contact interaction between the cube and table also change alternatively, i.e., between the \texttt{CW rotational slip} and \texttt{CCW rotational slip}. This alternative physical interactions have been captured in  the upper panel of Fig. \ref{fig.tf.task1.demo21}, where
 the mode activation from time step 5 also changes alternatively. This clearly shows the connections between  mode in the reduced-order LCS $\boldsymbol{g}()$ and  physical contacts.  For example,    the mode  $\sign(\boldsymbol{\lambda}_{t})=[0, 1, 0, 0, 0]\tran$ is  for the cube \texttt{CW rotational slip} (blue segments), while $\sign(\boldsymbol{\lambda}_{t})=[1, 1, 0, 0, 0]\tran$ and $\sign(\boldsymbol{\lambda}_{t})=[1, 0, 0, 0, 0]\tran$ are for  \texttt{CCW rotational slip} (green and purple segments).

From Tables \ref{table.tf.demo1.table1} and \ref{table.tf.demo1.table2}, we have the following comments.

(i)  The learned  LCS  $\boldsymbol{g}()$ can approximately capture the hybrid nature of the physical system. Since we limit the maximum   mode count in $\boldsymbol{g}()$  to  $32$, some physical contact interactions share the same mode activation of $\sign(\boldsymbol{\lambda})$. For example, in Table \ref{table.tf.demo1.table1}, the mode $\sign(\boldsymbol{\lambda})=[0, 0, 0, 1, 0]\tran$ captures two interactions between the cube and fingertips: (\texttt{R, G, B separate})  and   (\texttt{G, B separate} and \texttt{R touching}).

(ii) Note that Tables \ref{table.tf.demo1.table1} and  \ref{table.tf.demo1.table2} come from  empirical observation. For contact   interactions that are very similar in human eyes, there could exist unnoticeable physical differences,  leading to different modes in $\boldsymbol{g}()$.
For example, in Table \ref{table.tf.demo1.table2}, \texttt{CCW rotational slip (large)} corresponds to  $\sign(\boldsymbol{\lambda})=[1, 0, 0, 0, 0]\tran$, while \texttt{CCW rotational slip (small)} to $\sign(\boldsymbol{\lambda})=[1, 1, 0, 0, 0]\tran$.  There is no tight connection from the mode of $\boldsymbol{g}()$ to physical phenomena.

\subsection{Cube Moving Task}\label{section.tf.task2}
This session presents the results and analysis of using the proposed method to solve the Cube Moving task. In this task, we set weights in (\ref{equ.tf.cost1}) as
\begin{equation}\label{equ.tf.cost.weights.task2}
\begin{aligned}
&\boldsymbol{w}^c=\begin{bmatrix}
12.0 & 200.0& 0.2
\end{bmatrix}\tran,\\
&\boldsymbol{w}^h=\begin{bmatrix}
6.0 & 200.0& 1.0
\end{bmatrix}\tran.
\end{aligned}
\end{equation}
The above weight values are not deliberately picked. In fact, the  performance is not sensitive to the choice of  $\boldsymbol{w}^c$ and $\boldsymbol{w}^h$. As will be discussed  in Section \ref{trifinger.discussion}, the similar learning and  task performance permits a wide selection of  weight values.

\begin{table}[h]
\begin{center}
\caption{Performance of  learned $\boldsymbol{g}$-MPC for Cube Moving}
\label{table.tf.task2.table1}
\begin{threeparttable}
\begin{tabular}{l l}
    \toprule
    \textbf{Results} &  Value (mean+std)  \\ 
    \midrule
      Total number of hybrid modes in  $\boldsymbol{g}$-MPC   &    15 \\[4pt]
      Terminal position error  $\norm{\boldsymbol{p}_{\obj, H}{-}\boldsymbol{p}^\goal}$ & 0.0076 $\pm$ 0.0010 (m)\\[2pt]
       Terminal position error (rel.) $\frac{\norm{\boldsymbol{p}_{\obj, H}-\boldsymbol{p}^\goal}^2}{\norm{\boldsymbol{p}^\goal}^2}$ & 3.8\% $\pm$ 1.3\%\\[8pt]
        Terminal orientation error  $|\alpha_{\obj, H}-\alpha^\goal|$   & 0.065 $\pm$ 0.021 (rad)\\[4pt]
    Total training time     & 4.6 $\pm$ 0.4 (mins)\\[4pt]
    Total \# of environment samples in training      & $\le$4k\\[4pt]
     Running frequency of  $\boldsymbol{g}$-MPC controller   & $>$ 30 Hz\\[4pt]
   Max. $\frac{\text{w}^{\text{mag}}_{\text{disturb}}}{m_{\obj}}$ until $\norm{\boldsymbol{p}_{\obj, H}-\boldsymbol{p}^\goal}\geq0.01$   &  0.9 m/s$^2$\\[2pt]
      \bottomrule
\end{tabular}
\end{threeparttable}
\end{center}
\end{table}

\begin{figure*}[h]
	\centering
		\begin{subfigure}{0.95\textwidth}
		\centering
		\includegraphics[width=\linewidth]{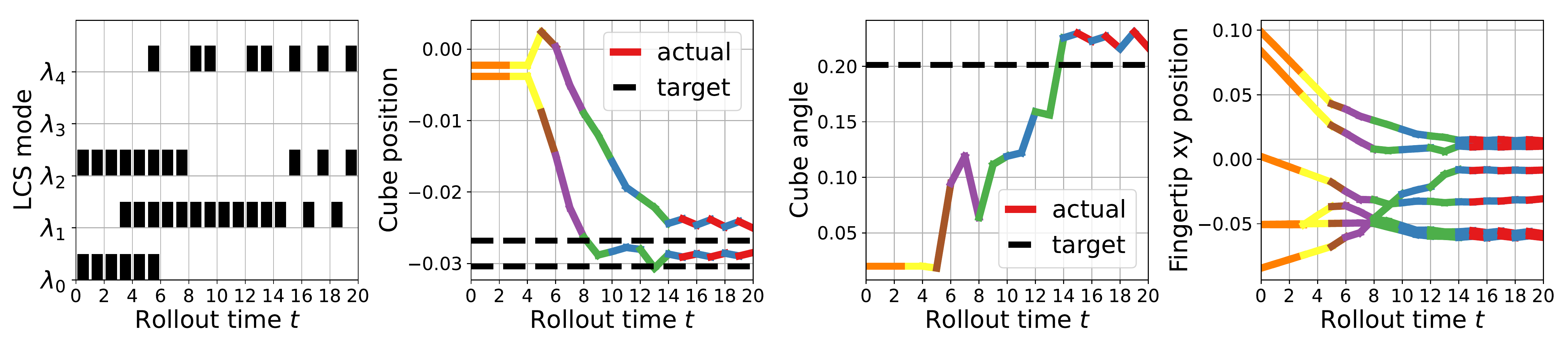}
		\caption{}
		\label{fig.tf.task2.demo11}
	\end{subfigure}
    \smallskip
	\begin{subfigure}{\textwidth}
		\centering
		\includegraphics[width=\linewidth]{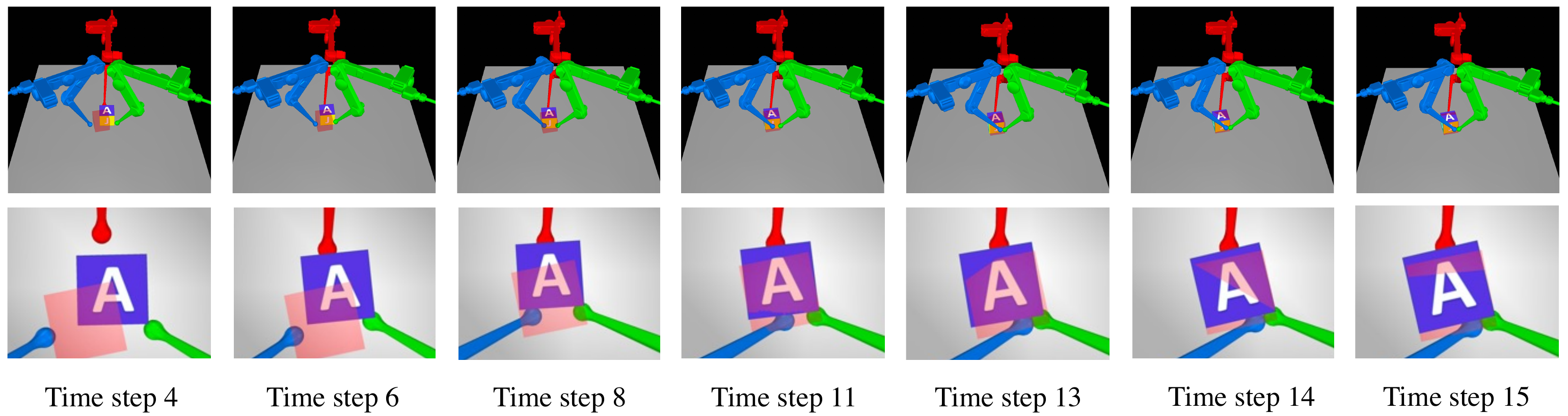}
		\caption{}
		\label{fig.tf.task2.demo12}
	\end{subfigure}
	\caption{
	\small One rollout by running the learned reduced-order  $\boldsymbol{g}$-MPC  on the three-finger manipulation system for the Cube Moving task. (a) shows the mode activation in $\boldsymbol{g}()$  (first panel), the trajectory of  the cube xy position  (second panel),  the trajectory of the cube  angle (third panel), and the trajectories of all three fingertip xy positions (fourth panel), over the duration of the rollout. Here, in the model activation  panel, black brick indicate ${\lambda}_i>0$ and blank  ${\lambda}_i=0$. In the trajectory panels,  the segments are colored differently, corresponding to different mode activation. (b) shows the snapshots of the environment  at  key time steps of the $\boldsymbol{g}$-MPC rollout. Here, the upper row of (b) shows the whole environment, and the lower row show the zoom-in  details. Explanations and analysis are given in Section \ref{section.trifinger.task2.analysis} and  Table \ref{table.tf.task2.table2}.
	} 
	\label{fig.tf.task2.demo}
 \vspace{-10pt}
\end{figure*}

\subsubsection{Results}
Similar to the previous session, we present the learning curves of the three-finger manipulation system in Fig. \ref{fig.tf.task2.plots}, and list the key  results  in Table \ref{table.tf.task2.table1}. Each result is the average of five random seeds. Specifically, Fig. \ref{fig.tf.task2.plots.1} shows the relative model   error of the learned  LCS $\boldsymbol{g}()$ evaluated on the $\boldsymbol{g}$-MPC policy data, defined in (\ref{equ.pwa.mpe}).
Fig. \ref{fig.tf.task2.plots.2}  shows the total cost of a rollout with the $\boldsymbol{g}$-MPC controller, defined in (\ref{equ.pwa.mpccost}). 
Fig. \ref{fig.tf.task2.plots.3} shows the   position cost of a  rollout with the $\boldsymbol{g}$-MPC controller, defined as 
\begin{equation}\label{equ.tf.pos_cost}
E_{{\boldsymbol{\beta}}\sim p({\boldsymbol{\beta}})} \E_{\boldsymbol{x}\sim p_{\boldsymbol{\beta}}(\boldsymbol{x}_0)}\sum\nolimits_{t=0}^{H}\norm{\boldsymbol{p}_{\obj, t}-\boldsymbol{p}^\goal}^2.
\end{equation}
Fig. \ref{fig.tf.task2.plots.4}  shows the orientation cost (\ref{equ.tf.ori_cost}) of a rollout with the $\boldsymbol{g}$-MPC controller.  
Table \ref{table.tf.task2.table1} lists some key quantitative results. Here,  the cube's terminal position error  $\norm{\boldsymbol{p}_{\obj, H}{-}\boldsymbol{p}^\goal}$ and terminal orientation error $|\alpha_{\obj, H}-\alpha^\goal|$ are calculated at the end (time step $H$) of a rollout.  In the last row of Table \ref{table.tf.task2.table1}, we   test the  robustness of the closed-loop $\boldsymbol{g}$-MPC controller against the external disturbance forces added to the cube during its motion. Here,
 we apply an external   disturbance wrench (3D torque and 3D forces), sampled from $ U[-\text{w}^{\text{mag}}_{\text{disturb}}, \text{w}^{\text{mag}}_{\text{disturb}}]$, during each time interval ($0.1$s) of the rollout steps. We increase the disturbance magnitude $\text{w}^{\text{mag}}_{\text{disturb}}$ until the resulting $\boldsymbol{g}$-MPC rollout has an average cube terminal position error  $\norm{\boldsymbol{p}_{\obj, H}-\boldsymbol{p}^\goal}\geq0.01$ (m).   We report the result using ${\text{w}^{\text{mag}}_{\text{disturb}}}/{m_{\obj}}$, with $m_{\obj}$ the cube mass.
Based on the results in  Fig. \ref{fig.tf.task2.plots} and Table \ref{table.tf.task2.table1}, we have the following conclusions.

(i) Without any prior knowledge, the proposed method  learns a reduced-order LCS and successfully solves the Cube Moving manipulation task  within  5 minutes of wall-clock time. The reduced-order model $\boldsymbol{g}()$ leads to a real-time $\boldsymbol{g}$-MPC controller with running frequency $>30$ Hz.  The method requires collecting less than 4k data points (175 rollouts $\times$ 20 steps/rollout)  from the environment.

(ii) The learned reduced-order LCS  $\boldsymbol{g}()$ only permits 32 modes (among them 15 are used for the task), which is much fewer than the estimated number of hybrid modes in full-order dynamics, which could be thousands. 

(iii) The  reduced-order  $\boldsymbol{g}$-MPC controller shows  robustness against large external wrench  disturbances. This is an advantage of using the closed-loop   $\boldsymbol{g}$-MPC controller.

\begin{table*}[h]
\begin{center}
\caption{Approximate correspondence between LCS mode  activation in Fig. \ref{fig.tf.task2.demo11} and  contact interaction in Fig.\ref{fig.tf.task2.demo12}. }
\label{table.tf.task2.table2}
\begin{threeparttable}
\begin{tabular}{l c c c c}
    \toprule
     \textbf{Duration $t$} & \textbf{Mode activation in $\boldsymbol{g}()$} & \textbf{Interaction between fingertips and cube}  
     & \textbf{Interaction between cube and table} 
     \\ 
    \midrule
      $t=0,1,2$ &  
            $\sign(\boldsymbol{\lambda})=[1, 0, 1, 0, 0]\tran$     & 
      \texttt{R, G, B separate}   & \texttt{stick to table}
      \\[3pt]
      $t=3, 4$ &  
            $\sign(\boldsymbol{\lambda})=[1, 1, 1, 0, 0]\tran$     & 
 \texttt{R, B separate} and \texttt{G separate}  & \texttt{translational slip}
      \\[5pt]
    $t=5,6,7$ &  
     {\begin{tabular}{@{}l@{}} 
     $\sign(\boldsymbol{\lambda})=[1, 1, 1, 0, 1]\tran$   \\ 
     \qquad \qquad or $[0, 1, 1, 0, 0]\tran$ 
     \end{tabular}} 
     & \texttt{R, G stick} \& \texttt{B separate}
     & \texttt{translational \& rotational slip}
      \\[10pt]
         $t=8, 9$ &  
            $\sign(\boldsymbol{\lambda})=[0, 1, 0, 0, 1]\tran$     & 
     \texttt{R, G, B stick}   
      &  \texttt{translational \& rotational slip}
      \\[5pt]
      $t=10, 11$ &  
       $\sign(\boldsymbol{\lambda})=[0, 1, 0, 0, 0]\tran$     & 
            \texttt{R, G stick}    
           and \texttt{B right slip} 
      &  \texttt{translational \& rotational slip}
      \\[5pt]
        $t=12, 13$ &  
            $\sign(\boldsymbol{\lambda})=[0, 1, 0, 0, 1]\tran$     & 
              \texttt{R stick}  and \texttt{B contacting G} 
      &  \texttt{translational \& rotational slip}
      \\[5pt]
    $t=14, 16, 18$ &  
    $\sign(\boldsymbol{\lambda})=[0, 1, 0, 0, 0]\tran$      & 
              \texttt{R, B, G stick}  and \texttt{B contacting G}   &  \texttt{CCW rotational slip}
      \\[5pt]
    $t=15, 17, 19$ &  
    $\sign(\boldsymbol{\lambda})=[0, 0, 1, 0, 1]\tran$     & 
              \texttt{R, B, G stick}  and \texttt{B contacting G}   &  \texttt{CW rotational slip}
      \\
      \bottomrule
\end{tabular}
\end{threeparttable}
\end{center}
\end{table*}

\begin{figure*}[h]
	\centering	\includegraphics[width=0.99\textwidth]{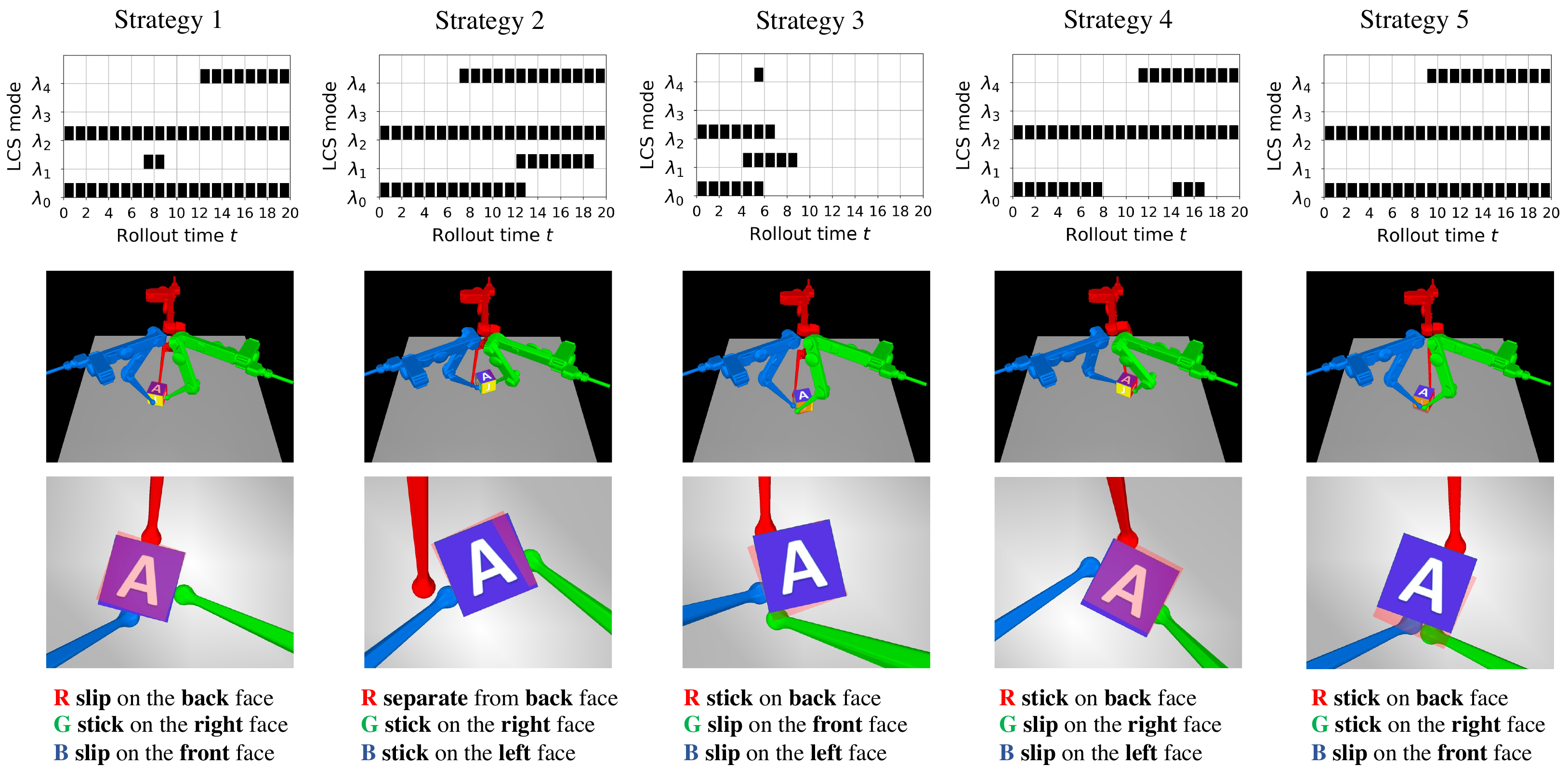}
	\caption{\small Different contact strategies generated by the same learned reduced-order LCS $\boldsymbol{g}()$ in its $\boldsymbol{g}$-MPC policy rollout given different targets. The first row shows the   mode activation of $\boldsymbol{g}$ over the rollout; the second row shows the snapshot of the environment at the end of the rollout; and the third row shows the zoom-in details of contact interactions. The bottom title in each column describes the main interactions for that strategy. Analysis is given in Section \ref{section.tf.strategies}.}
	\label{fig.tf.task2.grasp}
   \vspace{-10pt}
\end{figure*}

\smallskip

\subsubsection{Analysis of Reduced-Order Hybrid Modes} \label{section.trifinger.task2.analysis}

Fig. \ref{fig.tf.task2.demo} shows one rollout of running  the learned  $\boldsymbol{g}$-MPC on the three-finger manipulation system (the environment). Fig. \ref{fig.tf.task2.demo11} plots the trajectory of  mode activation in $\boldsymbol{g}()$ (first panel), the trajectory of the cube position (second panel),  the trajectory of the cube orientation angle (third panel), and the trajectories of xy position of three fingertips (fourth panel), over the duration of  rollout. All trajectories are colored differently for different mode activation in $\boldsymbol{g}()$.  Fig. \ref{fig.tf.task2.demo12} shows the snapshots of the environment at some key  time steps of the rollout.
As done in the previous task,  Table \ref{table.tf.task2.table2}  lists the empirical connection between mode activation in $\boldsymbol{g}$-MPC in Fig. \ref{fig.tf.task2.demo11} and physical contact interaction in Fig. \ref{fig.tf.task2.demo12}.

Results in Fig. \ref{fig.tf.task2.demo} and Table \ref{table.tf.task2.table2}  show  rich contact interactions  in the Cube Moving  task. Different hybrid mode  of  $\boldsymbol{g}()$ can approximately capture different contact interactions during  rollout. Those interactions include \texttt{separate}, \texttt{stick}, and \texttt{slip} between  fingertips and  cube,  \texttt{stick},  \texttt{CCW rotational slip}, and  \texttt{CW rotational slip} between  cube and  table, and \texttt{B contacting G} between fingertips, happening at different time steps.  Notably, in Fig. \ref{fig.tf.task2.demo11}, for $t\geq 14$, the three fingers start slightly shaking the cube, as shown by the cube position and angle trajectories, and active mode in $\boldsymbol{g}()$ also  changes alternatively, as shown in the mode activation panel. Those observations indicate the modes of the learned reduced-order  LCS are able to approximately capture the rich contact interactions in the manipulation system.

\subsubsection{Generation of Different Manipulation Strategies}
\label{section.tf.strategies}

Notably, for different  target poses (sampled from (\ref{equ.tf.task2})), we observe that the learned $\boldsymbol{g}$-MPC  produces different manipulation strategies to move the cube. We  show this in Fig. \ref{fig.tf.task2.grasp}.

Different columns in Fig. \ref{fig.tf.task2.grasp} shows different manipulation strategies given different targets. Note that all those strategies are generated from the same learned reduced-order LCS $\boldsymbol{g}()$ in its $\boldsymbol{g}$-MPC policy rollout, given different targets. The first row in Fig. \ref{fig.tf.task2.grasp} shows the   mode activation in $\boldsymbol{g}()$ during the rollout with $\boldsymbol{g}$-MPC controller; the second row shows the snapshot of the environment at the end of rollout; and the third row shows the zoom-in details of the contact interaction. The bottom title in each column describes the main physical interactions for the rollout. From Fig. \ref{fig.tf.task2.grasp}, we have the following comments.

(i) Fig. \ref{fig.tf.task2.grasp} clearly shows the learned reduced-order LCS $\boldsymbol{g}()$ enables generating different strategies for different targets. Particularly, $\boldsymbol{g}()$ enables the three fingertips to choose different faces (\texttt{right}, \texttt{left}, \texttt{front}, and \texttt{back}) of the cube with different contact interactions (\texttt{separate}, \texttt{stick}, and \texttt{slip}). For example, the blue fingertip chooses the \texttt{front} face in Strategies 1 and 5 while the \texttt{left} face in the others. The red fingertip is \texttt{slip} in Strategy 1, \texttt{separate} in Strategy 2, and \texttt{stick} in others.

(ii) Jointly looking at the mode activation of $\boldsymbol{g}()$ in the first row of Fig. \ref{fig.tf.task2.grasp}, we observe that  the mode activation at the beginning of all rollouts, e.g., $t\leq 4$, are quite similar since the cube  during this period is still, and all fingertips are separate from  the cube. After $t>4$, different modes in $\boldsymbol{g}$-MPC begin to activate, leading to different manipulation strategies.

(iii) Recall that all strategies shown  in Fig. \ref{fig.tf.task2.grasp} are generated by the same learned reduced-order LCS $\boldsymbol{g}()$. The results clearly show the effectiveness of the learned reduced-order $\boldsymbol{g}$-MPC to capture and make use of its task-relevant hybrid modes to produce different  strategies for  the  manipulation task.

\subsubsection{Occasional  Reorientation Failure}
We  report some failure cases in the above cube moving manipulation task. Fig. \ref{fig.tf.task2.rotationfail} shows one example of reorientation failure. Here, some  snapshots at key time steps of the $\boldsymbol{g}$-MPC rollout  are shown. At time step 10 (middle column), the three fingertips had successfully moved and turned the cube to the target pose. However, at the subsequent time steps, the green fingertip continues to slide along the  cube surface,  leading to   the misalignment of the cube orientation  (third column).

\begin{figure}[h]
	\centering	\includegraphics[width=0.48\textwidth]{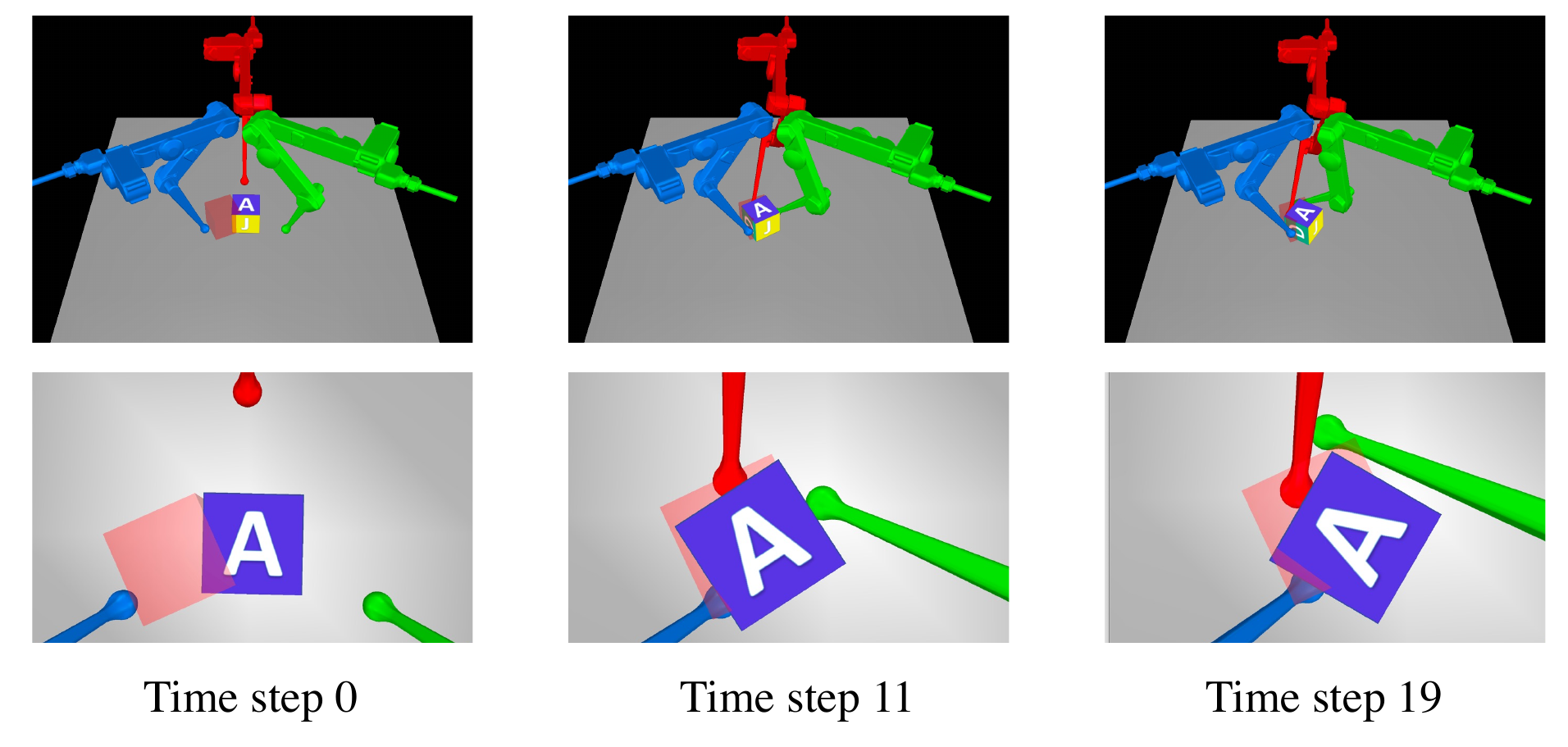}
	\caption{\small An  example of  reorientation failure in a  rollout.}
	\label{fig.tf.task2.rotationfail}
\end{figure}

We observed that  reorientation failure are target dependent,
meaning that  reorientation failures happen more frequently at some targets than at others. 
Also, changing the random seed for  target distribution (\ref{equ.tf.task2}) also changes the failure target locations. 
This makes us believe that the reorientation failure could be caused by  insufficient target sampling at such regions.  In fact, the whole training process sampled from fewer than 200 target poses from (\ref{equ.tf.task2}). Some regions of the target space  could be less  sampled than other regions, leading to the learned $\boldsymbol{g}()$ not well representing those regions. We expect that those failures could be reduced by decreasing the target space. In fact, in our previous Cube Turning task, since the target space is one-dimensional, we have not experienced the reorientation failures in that task.


\subsection{Discussion}\label{trifinger.discussion}
We conclude this section with some additional performance evaluations of the proposed method, and note some limitations (and future work) of the proposed method.

\medskip

\subsubsection{LCS with Different Hybrid Modes}
In the above three-finger manipulation tasks, we have used  LCS models $\boldsymbol{g}()$ in (\ref{equ.lcs.reduced}) with a fixed $\dim \boldsymbol{\lambda}=5$, which allows the representation of  32  modes. Results in Table \ref{table.tf.task1.table1} and Table \ref{table.tf.task2.table1}  show that the learned reduced-order LCS  has not used up all of those modes. Therefore, a natural question is whether it is possible to learn a LCS with fewer hybrid modes. To show this, we learn a reduced-order LCS  $\boldsymbol{g}()$ with different $\dim \boldsymbol{\lambda}$ for the Cube Turning  task, under the same other settings as in Section \ref{section.tf.task1}. We show the results in Fig. \ref{fig.discussion.comparedim}. Here, for each $\dim \boldsymbol{\lambda}$ case, we run the experiments with five random seeds, and the mean and variance are computed across different runs. 

  \begin{figure}[h]
	\centering
		\begin{subfigure}{.22\textwidth}
		\centering
		\includegraphics[width=\linewidth]{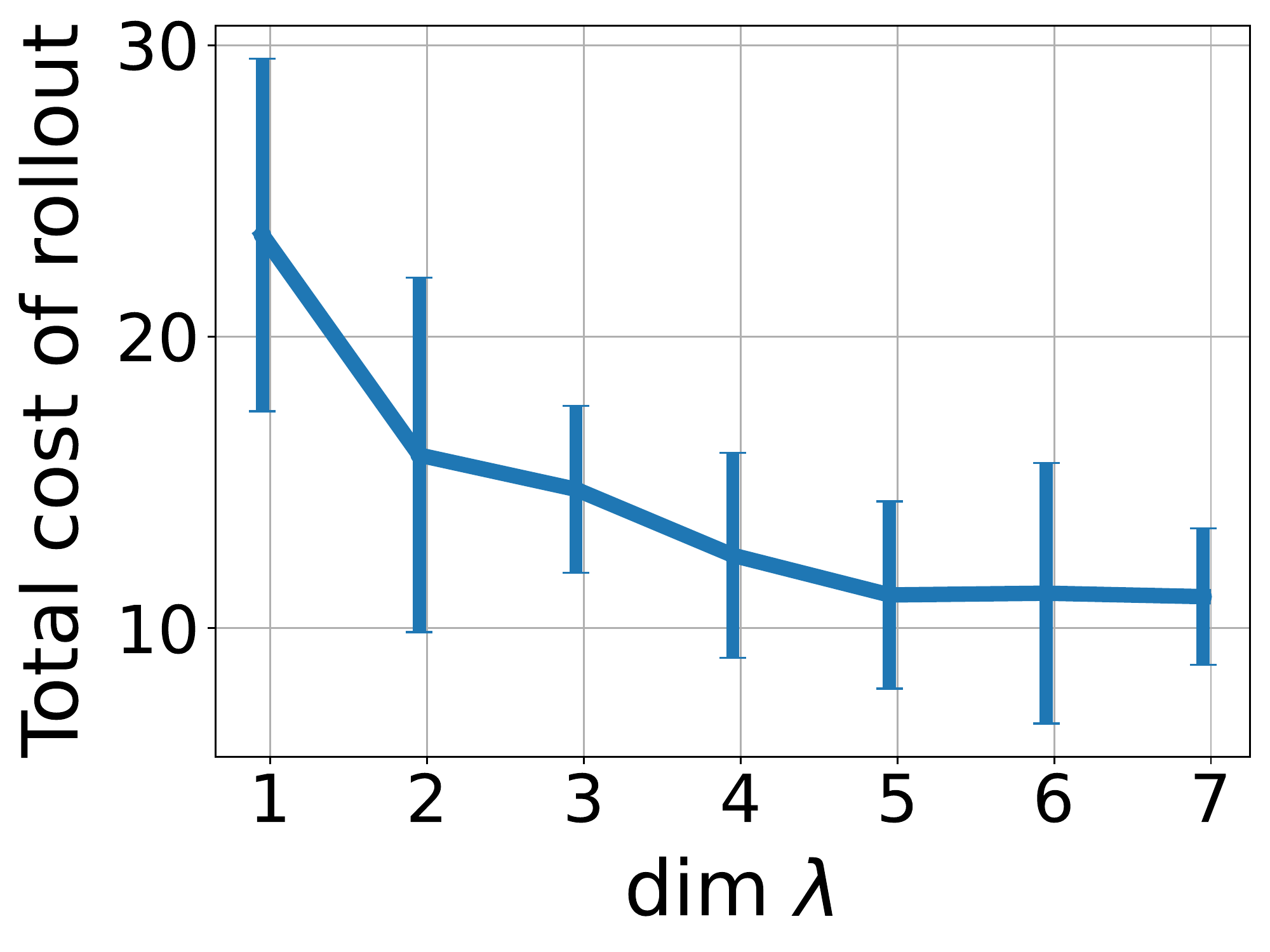}
		\caption{}
		\label{fig.discussion.comparedim.1}
	\end{subfigure}
		\begin{subfigure}{.22\textwidth}
		\centering
		\includegraphics[width=\linewidth]{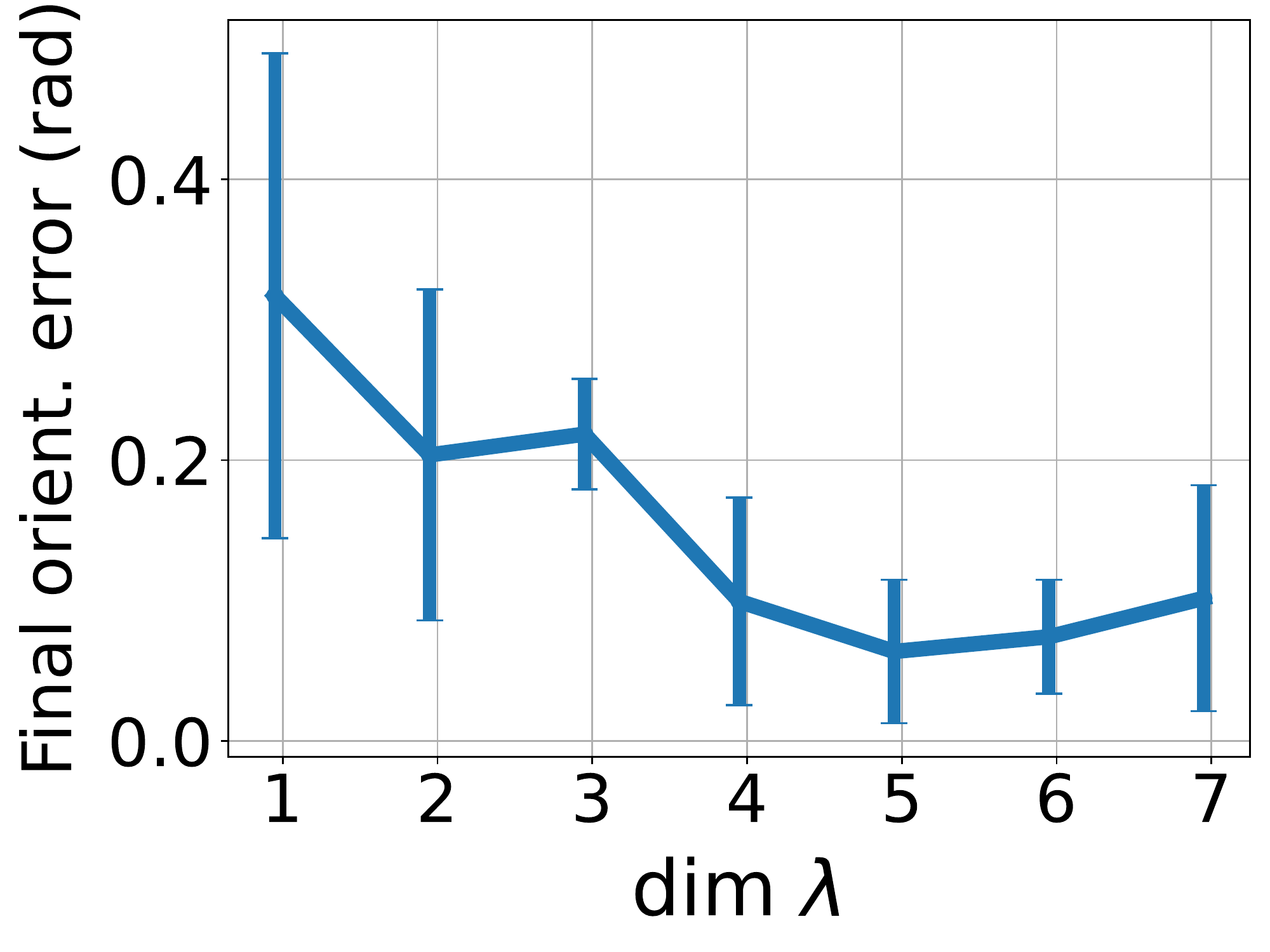}
		\caption{}
		\label{fig.discussion.comparedim.2}
	\end{subfigure}
	\begin{subfigure}{.22\textwidth}
		\centering
		\includegraphics[width=\linewidth]{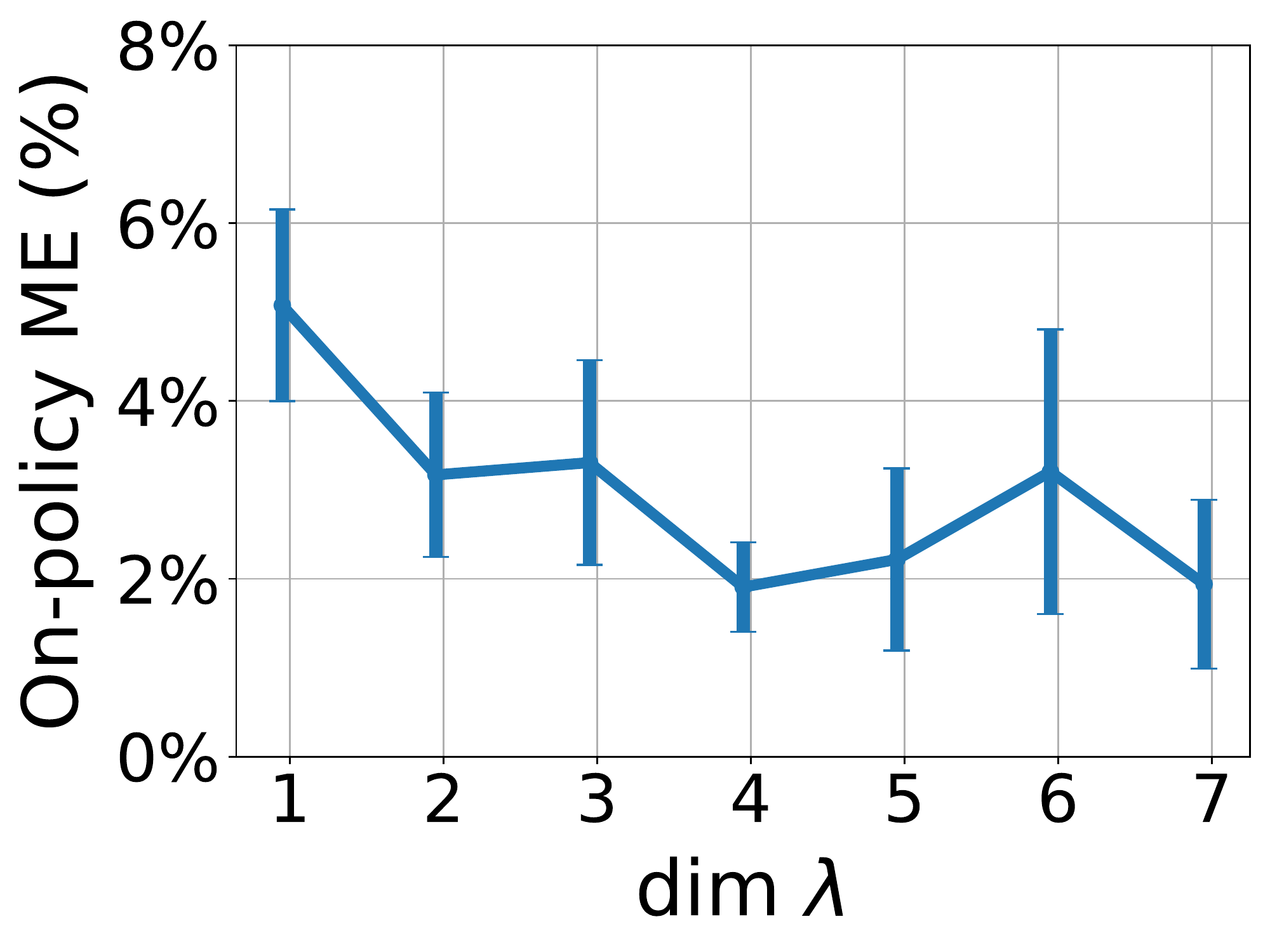}
		\caption{}
		\label{fig.discussion.comparedim.3}
	\end{subfigure}
	\begin{subfigure}{.22\textwidth}
		\centering
		\includegraphics[width=\linewidth]{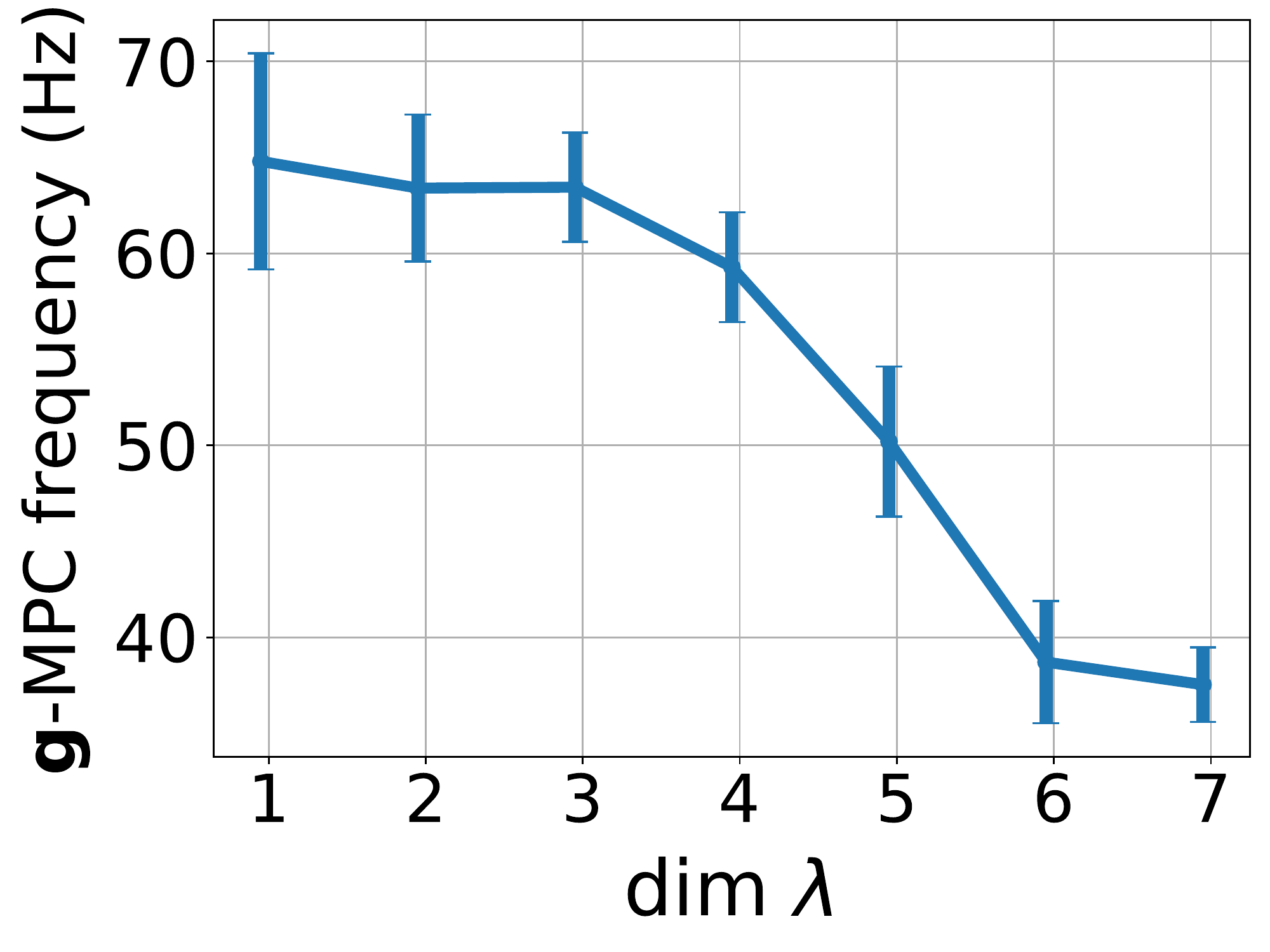}
		\caption{}
		\label{fig.discussion.comparedim.4}
	\end{subfigure}
	\caption{ \small
	Testing performance of the learned LCS $\boldsymbol{g}()$ with different $\dim \boldsymbol{\lambda}$ for the Cube Turning  task. The metrics in (a)-(c)   follow the same definitions as the ones in Section \ref{section.tf.task1},  and (d) is running frequency of $\boldsymbol{g}$-MPC. For each $\dim \boldsymbol{\lambda}$ case, the mean and variance are calculated across five  runs with different random seeds. The  other settings follow the ones  in Section \ref{section.tf.task1}. 
	} 
	\label{fig.discussion.comparedim}
\end{figure}

 Fig. \ref{fig.discussion.comparedim}  shows that   a LCS of fewer  modes, e.g.,  $\dim\boldsymbol{\lambda}\leq 3$,  leads to degraded performance. The previous Table \ref{table.tf.task1.table1} has showed that  successful manipulation  needs around 14 hybrid modes in $\boldsymbol{g}()$. Thus, the successful manipulation  on average requires   at least $\dim\boldsymbol{\lambda}\geq 4$.  Fig. \ref{fig.discussion.comparedim}  confirms this by showing   improved performances if $\dim\boldsymbol{\lambda}\geq 4$.   Fig. \ref{fig.discussion.comparedim} also shows that  further increasing of $\dim\boldsymbol{\lambda}$, say $\dim\boldsymbol{\lambda}=6$, will not help the performance too much. Also,  increasing $\dim\boldsymbol{\lambda}$ will slow  the speed of the closed-loop $\boldsymbol{g}$-MPC controller in Fig. \ref{fig.discussion.comparedim.4}.  In practice,  the  choice of $\dim \boldsymbol{\lambda}$  depends on  tasks and systems, and one typically needs to find a $\dim \boldsymbol{\lambda}$ (via trial and error) by   balancing  its   performance and computational complexity.

\medskip

\subsubsection{Choice of Cost Weights}
In this session, we  investigate how the choice of  cost  weights $\boldsymbol{w}^c=[w_1^c, w_2^c, w_3^c]\tran$ and  $\boldsymbol{w}^h=[w_1^h, w_2^h, w_3^h]\tran$
 in (\ref{equ.tf.cost1}) affects the algorithm performance. We  apply the proposed method to learn a reduced LCS   of $\dim\boldsymbol{\lambda}=5$ for  the Cube Turning task,   with the same other conditions as in
Section \ref{section.tf.task1}, except varying the values of $(w_1^c,  w_3^c)$ and $(w_1^h,  w_3^h)$ (note $w_2^c=w_2^h=0$ in the Cube Turning task). The results are shown in Fig. \ref{fig.discussion.cost_weights}.
  \begin{figure}[h]
	\centering
		\begin{subfigure}{.24\textwidth}
		\centering
		\includegraphics[width=\linewidth]{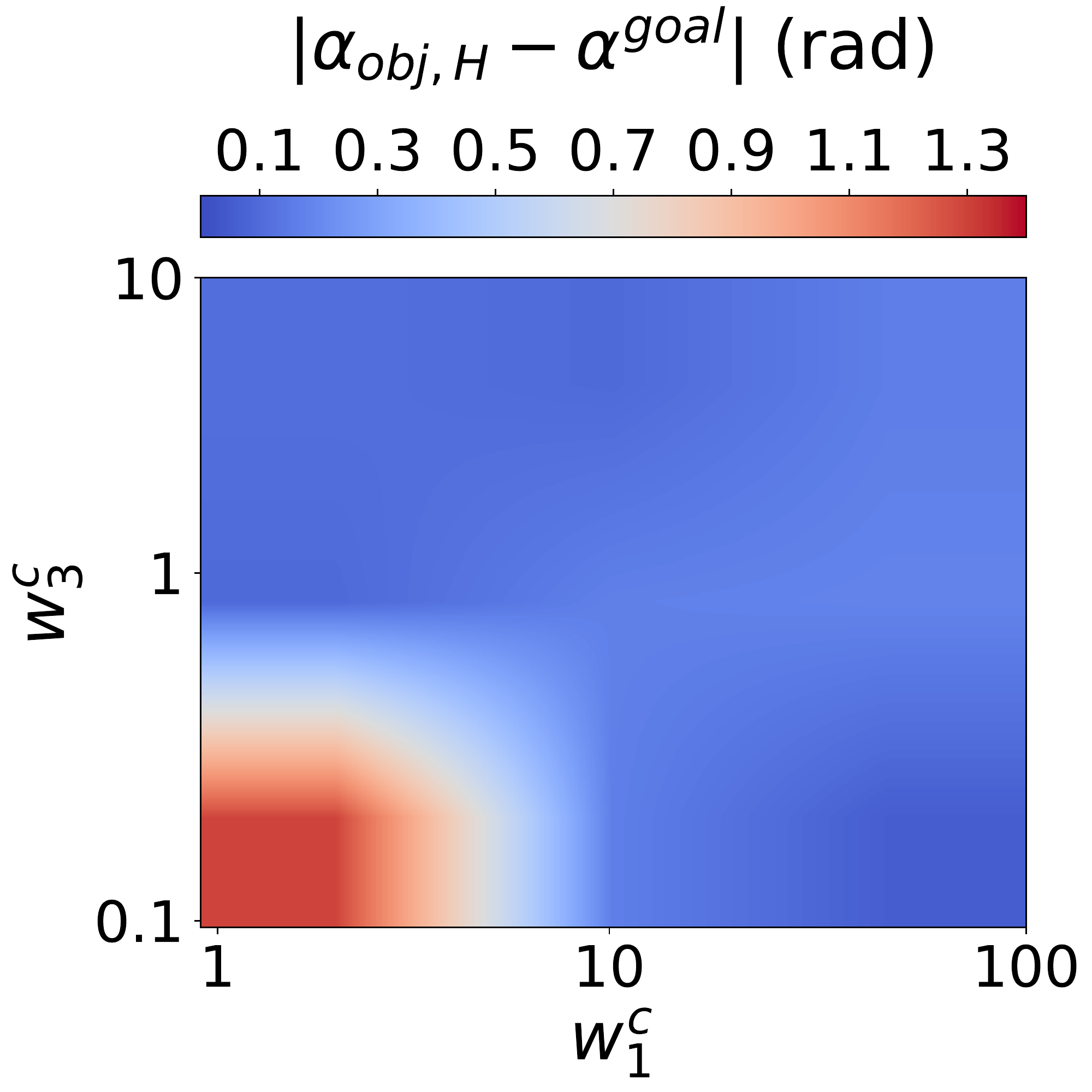}
		\caption{}
		\label{fig.discussion.path_weights.angle_error}
	\end{subfigure}
		\begin{subfigure}{.24\textwidth}
		\centering
		\includegraphics[width=\linewidth]{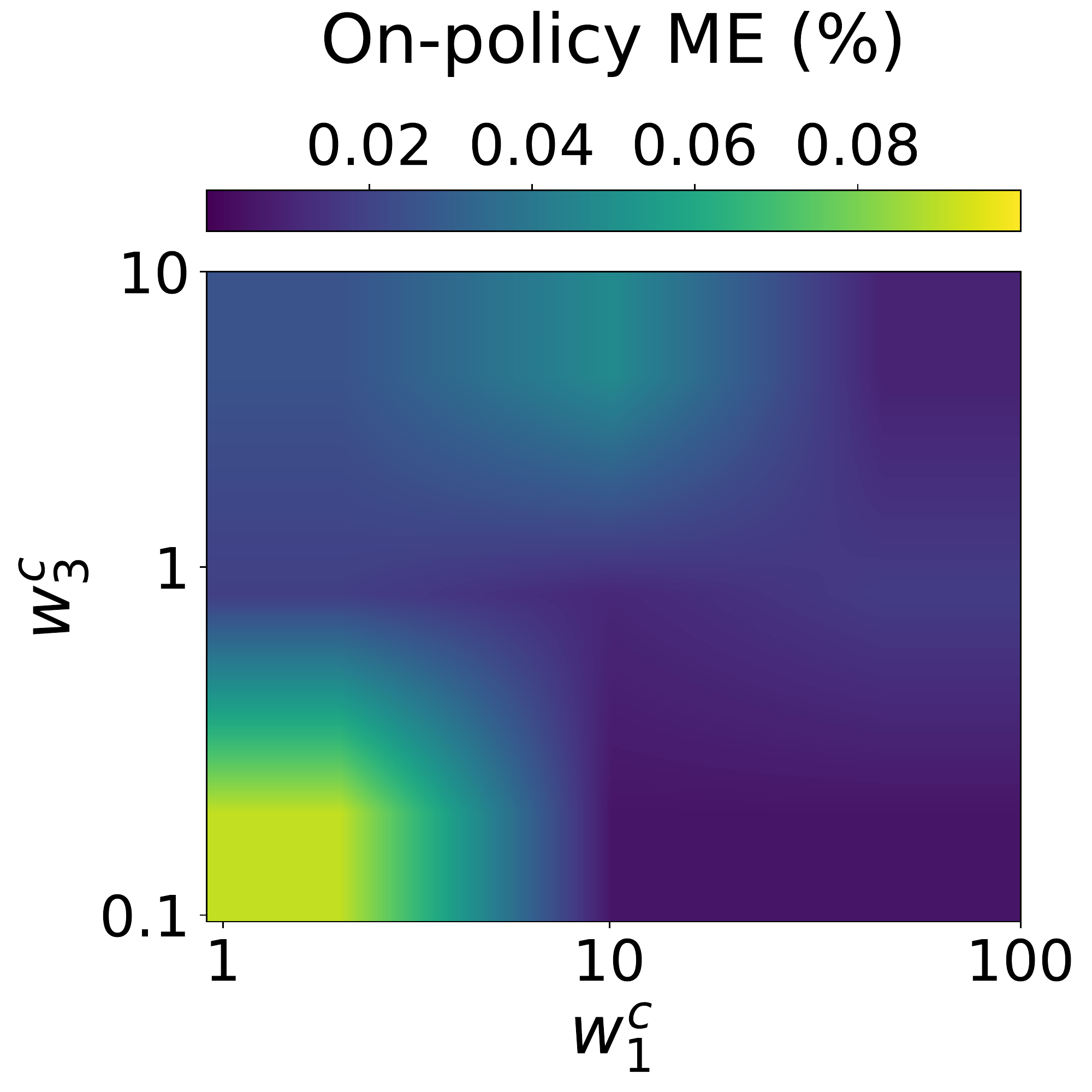}
		\caption{}
		\label{fig.discussion.path_weights.me}
	\end{subfigure}
 	\centering
		\begin{subfigure}{.24\textwidth}
		\centering
		\includegraphics[width=\linewidth]{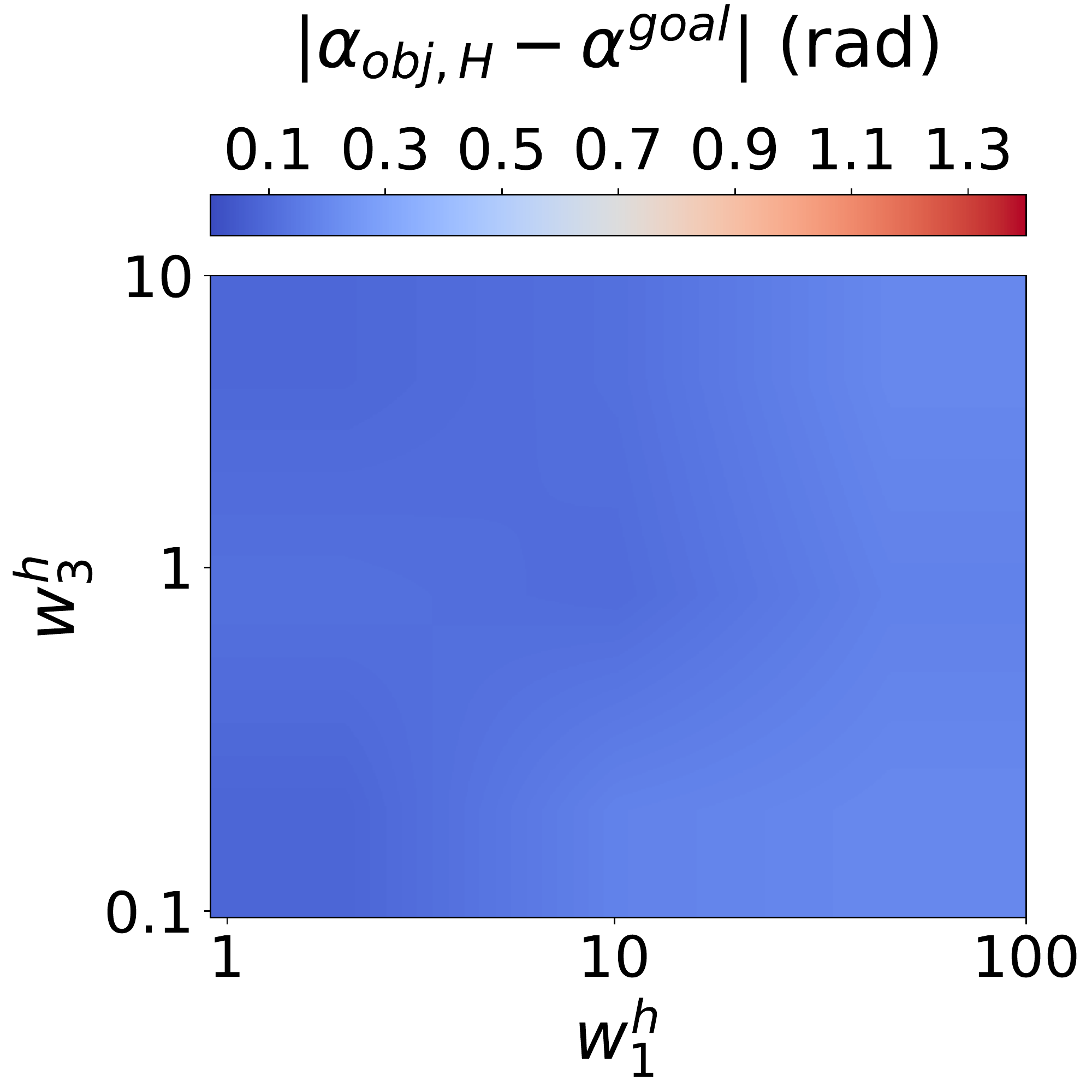}
		\caption{}
		\label{fig.discussion.final_weights.angle_error}
	\end{subfigure}
		\begin{subfigure}{.24\textwidth}
		\centering
		\includegraphics[width=\linewidth]{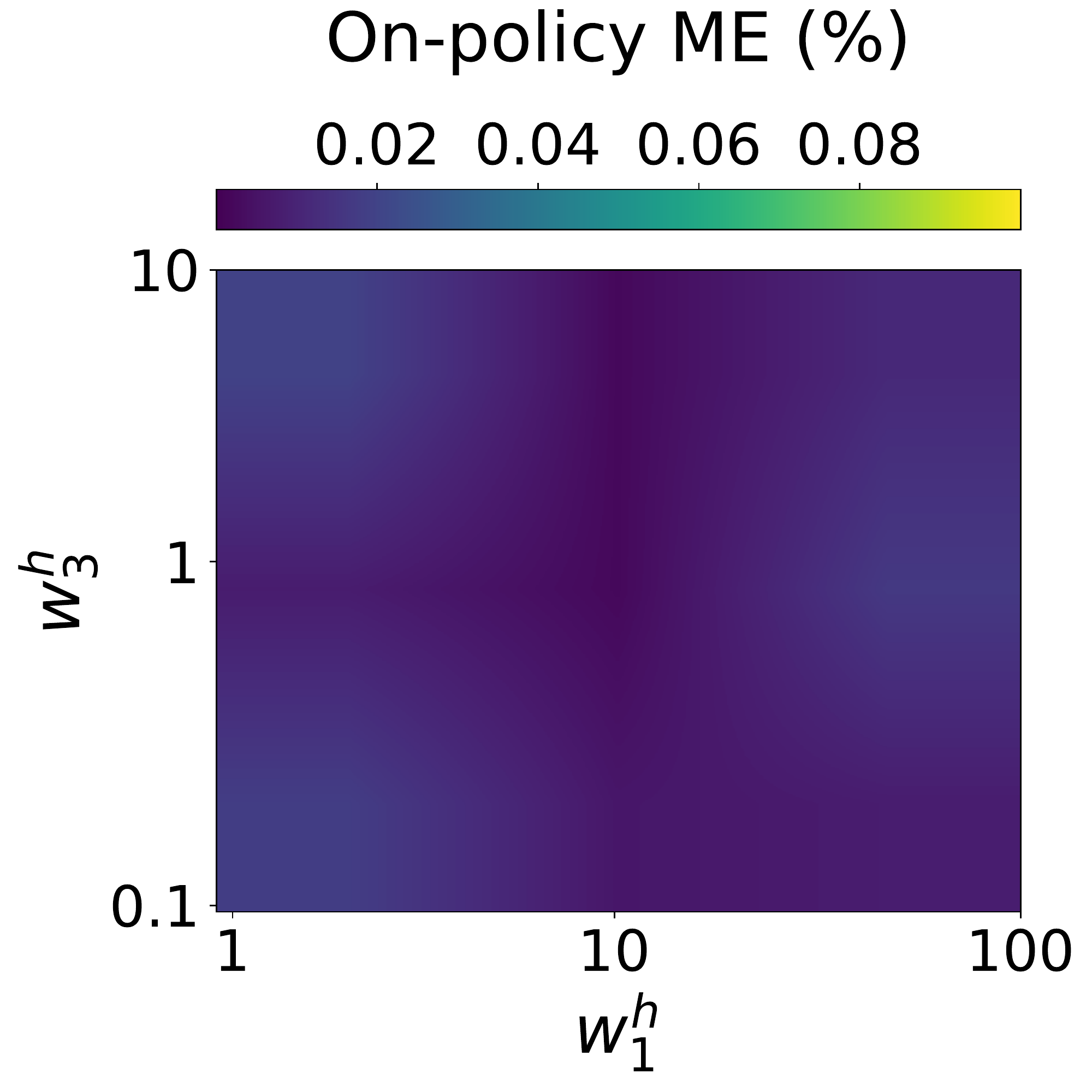}
		\caption{}
		\label{fig.discussion.final_weights.me}
	\end{subfigure}
	\caption{ \small
	Performance of the proposed method with different choices of cost weights $\boldsymbol{w}^c$ and  $\boldsymbol{w}^h$
 in (\ref{equ.tf.cost1}) for the Cube Turning  task. (a) and (b) report the performance using different  $\boldsymbol{w}^c$ while fixing  $\boldsymbol{w}^h=[2,0, 10]\tran$, and (c) and (d) using different  $\boldsymbol{w}^h$ while fixing  $\boldsymbol{w}^c=[10, 0, 2]\tran$. The metric in (a) and (c) is the cube's terminal orientation error $|\alpha_{\obj, H}-\alpha^\goal|$, as used in Table \ref{table.tf.task1.table1}, and the metric in (b) and (d) uses the on-policy model error, defined in (\ref{equ.pwa.mpe}).
	} 
	\label{fig.discussion.cost_weights}
  \vspace{-10pt}
\end{figure}

Compared to the previous performance (see Fig. \ref{fig.tf.task1.plots} and Table \ref{table.tf.task1.table1}) in Section VII.C,  Fig. \ref{fig.discussion.path_weights.angle_error} and Fig.  \ref{fig.discussion.path_weights.me} show that a similarly good performance permits a wide selection of running cost  weights $\boldsymbol{w}^c$, e.g.,  $w_3^c=0.5$ and $w_1^c\in(10,100)$. Typically, to achieve a good algorithm performance, Fig. \ref{fig.discussion.path_weights.angle_error} and Fig.  \ref{fig.discussion.path_weights.me} suggest that the selection of running cost weights $\boldsymbol{w}^c$ should  make the corresponding cost terms, here, $w_1^c$ for the distance between fingertip and cube and $w_3^c$ for the orientation cost, roughly have a similar scale. The results in Fig. \ref{fig.discussion.final_weights.angle_error} and Fig.  \ref{fig.discussion.final_weights.me} also say that the allowable choice of final cost weights $\boldsymbol{w}^h$ is more flexible, and different $\boldsymbol{w}^h$s have limited influence on the algorithm performance. Thus, we can  conclude that the proposed algorithm is not sensitive to the choice of   $\boldsymbol{w}^c$ and  $\boldsymbol{w}^h$
 in (\ref{equ.tf.cost1}). In practical implementation, it is always not difficult to find  the cost  weights to produce good performance, and one empirical principle to select cost weights is to make each cost term have a similar scale.

\subsubsection{Choice of Task Distribution}
We briefly discuss the choice of task parameter $\boldsymbol{\beta}$, which we have used to define a distribution of tasks $p(\boldsymbol{\beta})$.
In the examples above, we have considered uniformly distributed targets (e.g., $p(\boldsymbol{\beta})$ in (\ref{equ.tf.task1}) and (\ref{equ.tf.task2})). 
We note, however, that this choice is somewhat arbitrary, and for completeness we demonstrate similar performance across other distributions. The experiment settings follow Section \ref{section.tf.task1}, except using different targets distribution $p(\boldsymbol{\beta})$. 

  \begin{figure}[h]
	\centering
		\begin{subfigure}{.235\textwidth}
		\centering
		\includegraphics[width=\linewidth]{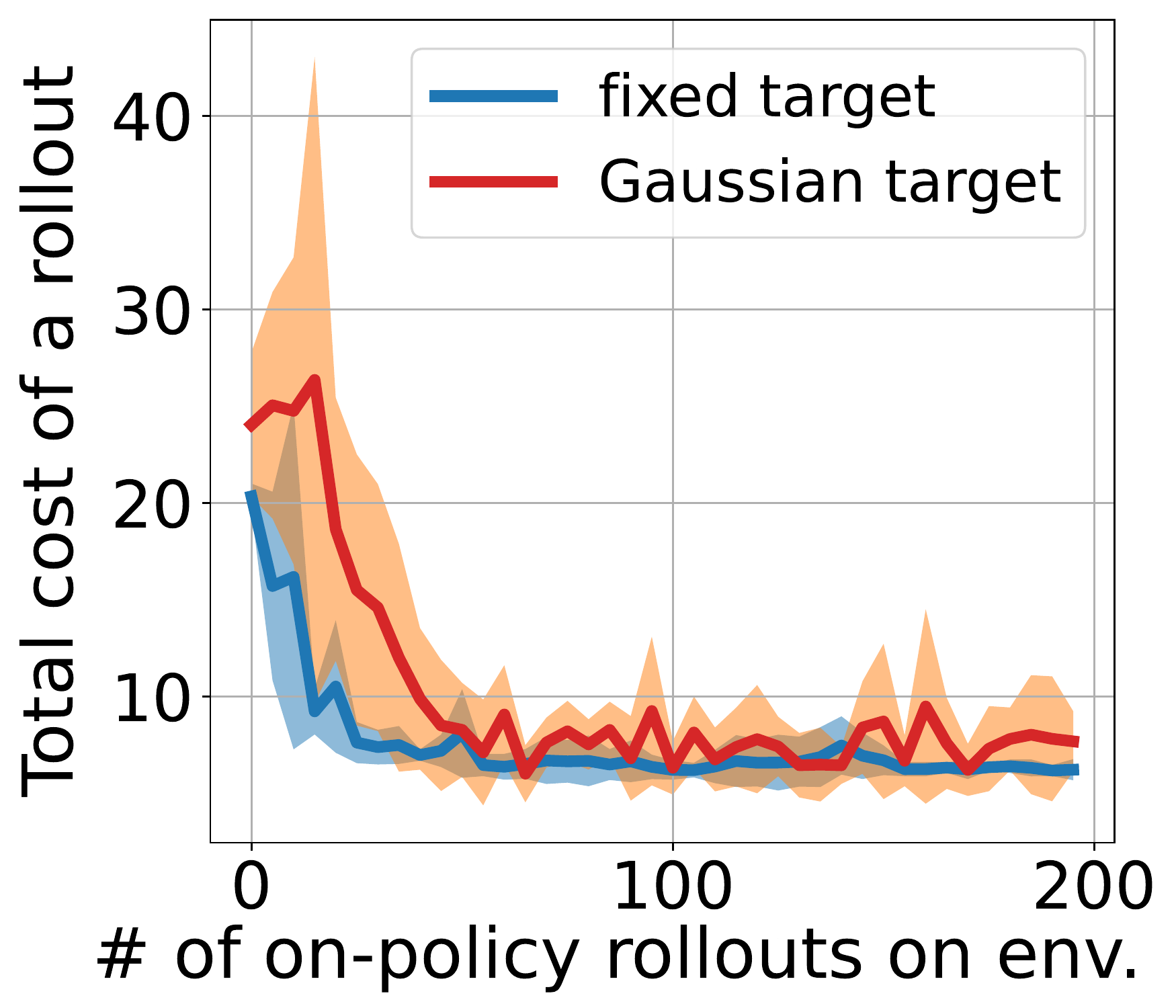}
		\caption{}
		\label{fig.discussion.beta.1}
	\end{subfigure}
		\begin{subfigure}{.235\textwidth}
		\centering
		\includegraphics[width=\linewidth]{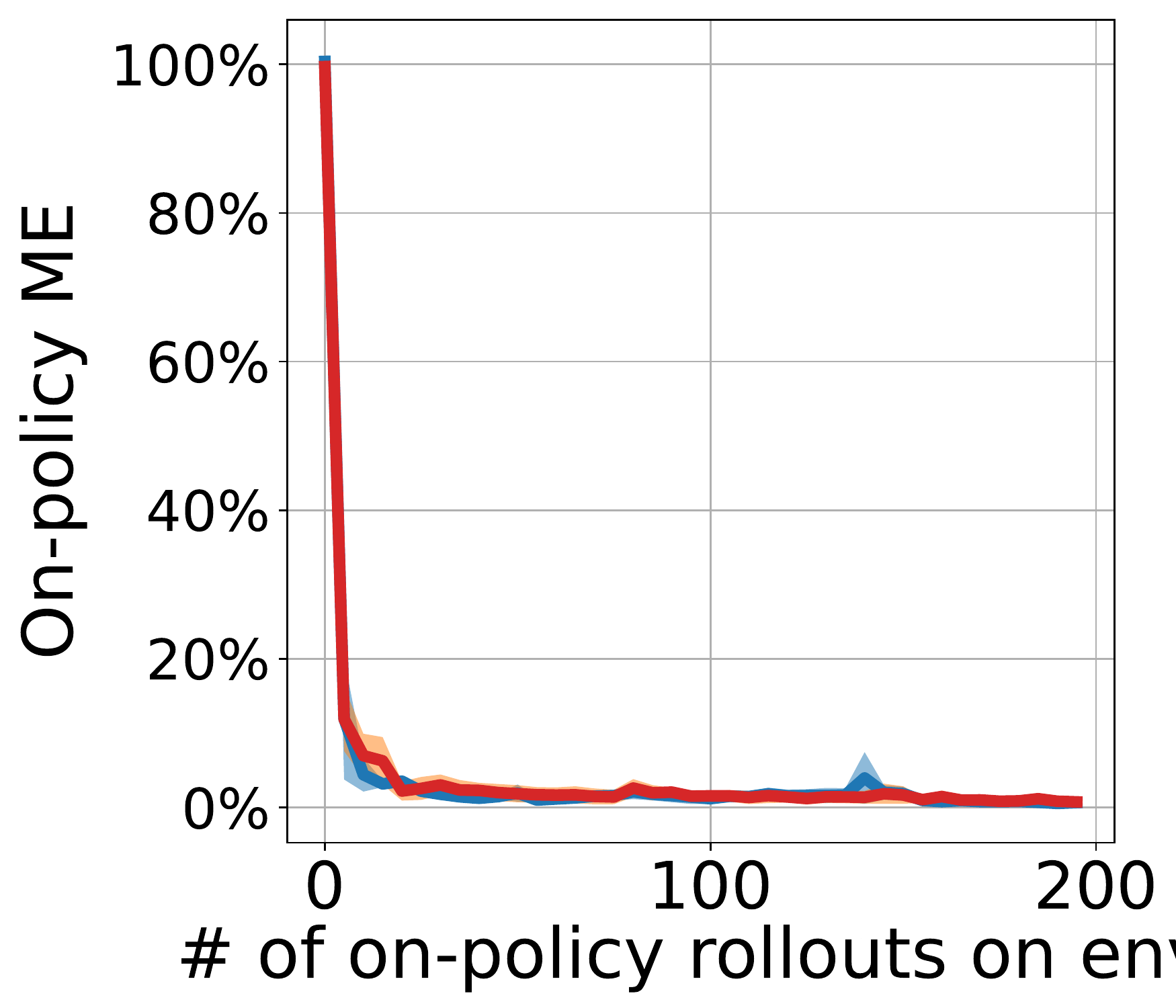}
		\caption{}
		\label{fig.discussion.beta.2}
	\end{subfigure}
	\caption{ \small
	Performance of the proposed method  under different target distribution in the Cube Turning task. We use two target distributions: ${\delta}$-distribution for a single fixed target, i.e., $p(\boldsymbol{\beta}=0.6)=1$, in the blue lines, and a Gaussian distributed target,  $p(\boldsymbol{\beta})=\mathcal{N}(0.6, 0.2)$, in the red lines. The  metric  follows the ones used in Fig. \ref{fig.tf.task1.plots}. 
	} 
	\label{fig.discussion.beta}
 \vspace{-10pt}
\end{figure}

Fig. \ref{fig.discussion.beta}, shows results from two other target distributions for the Cube Turning  task: ${\delta}$-distribution corresponding to a single fixed target,  $p(\boldsymbol{\beta}=0.6)=1$, shown in the blue lines, and a Gaussian distributed target,  $p(\boldsymbol{\beta})=\mathcal{N}(0.6, 0.2)$, shown in the red lines. The learning curves  in Fig.\ref{fig.discussion.beta.1} and Fig. \ref{fig.discussion.beta.2} indicate that for each task distribution, a reduced-order LCS model is successfully learned. Also, we notice that compared to the single target task, multiple targets bring more variance. 
In conclusion, given a different task distribution, the proposed method always finds a reduced order model to minimize the task performance gap in expectation over the task distribution, as in (\ref{equ.loss1}). In practice, we expect $p(\boldsymbol{\beta})$ to be chosen either by the practitioner, to best represent the tasks the robot must accomplish, or to be independently identified (e.g., via the output from some higher-level planner).

\subsubsection{Limitation of PWA Models} In the paper, we use PWA models (\ref{equ.lcs.reduced}) as the reduced-order hybrid representation. As these PWA models are inherently based on linearization, they do not naturally apply to all manipulation tasks, for example large rotations (e.g. full 360-degree rotation of the cube). Such large rotations involve  significant non-linearity that local PWA models cannot capture well, unless one  adds more `pieces' in PWA to approximate it \cite{dai2019global}. However, 
using more `pieces' to approximate smooth non-linearity is not the interest of this paper; instead, we focus on using pieces to capture the hybrid structure (i.e.,  mode boundaries) of a non-linear hybrid system. But this limitation motivates a future direction to extend LCS representation for non-linear complementarity models.

\section{Conclusions And Future Work}\label{section.conclusion}
This paper proposes the method of task-driven hybrid model reduction for multi-contact dexterous manipulation. Building upon our prior work of learning linear complementarity systems, we propose learning a reduced-order hybrid model with a limited number of task-relevant hybrid modes, such that  it enables real-time closed-loop MPC control and is sufficient to achieve high performance on hybrid systems like multi-finger dexterous manipulation. We have shown that learning a reduced-order hybrid model attains  a provably upper-bounded closed-loop performance. We have demonstrated the proposed method in reducing the mode count of synthetic hybrid control systems by multiple orders of magnitude while achieving task performance loss of less than $5\%$. We apply the proposed method to solve three-finger robotic hand manipulation for object reorientation. Without any prior knowledge, the proposed method achieves state-of-the-art closed-loop performance in less than five minutes of data collection and model learning. The future work includes building the hardware and testing it on the hardware robotic manipulation system, as well as extension into nonlinear reduced-order hybrid models.

\section*{Acknowledgements}
Toyota Research Institute  funded and supported this work.

\appendices
\section{Proof of Lemma \ref{lemma.1}}\label{appendix.0}
\begin{proof} 
\begin{equation}
\begin{aligned}
     &J_{{\boldsymbol{\beta}}}\big(\mathbf{u}^{\boldsymbol{g}}, \mathbf{F}(\mathbf{u}^{\boldsymbol{g}},\boldsymbol{x}_0)\big) -J_{{\boldsymbol{\beta}}}\big(\mathbf{u}^{\boldsymbol{f}}, \mathbf{F}(\mathbf{u}^{\boldsymbol{f}},\boldsymbol{x}_0)\big)
     \\[5pt]
    = \,\, &\underbrace{J_{{\boldsymbol{\beta}}}\big(\mathbf{u}^{\boldsymbol{g}}, \mathbf{F}(\mathbf{u}^{\boldsymbol{g}},\boldsymbol{x}_0)\big)
    -J_{{\boldsymbol{\beta}}}\big(\mathbf{u}^{\boldsymbol{g}}, \mathbf{G}(\mathbf{u}^{\boldsymbol{g}},\boldsymbol{x}_0)\big)}_{\text{Term I}}
    \\
    & + \,\, \underbrace{J_{{\boldsymbol{\beta}}}\big(\mathbf{u}^{\boldsymbol{g}}, \mathbf{G}(\mathbf{u}^{\boldsymbol{g}},\boldsymbol{x}_0)\big)
    -J_{{\boldsymbol{\beta}}}\big(\mathbf{u}^{\boldsymbol{f}}, \mathbf{G}(\mathbf{u}^{\boldsymbol{f}},\boldsymbol{x}_0)\big)}_{\text{Term II}}
    \\
    & + \,\,\underbrace{J_{{\boldsymbol{\beta}}}\big(\mathbf{u}^{\boldsymbol{f}}, \mathbf{G}(\mathbf{u}^{\boldsymbol{f}},\boldsymbol{x}_0)\big)
    -J_{{\boldsymbol{\beta}}}\big(\mathbf{u}^{\boldsymbol{f}}, \mathbf{F}(\mathbf{u}^{\boldsymbol{f}},\boldsymbol{x}_0)\big)}_{\text{Term III}},
\end{aligned}
\end{equation}
Here, Term I and Term III can follow the Lipschitz continuity:
\begin{equation}\label{equ.proof1}
    \begin{aligned}
          \text{Term I}&\leq M\norm{\mathbf{G}(\mathbf{u}^{\boldsymbol{g}},\boldsymbol{x}_0)-\mathbf{F}(\mathbf{u}^{\boldsymbol{g}},\boldsymbol{x}_0)} \\
         \text{Term III}&\leq M\norm{\mathbf{G}(\mathbf{u}^{\boldsymbol{f}},\boldsymbol{x}_0)-\mathbf{F}(\mathbf{u}^{\boldsymbol{f}},\boldsymbol{x}_0)}. \\      
    \end{aligned}
\end{equation}
Term II trivially satisfies 
\begin{equation}\label{equ.proof2}
    J_{{\boldsymbol{\beta}}}\big(\mathbf{u}^{\boldsymbol{g}}, \mathbf{G}(\mathbf{u}^{\boldsymbol{g}},\boldsymbol{x}_0)\big)
    -J_{{\boldsymbol{\beta}}}\big(\mathbf{u}^{\boldsymbol{f}}, \mathbf{G}(\mathbf{u}^{\boldsymbol{f}},\boldsymbol{x}_0)\big)\leq 0
\end{equation}
because $\mathbf{u}^{\boldsymbol{g}}$ is the minimum by the definition of $\boldsymbol{g}$-MPC  (\ref{equ.mpcsimple}). 
Putting (\ref{equ.proof1}) and (\ref{equ.proof2}) together, one has 
\begin{multline}
    J_{{\boldsymbol{\beta}}}\big(\mathbf{u}^{\boldsymbol{g}}, \mathbf{F}(\mathbf{u}^{\boldsymbol{g}},\boldsymbol{x}_0)\big) {-}J_{{\boldsymbol{\beta}}}\big(\mathbf{u}^{\boldsymbol{f}}, \mathbf{F}(\mathbf{u}^{\boldsymbol{f}},\boldsymbol{x}_0)\big)
    \leq\\ M
     \Big(\norm{\mathbf{G}(\mathbf{u}^{\boldsymbol{g}},\boldsymbol{x}_0){-}\mathbf{F}(\mathbf{u}^{\boldsymbol{g}},\boldsymbol{x}_0)}
    + 
     \norm{\mathbf{G}(\mathbf{u}^{\boldsymbol{f}},\boldsymbol{x}_0)-\mathbf{F}(\mathbf{u}^{\boldsymbol{f}},\boldsymbol{x}_0)}\Big). \nonumber
\end{multline}
The above still holds with expectation of both sides over $ p(\boldsymbol{x}_0)$ and $ p({\boldsymbol{\beta}})$,  yielding (\ref{equ.lossupperbound}). This completes the proof.
\end{proof}

\begin{remark*}  $
          \E_{ p({\boldsymbol{\beta}})} \E_{ p_{\boldsymbol{\beta}}(\boldsymbol{x}_0)}  
     \Big\lVert\mathbf{G}(\mathbf{u}^{\boldsymbol{f}},\boldsymbol{x}_0)-\mathbf{F}(\mathbf{u}^{\boldsymbol{f}},\boldsymbol{x}_0)\Big\rVert
$ in (\ref{equ.lossupperbound}) is called the  domain adaption    in reinforcement learning  \cite{rajeswaran2020game}. Specifically, if $\boldsymbol{g}()$ is trained only with   $\boldsymbol{g}\text{-MPC}$  data  $\mathcal{D}^{\boldsymbol{g}}$, the \emph{domain adaption} term  captures the model error when the   $\boldsymbol{g}()$ is evaluated with  $\boldsymbol{f}\text{-MPC}$  data $\mathcal{D}^{\boldsymbol{f}}$. This domain adaption term  is inevitable  if want a learned proxy model $\boldsymbol{g}()$ trained on its generated data to capture the  true dynamics $\boldsymbol{f}()$.
\end{remark*}

\section{Proof of Lemma \ref{lemma.2}}\label{appendix.1}
\begin{proof}
Given any $\boldsymbol{x}_0\sim p(\boldsymbol{x}_0)$ and ${\boldsymbol{\beta}}\sim p({\boldsymbol{\beta}})$, as $\mathbf{u}^{\boldsymbol{g}}(\boldsymbol{x}_0,{\boldsymbol{\beta}})$ is a  solution to (\ref{equ.mpcsimple}), it satisfies the  first-order condition
\begin{multline}\label{equ.proof.lemma2.1}
\small
         \nabla_{\mathbf{u}} J_{\boldsymbol{\beta}}\big(\mathbf{u}^{\boldsymbol{g}}, \mathbf{G}(\mathbf{u}^{\boldsymbol{g}},\boldsymbol{x}_0)\big)=\left(\frac{\partial J_{\boldsymbol{\beta}}}{\partial \mathbf{u}^{\boldsymbol{g}}}+\frac{\partial J_{\boldsymbol{\beta}}}{\partial \mathbf{G}(\mathbf{u}^{\boldsymbol{g}})}\frac{\partial \mathbf{G}}{\partial \mathbf{u}^{\boldsymbol{g}}}\right)\tran=\boldsymbol{0},
\end{multline}
where $\frac{\partial J_{\boldsymbol{\beta}}}{\partial \mathbf{u}^{\boldsymbol{g}}}$ denotes  the partial gradient of $J_{\boldsymbol{\beta}}\big(\mathbf{u}, \mathbf{G}(\mathbf{u},\boldsymbol{x}_0)\big)$ with respect to $\boldsymbol{\mathbf{u}}$ evaluated at $\mathbf{u}^{\boldsymbol{g}}(\boldsymbol{x}_0,{\boldsymbol{\beta}})$, and similar notations applies to $\frac{\partial J_{\boldsymbol{\beta}}}{\partial \mathbf{G}(\mathbf{u}^{\boldsymbol{g}})}$ and $\frac{\partial \mathbf{G}}{\partial \mathbf{u}^{\boldsymbol{g}}}$ and  below. Using (\ref{equ.proof.lemma2.1}), one has
\begin{equation}\label{equ.proof.lemma2.2}
\small
    \begin{aligned}
             &\Big
 \lVert\nabla_{\mathbf{u}} J_{\boldsymbol{\beta}}\big(\mathbf{u}^{\boldsymbol{g}}, \mathbf{F}(\mathbf{u}^{\boldsymbol{g}},\boldsymbol{x}_0)\big)\Big
 \rVert\\[5pt]
             = &\Big
 \lVert\nabla_{\mathbf{u}} J_{\boldsymbol{\beta}}\big(\mathbf{u}^{\boldsymbol{g}}, \mathbf{F}(\mathbf{u}^{\boldsymbol{g}},\boldsymbol{x}_0)\big)-\nabla_{\mathbf{u}} J_{\boldsymbol{\beta}}\big(\mathbf{u}^{\boldsymbol{g}}, \mathbf{G}(\mathbf{u}^{\boldsymbol{g}},\boldsymbol{x}_0)\big)
             \Big
 \rVert\\[5pt]
             =&\left
             \lVert 
             \frac{\partial J_{\boldsymbol{\beta}}\big(\mathbf{u}, \mathbf{F}\big)}{\partial \mathbf{u}^{\boldsymbol{g}}}{+}\frac{\partial J_{\boldsymbol{\beta}}}{\partial \mathbf{F}(\mathbf{u}^{\boldsymbol{g}})}\frac{\partial \mathbf{F}}{\partial \mathbf{u}^{\boldsymbol{g}}}
             {-}
              \frac{\partial J_{\boldsymbol{\beta}}\big(\mathbf{u}, \mathbf{G}\big)}{\partial \mathbf{u}^{\boldsymbol{g}}}{-}\frac{\partial J_{\boldsymbol{\beta}}}{\partial \mathbf{G}(\mathbf{u}^{\boldsymbol{g}})}\frac{\partial \mathbf{G}}{\partial \mathbf{u}^{\boldsymbol{g}}}
             \right
             \rVert
             \\
             \leq&\left
             \lVert 
             \frac{\partial J_{\boldsymbol{\beta}}\big(\mathbf{u}, \mathbf{F}\big)}{\partial \mathbf{u}^{\boldsymbol{g}}}
             {-}
             \frac{\partial J_{\boldsymbol{\beta}}\big(\mathbf{u}, \mathbf{G}\big)}{\partial \mathbf{u}^{\boldsymbol{g}}}
              \right
             \rVert
             {+}
             \left
             \lVert 
              \frac{\partial J_{\boldsymbol{\beta}}}{\partial \mathbf{F}(\mathbf{u}^{\boldsymbol{g}})}\frac{\partial \mathbf{F}}{\partial \mathbf{u}^{\boldsymbol{g}}}
              {-}\frac{\partial J_{\boldsymbol{\beta}}}{\partial \mathbf{G}(\mathbf{u}^{\boldsymbol{g}})}\frac{\partial \mathbf{G}}{\partial \mathbf{u}^{\boldsymbol{g}}}
             \right
             \rVert,
    \end{aligned}
\end{equation}
where last inequality is due to Cauchy–Schwarz inequality, and for clarity, we drop the dependence on $\boldsymbol{x}_0$ momentarily.
The first term in (\ref{equ.proof.lemma2.2}) has
\begin{equation}\label{equ.proof.lemma2.3}
\small
    \left
     \lVert 
     \frac{\partial J_{\boldsymbol{\beta}}\big(\mathbf{u}, \mathbf{F}\big)}{\partial \mathbf{u}^{\boldsymbol{g}}}
     -
     \frac{\partial J_{\boldsymbol{\beta}}\big(\mathbf{u}, \mathbf{G}\big)}{\partial \mathbf{u}^{\boldsymbol{g}}}
      \right
     \rVert\leq L_2\Big
 \lVert\mathbf{F}(\mathbf{u}^{\boldsymbol{g}})-\mathbf{G}(\mathbf{u}^{\boldsymbol{g}})\Big
 \rVert.
\end{equation}
because of the $L_2$-Lipschitz continuity of $\nabla_{\mathbf{u}}{J}_{\boldsymbol{\beta}}({\mathbf{u}}, {\mathbf{x}})$ assumed in Lemma \ref{lemma.2}. The second term of (\ref{equ.proof.lemma2.2}) has 
\begin{equation}\label{equ.proof.lemma2.4}
\small
\begin{aligned}
&\left
 \lVert 
  \frac{\partial J_{\boldsymbol{\beta}}}{\partial \mathbf{F}(\mathbf{u}^{\boldsymbol{g}})}\frac{\partial \mathbf{F}}{\partial \mathbf{u}^{\boldsymbol{g}}}
  {-}\frac{\partial J_{\boldsymbol{\beta}}}{\partial \mathbf{G}(\mathbf{u}^{\boldsymbol{g}})}\frac{\partial \mathbf{G}}{\partial \mathbf{u}^{\boldsymbol{g}}}
 \right
 \rVert\\
 =&
 \bigg
 \lVert 
  \frac{\partial J_{\boldsymbol{\beta}}}{\partial \mathbf{F}(\mathbf{u}^{\boldsymbol{g}})}\frac{\partial \mathbf{F}}{\partial \mathbf{u}^{\boldsymbol{g}}}-
    \frac{\partial J_{\boldsymbol{\beta}}}{\partial \mathbf{F}(\mathbf{u}^{\boldsymbol{g}})}\frac{\partial \mathbf{G}}{\partial \mathbf{u}^{\boldsymbol{g}}} 
    {+} 
      \frac{\partial J_{\boldsymbol{\beta}}}{\partial \mathbf{F}(\mathbf{u}^{\boldsymbol{g}})}\frac{\partial \mathbf{G}}{\partial \mathbf{u}^{\boldsymbol{g}}}
  {-}\frac{\partial J_{\boldsymbol{\beta}}}{\partial \mathbf{G}(\mathbf{u}^{\boldsymbol{g}})}\frac{\partial \mathbf{G}}{\partial \mathbf{u}^{\boldsymbol{g}}}
 \bigg
 \rVert
 \\
 \leq& \bigg
 \lVert 
  \frac{\partial J_{\boldsymbol{\beta}}}{\partial \mathbf{F}(\mathbf{u}^{\boldsymbol{g}})}\frac{\partial \mathbf{F}}{\partial \mathbf{u}^{\boldsymbol{g}}}{-}
    \frac{\partial J_{\boldsymbol{\beta}}}{\partial \mathbf{F}(\mathbf{u}^{\boldsymbol{g}})}\frac{\partial \mathbf{G}}{\partial \mathbf{u}^{\boldsymbol{g}}}  \bigg
 \rVert
 {+}\bigg
 \lVert 
       \frac{\partial J_{\boldsymbol{\beta}}}{\partial \mathbf{F}(\mathbf{u}^{\boldsymbol{g}})}\frac{\partial \mathbf{G}}{\partial \mathbf{u}^{\boldsymbol{g}}}
  {-}\frac{\partial J_{\boldsymbol{\beta}}}{\partial \mathbf{G}(\mathbf{u}^{\boldsymbol{g}})}\frac{\partial \mathbf{G}}{\partial \mathbf{u}^{\boldsymbol{g}}}
 \bigg
 \rVert\\
 \leq &M_1 \bigg
 \lVert
\frac{\partial \mathbf{F}}{\partial \mathbf{u}^{\boldsymbol{g}}}-\frac{\partial \mathbf{G}}{\partial \mathbf{u}^{\boldsymbol{g}}}
  \bigg
 \rVert
 +M_gL_1\Big
 \lVert\mathbf{F}(\mathbf{u}^{\boldsymbol{g}})-\mathbf{G}(\mathbf{u}^{\boldsymbol{g}})\Big
 \rVert,
\end{aligned}
\end{equation}
where the last inequality is due to the bound $\norm{\nabla_{\mathbf{x}}{J}_{\boldsymbol{\beta}}({\mathbf{u}}, {\mathbf{x}})}\leq M_1$,
$L_1$-Lipschitz continuity of $\nabla_{\mathbf{x}}{J}_{\boldsymbol{\beta}}({\mathbf{u}}, {\mathbf{x}})$, and the bound $\norm{\nabla_{\mathbf{u}}{G}({\mathbf{u}})}\leq M_g$ given in  Lemma \ref{lemma.2}.

Combining (\ref{equ.proof.lemma2.2})-(\ref{equ.proof.lemma2.4}) and  replacing $\big
 \lVert
\frac{\partial \mathbf{F}}{\partial \mathbf{u}^{\boldsymbol{g}}}-\frac{\partial \mathbf{G}}{\partial \mathbf{u}^{\boldsymbol{g}}}
  \big
 \rVert$ compactly with
 $\big
 \lVert
\nabla_{\mathbf{u}}\mathbf{F}({\mathbf{u}})-\nabla_{\mathbf{u}}\mathbf{G}({\mathbf{u}})
  \big
 \rVert$
 we have  (\ref{equ.lemma2.1}) and (\ref{equ.lemma2.2}) in Lemma \ref{lemma.2}. This completes the proof.
\end{proof}


\bibliographystyle{IEEEtran}
\bibliography{trobib}

\vspace{0pt}

\begin{IEEEbiography}[{\includegraphics[width=1.2in,height=1.2in,keepaspectratio]{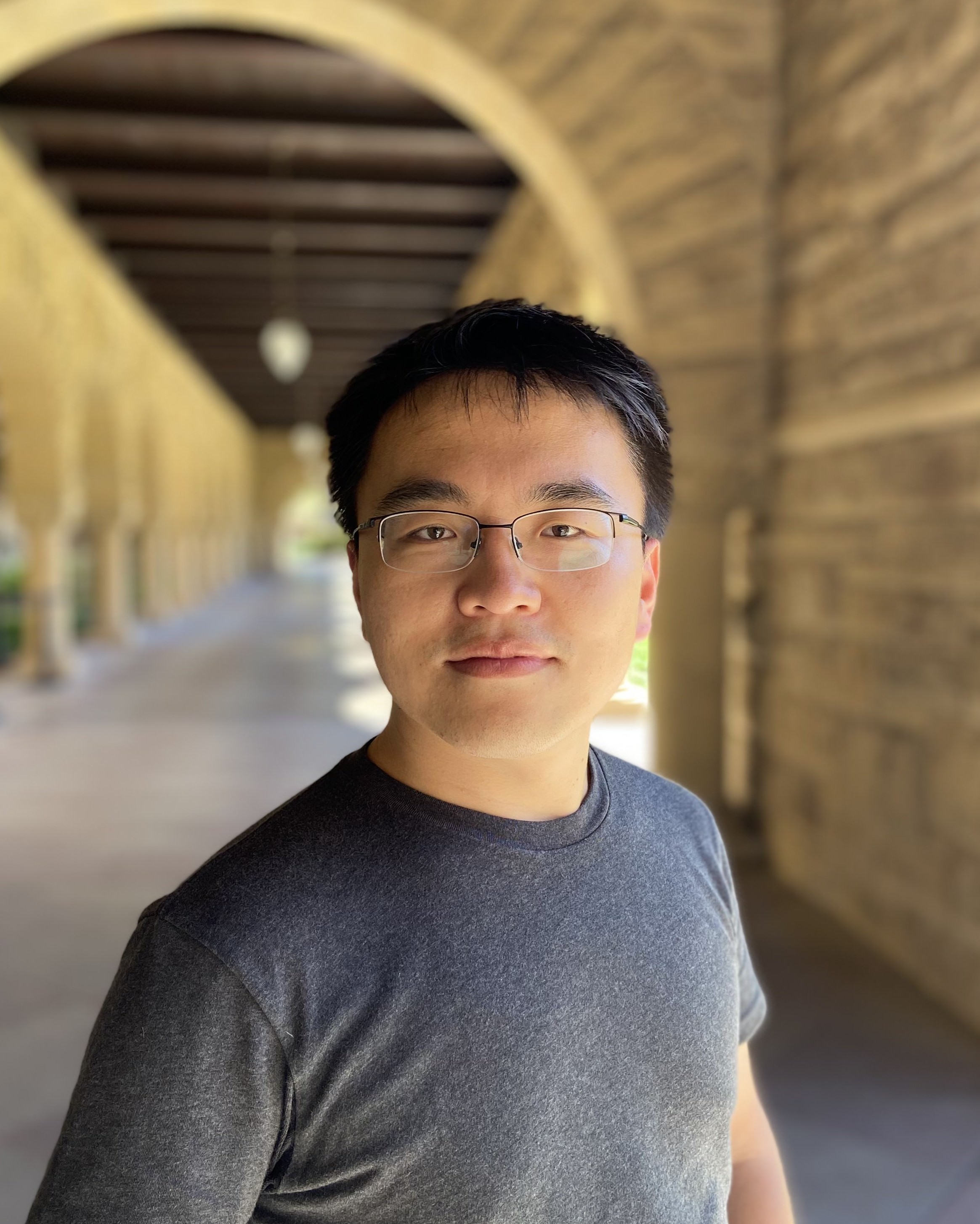}}]
{Wanxin Jin} is an Assistant Professor in the School for Engineering of Matter, Transport, and Energy at Arizona State University in Tempe, AZ, USA.  He earned his Ph.D. degree from Purdue University in West Lafayette, IN, USA, in 2021. Between 2021 and 2023, Dr. Jin held the position of Postdoctoral Researcher at the GRASP Lab, University of Pennsylvania, Philadelphia, PA, USA. At Arizona State University, Dr. Jin leads the Intelligent Robotics and Interactive Systems Lab, focusing on developing fundamental methods at the convergence of control, machine learning, and optimization, with the goal to enable robots to  safely and efficiently interact with humans and physical objects.
\end{IEEEbiography}

\begin{IEEEbiography}[{\includegraphics[width=1.2in,height=1.2in,keepaspectratio]{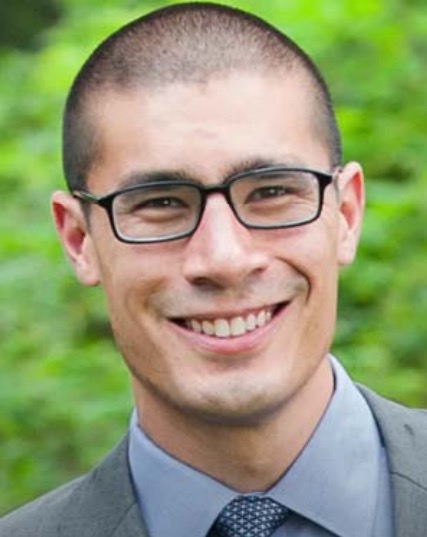}}]
	{Michael Posa} received the B.S. and M.S. degrees in mechanical engineering from Stanford University, Stanford, CA, USA, in 2007 and 2008, respectively. He received the Ph.D. degree in electrical engineering and computer science from the Massachusetts Institute of Technology, Cambridge, MA, USA, in 2017. He is currently an Assistant Professor of Mechanical Engineering and Applied Mechanics with the University of Pennsylvania, Philadelphia, PA, USA, where he is a Member of the General Robotics, Automation, Sensing and Perception (GRASP) Lab. He holds secondary appointments in electrical and systems engineering and in computer and information science. He leads the Dynamic Autonomy and Intelligent Robotics Lab, University of Pennsylvania, which focuses on developing computationally tractable algorithms to enable robots to operate both dynamically and safely as they maneuver through and interact with their environments, with applications including legged locomotion and manipulation. Dr. Posa was a recipient of multiple awards, including the NSF CAREER Award, RSS Early Career Spotlight Award, and best paper awards. He is an Associate Editor for IEEE Transactions on Robotics.
\end{IEEEbiography}

\end{document}